\documentclass{article}

\usepackage{ijcai15}

\newif\ifdraft
\drafttrue
\draftfalse 

\usepackage{times}
\usepackage{latexsym,amssymb,amsmath}
\usepackage{xspace}
\usepackage{comment}
\usepackage{url}
\usepackage{enumerate}
\usepackage[defblank]{paralist}
\usepackage{microtype}
\usepackage{hyperref}
\usepackage[english]{cleveref}

 \usepackage[normalem]{ulem} 
 \ifdraft
   \usepackage{mychanges}
   \usepackage[colorinlistoftodos]{todonotes}
 \else
   \usepackage[final]{mychanges}
   \usepackage[disable]{todonotes}
 \fi
 \setauthormarkup[right]{{\scriptsize[#1]}}
 \definechangesauthor{DC}{orange}
 \definechangesauthor{MM}{green}
 \definechangesauthor{AS}{blue}

 \newcommand{\asd}[1]{\deleted[AS]{#1}} 
 \newcommand{\asr}[2]{\replaced[AS]{#1}{#2}} 

\newcommand{\gkabsym}{\G}
\newcommand{\ginitprog}{\delta}

\newcommand{\gemptyprog}{\varepsilon}

\newcommand{\gPICK}{\textbf{pick}}
\newcommand{\gWHILE}{\textbf{while}}
\newcommand{\gDO}{\textbf{do}}

\newcommand{\gIF}{\textbf{if}}
\newcommand{\gTHEN}{\textbf{then}}
\newcommand{\gELSE}{\textbf{else}}

\newcommand{\gif}[3]{\gIF~#1~\gTHEN~#2~\gELSE~#3}
\newcommand{\gact}[2]{\gPICK~#1.#2}
\newcommand{\gwhile}[2]{\gWHILE~#1~\gDO~#2}

\newcommand{\gprogtrans}[1]{\xrightarrow{#1}}
\newcommand{\gfin}{\mathbb{F}}

\newcommand{\gexectrans}{\ra}

\newcommand{\pid}{pid}
\newcommand{\tpid}{\tau_{id}}

\newcommand{\final}[1]{#1\in\gfin}

 \newcommand{\ask}{\Ans}
\newcommand{\tell}{\textsc{tell}}
\newcommand{\filter}{f}

\newcommand{\progres}{\textsc{res}} 

\newcommand{\tkabs}{\tau_S} 

\newcommand{\ppre}{pre} 
\newcommand{\ppost}{post}

\newcommand{\mimic}{\cong}

\newcommand{\tgkab}{\tau_{\gkabsym}} 
\newcommand{\tgprog}{t_{\gkabsym}} 

\newcommand{\eqm}{\simeq}

\newcommand{\nnf}{\textsc{nnf}}

\newcommand{\setinvocation}{\Lambda}

\newcommand{\tmpconceptname}{\mathsf{State}}
\newcommand{\tmpconst}{\mathit{temp}}
\newcommand{\tmp}{\tmpconceptname(\tmpconst)}

\newcommand{\pre}{\mathit{st}} 
\newcommand{\post}{\mathit{ed}} 
\newcommand{\flagconceptname}{\mathsf{Flag}} 
\newcommand{\flagconcept}[1]{\flagconceptname(#1)} 

\newcommand{\noopconceptname}{\mathsf{Noop}} 
\newcommand{\noopconcept}[1]{\noopconceptname(#1)} 

\newcommand{\tforb}{t_B} 
\newcommand{\tforj}{t_{j}} 
\newcommand{\tford}{t_{dup}} 

\newcommand{\inc}{\textsc{inc}}

\newcommand{\tgkabb}{\tau_B} 
\newcommand{\tgkabc}{\tau_C} 
\newcommand{\tgkabe}{\tau_E} 

\newcommand{\tgprogb}{\kappa_B} 
\newcommand{\tgprogc}{\kappa_C} 
\newcommand{\tgproge}{\kappa_E} 

\newcommand{\actsettmpa}{\actset_{\varepsilon}^+}
\newcommand{\actsettmpd}{\actset_{\varepsilon}^-}

\newcommand{\temp}{\mathsf{M}(\mathit{rep})}

\newcommand{\tap}[1]{[#1]} 

\newcommand{\voc}{\textsc{voc}}

\newcommand{\initABox}{A_0} 

\newcommand{\act}{\alpha} 
\newcommand{\actset}{\Gamma} 
\newcommand{\add}{\textbf{add \xspace}}
\newcommand{\del}{\textbf{del \xspace}}
\newcommand{\eff}[1]{\textsc{Eff}(#1)}
\newcommand{\facta}{F^+}
\newcommand{\factd}{F^-}

\newcommand{\procset}{\Pi} 

\newcommand{\kabsym}{\K}

\newcommand{\const}{\Delta} 
\newcommand{\iconst}{\const_0}

\newcommand{\servcall}{\F} 
\newcommand{\map}[2]{#1 \rightsquigarrow #2} 
\newcommand{\adom}[1]{\textsc{adom}(#1)} 
\newcommand{\carule}[2]{#1 \mapsto #2} 

\newcommand{\stateset}{\Sigma} 
\newcommand{\trans}{\Rightarrow} 
\newcommand{\abox}{\mathit{abox}} 

\newcommand{\ts}[1]{\varUpsilon_{#1}} 

\newcommand{\scmap}{\ensuremath{m}\xspace} 
\newcommand{\scset}{\ensuremath{\mathbb{SC}}\xspace} 

\newcommand{\doo}[1]{\textsc{do}(#1)} 
\newcommand{\calls}[1]{{\textsc{calls}({#1})}} 
\newcommand{\eval}[1]{{\textsc{eval}(#1)}} 
\newcommand{\exec}[1]{\textsc{exec}_{#1}\xspace} 
\newcommand{\exect}[1]{\xrightarrow[]{#1}} 

\newcommand{\addfactssym}{\textsc{add}}
\newcommand{\delfactssym}{\textsc{del}}

\newcommand{\addfactss}[2]{\addfactssym^{#1}_{#2}}
\newcommand{\delfactss}[2]{\delfactssym^{#1}_{#2}}

\newcommand{\domain}[1]{\ensuremath{\textsc{dom}(#1)}\xspace}

\newcommand{\jbsim}{\sim_{\textsc{j}}}

\newcommand{\sbsim}{\sim_{\textsc{so}}}

\newcommand{\lbsim}{\sim_{\textsc{l}}}

\newcommand{\ebsim}{\sim_{\textsc{e}}}

\newcommand{\qedb}{\hfill\ensuremath{\blacksquare}}

\newcommand{\arset}[1]{\textsc{b-rep}(#1)}

\newcommand{\iarset}[1]{\textsc{c-rep}(#1)}

\newcommand{\evol}{\textsc{evol}}

\newcommand{\qunsatf}{q^f_{\textnormal{unsat}}}
\newcommand{\qunsatn}{q^n_{\textnormal{unsat}}}
\newcommand{\qunsatecq}[1]{Q^#1_{\textnormal{unsat}}}

\newcommand{\muL}{\mu\L} 
\newcommand{\muladom}{\ensuremath{\muL_{A}^{{\textnormal{EQL}}}}\xspace}
\newcommand{\ladom}{\ensuremath{\L_{A}^{{\textnormal{EQL}}}}\xspace}

\newcommand{\mula}{\ensuremath{\muL_{A}}\xspace}

\newcommand{\vfo}{\ensuremath{v}} 
\newcommand{\vso}{\ensuremath{V}} 

\newcommand{\MOD}[1]{(#1)^{\ts{}}}
\newcommand{\MODA}[1]{(#1)_{\vfo,\vso}^{\ts{}}}
\newcommand{\MODAX}[2]{(#1)_{\vfo #2,\vso}^{\ts{}}}

\newcommand{\Ans}{\textsc{ans}}
\newcommand{\ans}{\mathit{ans}}

\newcommand{\conj}{\mathit{conj}}

\newcommand{\true}{\mathsf{true}}
\newcommand{\false}{\mathsf{false}}

\newcommand{\dllitea}{\textit{DL-Lite}\ensuremath{_{\mathcal{A}}}\xspace}

\newcommand{\funct}[1]{(\mathsf{funct}~#1)}

\usepackage{amsthm}

\newtheorem{theorem}{Theorem}
\newtheorem{lemma}[theorem]{Lemma}

\theoremstyle{definition}

\newtheorem{exampleAux}{Example}

\newtheorem{mydefinition}[theorem]{Definition}
\newenvironment{definition}{\begin{mydefinition}}{\null\hfill\qedb\smallskip\end{mydefinition}}

\def\qed{\hfill{\qedboxempty}} 
\def\qedboxempty{\vbox{\hrule\hbox{\vrule\kern3pt
\vbox{\kern3pt\kern3pt}\kern3pt\vrule}\hrule}}

\newcommand{\B}{\mathcal{B}}
 
\newcommand{\E}{\mathcal{E}} \newcommand{\F}{\mathcal{F}}
\newcommand{\G}{\mathcal{G}} 
\newcommand{\I}{\mathcal{I}} 
\newcommand{\K}{\mathcal{K}} \renewcommand{\L}{\mathcal{L}}

\newcommand{\Q}{\mathcal{Q}}

\newcommand{\ra}{\rightarrow}

\newcommand{\lra}{\leftrightarrow}

\newcommand{\Lora}{\Longrightarrow}

\newcommand{\Lola}{\Longleftarrow}

\newcommand{\per}{\mbox{\bf .}}                  

\newcommand{\set}[1]{\{#1\}}                      

\newcommand{\card}[1]{|{#1}|}                     

\newcommand{\tup}[1]{\langle #1\rangle}            

\newcommand{\dom}[1][\I]{\Delta^{#1}}  
\newcommand{\Int}[2][\I]{#2^{#1}}      

\newcommand{\SOMET}[1]{\exists #1}

\newcommand{\INV}[1]{#1^{-}}

\newcommand{\BOX}[1]{ [\!-\!] #1}
\newcommand{\DIAM}[1]{\langle \!-\! \rangle #1}

\newcommand{\ISA}{\sqsubseteq}

\newcommand{\citeasnoun}[1]{\citeauthor{#1}~\shortcite{#1}}

\begin{document}

\title{Verification of Generalized Inconsistency-Aware Knowledge and Action
 Bases
 \\(Extended Version)
}

\author{Diego Calvanese, Marco Montali, Ario Santoso\\
 KRDB Research Centre for Knowledge and Data\\
 Free University of Bozen-Bolzano\\
 \textit{lastname}@inf.unibz.it}

\maketitle

\begin{abstract}
  Knowledge and Action Bases (KABs) have been put forward as a semantically
  rich representation of a domain, using a DL KB to account for its static
  aspects, and actions to evolve its extensional part over time, possibly
  introducing new objects. Recently, KABs have been extended to manage
  inconsistency, with ad-hoc verification techniques geared towards specific
  semantics.  This work provides a twofold contribution along this line of
  research. On the one hand, we enrich KABs with a high-level, compact action
  language inspired by Golog, obtaining so called Golog-KABs (GKABs). On the
  other hand, we introduce a parametric execution semantics for GKABs, so as to
  elegantly accomodate a plethora of inconsistency-aware semantics based on the
  notion of repair.  We then provide several reductions for the verification of
  sophisticated first-order temporal properties over inconsistency-aware GKABs,
  and show that it can be addressed using known techniques, developed for
  standard KABs.
\end{abstract}

\section{Introduction}
\label{sec:introduction}

The combination of static and dynamic aspects in modeling complex
organizational domains is a challenging task that has received increased
attention, and has led to the study of settings combining formalisms from
knowledge representation, database theory, and process management
\cite{Hull08,Vian09,CaDM13}.  Specifically, Knowledge and Action Bases (KABs)
\cite{BCMD*13} have been put forward to provide a semantically rich
representation of a domain.  In KABs, static aspects are modeled using a
knowledge base (KB) expressed in the lightweight Description Logic (DL)
\cite{BCMNP03} \dllitea \cite{CDLLR07,CDLL*09}, while actions are used to
evolve its extensional part over time, possibly introducing fresh individuals
from the external environment.  An important aspect that has received little
attention so far in such systems is the management of inconsistency with
respect to domain knowledge that may arise when the extensional information is
evolved over time.  In fact, inconsistency is typically handled naively by just
rejecting updates in actions when they would lead to inconsistency.  This
shortcoming is not only present in KABs, but virtually in all related
approaches in the literature, e.g., \cite{DHPV09,BeLP12,BCDDM13}.

To overcome this limitation, KABs have been extended lately with mechanisms to
handle inconsistency \cite{CKMSZ13}.  However, this has been done by defining
ad-hoc execution semantics and corresponding ad-hoc verification techniques
geared towards specific semantics for inconsistency management.  Furthermore,
it has been left open whether introducing inconsistency management in the rich
setting of KABs, effectively leads to systems with a different level of
expressive power.
In this paper, we attack these issues by:
\begin{inparaenum}[\it (i)]
\item Proposing (standard) GKABs, which enrich KABs with a compact
  action language inspired by Golog \cite{LRLLS97} that can be
  conveniently used to specify processes at a high-level of
  abstraction. As in KABs, standard GKABs still manage inconsistency
  naively.
\item Defining a parametric execution semantic for GKABs that is able
  to elegantly accomodate a plethora of inconsistency-aware semantics
  based on the well-known notion of repair
  \cite{EiGo92,Bert06,LLRRS10,CKNZ10b}.
\item Providing several reductions showing that verification of
  sophisticated first-order temporal properties over
  inconsistency-aware GKABs can be recast as a corresponding
  verification problem over standard GKABs.
\item Showing that verification of standard and inconsistency-aware
  GKABs can be addressed using known techniques, developed for
  standard KABs.
\end{inparaenum}

\section{Preliminaries}
\label{sec:prelim}

We start by introducing the necessary technical preliminaries.

\subsection{\dllitea}
We fix a countably infinite set $\const$ of \textit{individuals}, acting as
standard names.  To model KBs, we use the lightweight logic \dllitea
\cite{CDLLR07,CDLL*09}, whose \emph{concepts} and \emph{roles} are built
according to $B ::= N \mid \SOMET{R}$ and $R ::= P \mid \INV{P}$, where
\begin{inparaenum}[]
\item $N$ is a \emph{concept name},
\item $B$ a \emph{basic concept},
\item $P$ a \emph{role name},
\item $\INV{P}$ an \emph{inverse role}, and
\item $R$ a \emph{basic role}.
\end{inparaenum}

A \emph{\dllitea KB} is a pair $\tup{T, A}$, where:
\begin{inparaenum}[\it (i)]
\item $A$ is an Abox, i.e., a finite set of \emph{ABox 
   assertions} (or \emph{facts}) of the form $ N(c_1)$ or $P(c_1,c_2)$, where
  $c_1$, $c_2$ are individuals.
\item $T=T_p\uplus T_n\uplus T_f$ is a \emph{TBox}, i.e., a finite set
  constituted by a subset $T_p$ of \emph{positive inclusion assertions} of the
  form $B_1\ISA B_2$ and $R_1\ISA R_2$, a subset $T_n$ of \emph{negative
   inclusion assertions} of the form $B_1 \ISA \neg B_2$ and $R_1 \ISA \neg
  R_2$, and a subset $T_f$ of \emph{functionality assertions} of the form
  $\funct{R}$.
\end{inparaenum}
We denote by $\adom{A}$ the set of individuals explicitly present in $A$.

We rely on the standard semantics of DLs based on FOL interpretations
$\I=(\dom,\Int{\cdot})$, where $\Int{c}\in\dom$, $\Int{N}\subseteq\dom$, and
$\Int{P}\subseteq\dom\times\dom$.  The semantics of the \dllitea constructs and
of TBox and ABox assertions, and the notions of \emph{satisfaction} and of
\emph{model} are as usual (see, e.g., \cite{CDLLR07}).
We say that $A$ is \emph{$T$-consistent} if $\tup{T,A}$ is
satisfiable, i.e., admits at least one model. We also assume that
all concepts and roles in $T$ are satisfiable, i.e., for every
concept $N$ in $T$, there exists at
least one model $\I$ of $T$ such that $\Int{N}$ is non-empty, and
similarly for roles.

\smallskip
\noindent
\textbf{Queries.}
We use queries to access KBs and extract individuals of interest.
A \emph{union of conjunctive queries  (UCQ)} $q$ over a KB $\tup{T, A}$ is a FOL
formula of the form
$\bigvee_{1\leq i\leq n}\exists\vec{y_i}\per\conj_i(\vec{x},\vec{y_i})$, where
each $\conj_i(\vec{x},\vec{y_i})$ is a conjunction of atoms, whose predicates
are either concept/role names of $T$, or equality assertions
involving variables $\vec{x}$ and $\vec{y}_i$, and/or individuals.
%

The \emph{(certain) answers} of $q$ over $\tup{T,A}$ are defined as
the set $\ans(q,T,A)$ of substitutions $\sigma$ of the
free variables in $q$ with inviduals in $\adom{A}$, such that
$q\sigma$ evaluates to true in every model of $\tup{T,A}$.
If $q$ has no free variables, then it is called \emph{boolean} and its certain
answers are either the empty substitution (corresponding to $\true$), or the
empty set (corresponding to $\false$).
We also consider the extension of UCQs named
\textit{EQL-Lite}(UCQ)~\cite{CDLLR07b} (briefly, ECQs), that is, the FOL query
language whose atoms are UCQs evaluated according to the certain answer
semantics above.
Formally, an \emph{ECQ} over a TBox $T$ is a (possibly open) formula of the
form\footnote{In this work we only consider domain independent ECQs.}:
\begin{center}
$  Q ~::=~  [q]  ~\mid~ \lnot Q ~\mid~ Q_1\land Q_2 ~\mid~
              \exists x\per Q
$
\end{center}
where $q$ is a UCQ,  and $[q]$ denotes the fact that $q$ is evaluated
under the (minimal) knowledge operator \cite{CDLLR07b}.\footnote{We
  omit the square brackets for single-atom UCQs.}
Intuitively, the \emph{certain answers $\Ans(Q,T,A)$ of an ECQ $Q$
  over $\tup{T,A}$} are obtained by computing the certain answers of
the UCQs embedded in $Q$, then composing such answers
through the FO constructs in $Q$ (interpreting existential variables as ranging over
$\adom{A}$).
%

%

\subsection{Inconsistency Management in DL KBs}
\label{sec:inconsistency}
Retrieving certain answers from a KB makes sense only if the KB is
consistent: if it is not, then each query returns all possible tuples of
individuals of the ABox. In a dynamic setting where the ABox
evolves over time, consistency is a too strong requirement, and in fact a number of approaches have been proposed to handle
the instance-level evolution of KBs, managing inconsistency when it
arises. Such approaches
typically follow one of the two following two strategies:
\begin{inparaenum}[\it (i)]
\item inconsistencies are kept in the KBs, but the semantics of query
  answering is refined to take this into account (\emph{consistent
    query answering} \cite{Bert06});
\item the extensional part of an
inconsistent KB is (minimally) \emph{repaired} so as to remove
inconsistencies, and certain answers are
then applied over the curated KB.
\end{inparaenum}
In this paper, we follow the approach in \cite{CKMSZ13}, and
consequently focus on repair-based approaches. However, our results
seamlessly carry over the setting of consistent query answering. We
then recall the basic notions related to inconsistency
management via repair, distinguishing approaches that repair
an ABox and those that repair an update.

\smallskip
\noindent
\textbf{ABox repairs.} Starting from the seminal work in
\cite{EiGo92}, in \cite{LLRRS10} two approaches for repairing KBs are proposed: \emph{ABox repair} (AR) and \emph{intersection ABox repair} (IAR). In \cite{CKMSZ13}, these approaches are used to handle
inconsistency in KABs, and are respectively called \emph{bold-repair} (\emph{b-repair}) and
\emph{certain-repair} (\emph{c-repair}).
Formally, a \emph{b-repair of an ABox $A$ w.r.t.\ a TBox $T$} is a
\emph{maximal T-consistent subset $A'$ of $A$}, i.e.:
\begin{inparaenum}[\it (i)]
\item $A' \subseteq A$,
\item $A'$ is $T$-consistent, and
\item there does not exists $A''$ such that $A' \subset A'' \subseteq
  A$ and  $A''$ is $T$-consistent.
\end{inparaenum}
We denote by $\arset{T,A}$ the set of all b-repairs of $\tup{T,A}$.
The \emph{c-repair of an ABox $A$ w.r.t.\ a TBox $T$} is the (unique) set
$\iarset{T,A} = \cap_{A_i \in\arset{T,A}} A_i$ of ABox assertions, obtained by
intersecting all b-repairs.

\smallskip
\noindent
\textbf{Inconsistency in KB evolution.}~ In a setting where the KB is
subject to instance-level evolution, b- and c-repairs are computed
agnostically from the updates: each update is committed, and only
secondly the obtained ABox is repaired if inconsistent. In
\cite{CKNZ10b}, a so-called \emph{bold semantics} is proposed to apply
the notion of repair to the update itself.  Specifically, the bold
semantics is defined over a consistent KB $\tup{T, A}$ and an
instance-level update that comprises two ABoxes $F^-$ and $F^+$,
respectively containing those assertions that have to be deleted from
and then added to $A$. It is assumed that $F^+$ is consistent with
$T$, and that new assertions have ``priority'': if an inconsistency
arises, newly introduced facts are preferred to those already present
in $A$.
Formally, the \emph{evolution of an ABox $A$ w.r.t.\ a TBox $T$ by $F^+$ and
 $F^-$}, written $\evol(T, A, F^+, F^-)$, is an ABox $A_e = F^+ \cup A'$, where
\begin{inparaenum}[\it (i)]
\item $A' \subseteq (A \setminus F^-)$,
\item $F^+ \cup A'$ is $T$-consistent, and
\item there does not exists $A''$ such that $A' \subset A'' \subseteq
  (A \setminus F^-)$ and  $F^+\cup A''$ is $T$-consistent.
\end{inparaenum}

\subsection{Knowledge and Action Bases}
\label{sec:kab}
Knowledge and Action Bases (KABs) \cite{BCMD*13} have been proposed as
a unified framework to simultaneously account for the static and
dynamic aspects of an application domain. This is done by combining a
semantically-rich representation of the domain (\asr{via}{through} a
DL KB), with a process that evolves the extensional part of such a KB,
possibly introducing, through service calls, new individuals from the
external world. We briefly recall the main aspects of KABs, by
combining the framework in \cite{BCMD*13} with the \asd{high-level}
action specification formalism in \cite{MoCD14}.

We consider a finite set of distinguished individuals $\iconst \subset \const$,
and a finite set $\servcall$ of \textit{functions} representing \textit{service
 calls}, which abstractly account for the injection of fresh individuals from
$\const$ into the system.  A \emph{KAB} is a tuple $\kabsym = \tup{T,
 \initABox, \actset, \procset}$ where:
\begin{inparaenum}[\it(i)]
\item $T$ is a \dllitea TBox that captures the intensional aspects of the domain of interest;
\item $\initABox$ is the initial \dllitea ABox, describing the
 initial configuration of data;
\item $\actset$ is a finite set of parametric actions that evolve the
  ABox;
\item $\procset$ is a finite set of condition-action rules forming a
  process, which describes when actions can be executed, and with
  which parameters.
\end{inparaenum}
We assume that $\adom{\initABox} \subseteq \iconst$.

An \textit{action} $\act \in \actset$ has the form
$\act(\vec{p}):\set{e_1,\ldots,e_m}$, where
\begin{inparaenum}[\it (i)]
\item $\act$ is the \emph{action name},
\item $\vec{p}$ are the \emph{input parameters}, and
\item $\set{e_1,\ldots,e_m}$ is the set of \emph{effects}.
\end{inparaenum}
 Each effect has
  the form $\map{Q(\vec{x})}{\add F^+, \del F^-}$, where:
  \begin{inparaenum}[\it (i)]
  \item $Q(\vec{x})$ is an ECQ, possibly mentioning individuals in
    $\iconst$ and action parameters $\vec{p}$.
  \item $F^+$ is a set of \asr{atoms}{facts} (over the alphabet of $T$) to be \emph{added} to
    the ABox, each having as terms: individuals in $\iconst$, action parameters
    $\vec{p}$, free variables $\vec{x}$ of $Q$, and service calls,
    represented as Skolem terms formed by
    applying a function $f \in \servcall$ to one of the previous
    kinds of terms.
  \item $F^-$ is a set of \asr{atoms}{facts} (over the alphabet of $T$) to be \emph{deleted}
    from the ABox, each having as terms: individuals in $\iconst$, input
    parameters $\vec{p}$, and free variables of $Q$.
  \end{inparaenum}
  We denote by $\eff{\act}$ the set of effects in $\act$.
  Intuitively, action $\alpha$ is executed by grounding its
  parameters, and then applying its effects in parallel. Each effect
  instantiates the \asr{atoms}{facts} mentioned in its head with all
  answers of $Q$, then issues the corresponding service calls possibly
  contained in $F^+$, and substitutes them with the obtained results
  (which are individuals from $\const$). The update induced by
  $\alpha$ is produced by adding and removing the ground
  \asr{atoms}{facts} so-obtained to/from the current ABox, giving
  higher priority to additions.

  %


The \emph{process} $\procset$ comprises a finite set of
\emph{condition-action rules} of the form $Q(\vec{x})\mapsto\alpha(\vec{x})$,
where:
\begin{inparaenum}[]
\item $\act\in\actset$ is an action, and
\item $Q(\vec{x})$ is an ECQ over $T$, whose terms are free variables
  $\vec{x}$, quantified variables, and individuals in $\iconst$.
\end{inparaenum}
Each condition-action rule determines the instantiations
of parameters with which to execute the action in its head over the
current ABox.

The execution semantics of a KAB
 is given in terms of a possibly infinite-state \emph{transition
   system}, whose construction depends on the adopted semantics of
 inconsistency \cite{CKMSZ13}. In general, the transition systems we consider are
of the form $\tup{\const,T,\stateset,s_0,\abox,\trans}$, where:
\begin{inparaenum}[\it (i)]
\item $T$ is a \dllitea TBox;
\item $\stateset$ is a (possibly infinite) set of states;
\item $s_0 \in \stateset$ is the initial state;
\item $\abox$ is a function that, given a state $s\in\stateset$,
  returns an ABox associated to $s$;
\item ${\Rightarrow} \subseteq \Sigma\times\Sigma$ is a transition
  relation between pairs of states.
\end{inparaenum}

Following the terminology in \cite{CKMSZ13}, we call S-KAB a
KAB under the standard execution semantics of KABs, where
inconsistency is naively managed by simply rejecting those updates
that lead to an inconsistent state.  The transition system
$\ts{\kabsym}^{S}$ accounting for the standard execution semantics of
KAB $\kabsym$ is then constructed by starting from the initial ABox,
applying the executable actions in all possible ways, and generating
the (consistent) successor states by applying the corresponding
updates, then iterating through this procedure. As for the semantics
of service calls, in line with \cite{CKMSZ13} we adopt the
\emph{deterministic semantics}, i.e., services return always the same
result when called with the same inputs. Nondeterministic services can
be seamlessly added without affecting our technical results.

To ensure that services behave deterministically, the states of the
transition system are also equipped with a service call map that
stores the service calls issued so far, and their corresponding
results. Technically, a \emph{service call map} is a partial function
$\scmap:\scset\ra\const$, where  $\scset = \{\mathit{sc}(v_1,\ldots,v_n) \mid
\mathit{sc}/n \in \servcall \textrm{ and } \{v_1,\ldots,v_n\} \subseteq \const
\}$ is the set of (Skolem terms
representing) service calls.

\subsection{Verification Formalism}\label{sec:VerificationFormalism}
To specify sophisticated temporal properties to be verified over KABs, taking
into account the system dynamics as well as the evolution of data over time, we
rely on the $\muladom$ logic, the FO variant of the $\mu$-calculus defined in
\cite{BCMD*13}.  $\muladom$ combines the standard temporal operators of the
$\mu$-calculus with EQL queries over the states. FO quantification is
interpreted with an active domain semantics, i.e., it ranges over those
individuals that are explicitly present in the current ABox, and fully
interacts with temporal modalities, i.e., it applies \emph{across} states. The
\muladom syntax is:
\begin{center}
$
  \Phi ~:=~ Q ~\mid~ \lnot \Phi ~\mid~ \Phi_1 \land \Phi_2
  ~\mid~ \exists x.\Phi ~\mid~ \DIAM{\Phi} ~\mid~ Z ~\mid~ \mu Z.\Phi
$
\end{center}
where $Q$ is a possibly open EQL query that can make use of the distinguished
individuals in $\iconst$, $Z$ is a second-order variable denoting a predicate
(of arity 0), $\DIAM{\Phi}$ indicates the existence of a next state where
$\Phi$ holds, and $\mu$ is the least fixpoint operator, parametrized with the
free variables of its bounding formula. 
We make use of the following standard abbreviations: $\forall x.  \Phi = \neg
(\exists x.\neg \Phi)$, $\Phi_1 \lor \Phi_2 = \neg (\neg\Phi_1 \land
\neg \Phi_2)$, $\BOX{\Phi} = \neg \DIAM{\neg \Phi}$, and $\nu Z. \Phi
= \lnot\mu Z. \neg \Phi[Z/\neg Z]$.

For the semantics of \muladom, which is given over transition systems of the form
specified in Section~\ref{sec:kab}, we refer to \cite{BCMD*13}.
Given a transition system $\ts{}$ and a closed \muladom  formula
$\Phi$, we call \emph{model checking} verifying whether
 $\Phi$ holds in the initial state of $\ts{}$, written $\ts{} \models \Phi$.


\section{Golog-KABs and Inconsistency}
\label{sec:gkab}
In this section, we leverage on the KAB framework
(cf.~Section~\ref{sec:kab}) and provide a twofold contribution. On the
one hand, we enrich KABs with a high-level action
language inspired by Golog \cite{LRLLS97}.
 This allows modelers to
represents processes much more compactly, and will be instrumental for
the reductions discussed in Sections~\ref{sec:compilation} and~\ref{sec:gtos}.
On the other hand, we introduce a parametric execution semantics,
which elegantly accomodates a plethora of inconsistency-aware semantics based
on the notion of repair.

A \emph{Golog-KAB (GKAB)} is a tuple $\gkabsym =
\tup{T,\initABox,\actset,\ginitprog}$, where
$T$, $\initABox$, and $\actset$ are as in standard KABs, and
$\ginitprog$ is the Golog program characterizing the evolution of
the GKAB over time, using the atomic actions in $\actset$.
For simplicity, we only consider a core fragment\footnote{The other
  Golog constructs, including non-deterministic iteration and unrestricted pick, can be
 simulated with the constructs considered here.} of Golog based on the action
language in \cite{CDLR11}, and define a \emph{Golog program} as:
%
%
\begin{center}
$\begin{array}{@{}r@{\ }l@{\ }}
  \delta ::= &
  \gemptyprog ~\mid~
  \gact{Q(\vec{p})}{\act(\vec{p})} ~\mid~
  \delta_1|\delta_2  ~\mid~
  \delta_1;\delta_2 ~\mid~ \\
  &\gif{\varphi}{\delta_1}{\delta_2} ~\mid~
  \gwhile{\varphi}{\delta}
\end{array}
$
\end{center}
where:
\begin{inparaenum}[(1)]
\item $\gemptyprog$ is the \emph{empty program};
\item $\gact{Q(\vec{p})}{\act(\vec{p})}$ is an \emph{atomic action invocation}
  guarded by an ECQ $Q$, such that $\act\in\actset$ is applied by
  non-deterministically substituting its parameters $\vec{p}$ with an
  answer of $Q$;
\item $\delta_1|\delta_2$ is a \emph{non-deterministic choice} between
  programs;
\item $\delta_1;\delta_2$ is \emph{sequencing};
\item $\gif{\varphi}{\delta_1}{\delta_2}$ and
  $\gwhile{\varphi}{\delta}$ are \emph{conditional} and
  \emph{loop} constructs, using a boolean ECQ $\varphi$ as condition.
\end{inparaenum}

\smallskip
\noindent
\textbf{Execution Semantics.}~As for normal KABs, the execution semantics of a
GKAB \asr{$\gkabsym$}{$\gkabsym = \tup{T,\initABox,\actset,\ginitprog}$} is given in terms of a
possibly infinite-state transition system $\ts{\gkabsym}$, whose states are
labelled with ABoxes.  The states we consider, are tuples of the form
$\tup{A,\scmap,\delta}$, where $A$ is an ABox, $\scmap$ a service call map, and
$\delta$ a program.  Together, $A$ and $\scmap$ constitute the
\emph{data-state}, which captures the result of the actions executed so far,
together with the answers returned by service calls issued in the
past. Instead, $\delta$ is the \emph{process-state}, which represents the
program that still needs to be executed from the current data-state.

We adopt the functional approach by \citeasnoun{Leve84} in defining the
semantics of action execution over $\gkabsym$, i.e., we assume $\gkabsym$
provides two operations:
\begin{inparaenum}[\it (i)]
\item $\textsc{ask}$, to
answer queries over the current KB;
\item  $\tell$, to update the KB through an atomic action.
\end{inparaenum}
Since we adopt repairs to handle inconsistency, the $\textsc{ask}$
operator corresponds to certain answers computation.

We proceed now to formally define $\tell$.
Given an action invocation $\gact{Q(\vec{p})}{\act(\vec{p})}$ and an
ABox $A$, we say that substitution $\sigma$ of
parameters $\vec{p}$ with individuals in $\const$ is \emph{legal
  for $\alpha$ in $A$} if
$\ask(Q\sigma, T,A)$ is $\true$. If so, we also say that
\emph{$\act\sigma$ is executable in $A$}, and we define
the sets
of \asr{atoms}{facts} to be added and deleted by
$\gact{Q(\vec{p})}{\act(\vec{p})}$ with $\sigma$ in $A$ as follows:
$
\begin{array}{@{}r@{~}c@{~}l}
  \addfactss{A}{\act\sigma} &=&
  \bigcup_{(\map{Q}{\add F^+, \del F^-}) \text{ in } \eff{\act}}
  \bigcup_{\rho\in\ask(Q\sigma,T,A)}
  F^+\sigma\rho\\
  \delfactss{A}{\act\sigma} &=&
  \bigcup_{(\map{Q}{\add F^+, \del F^-}) \text{ in } \eff{\act}}
  \bigcup_{\rho\in\ask(Q\sigma,T,A)}
  F^-\sigma\rho
\end{array}
$

In general, $\addfactss{A}{\act\sigma}$ is not a proper set
 of facts, because it could contain (ground) service calls, to be
 substituted with corresponding results.
 We denote by $\calls{\addfactss{A}{\act\sigma}}$ the set of ground
 service calls in $\addfactss{A}{\act\sigma}$, and by
 $\eval{\addfactss{A}{\act\sigma}}$ the set of call substitutions with
 individuals in $\const$, i.e., the set
%
\begin{center}
$
 \{ \theta \mid \theta \mbox { is a total function, }
                        \theta: \calls{\addfactss{A}{\act\sigma}} \ra \const \}
$
\end{center}
%

%
%

Given two ABoxes $A$ and $A'$ where $A$ is assumed to be $T$-consistent, and two sets $F^+$ and $F^-$ of facts,
we introduce a so-called \emph{filter relation} to indicate that $A'$
is obtained from $A$ by adding the $F^+$ facts and removing the
$F^-$ ones.
To account for inconsistencies, the filter could drop some
additional facts when producing $A'$.
%
%
Hence, a filter consists of tuples of the form
$\tup{A, F^+, F^-,
  A'}$ satisfying $\emptyset \subseteq A'
\subseteq ((A \setminus F^-) \cup F^+)$.  In this light, filter relations provide
an abstract mechanism to accommodate several inconsistency management
approaches.

We now concretize $\tell$ as follows. Given a
GKAB $\gkabsym$ and a
filter $\filter$, we define
$\tell_\filter$ as the following relation over pairs of data-states in
$\ts{\gkabsym}$: tuple $\tup{\tup{A,\scmap}, \act\sigma, \tup{A',
    \scmap'}} \in \tell_\filter$ if\\
\begin{inparaitem}[$\bullet$]
\item $\sigma$ is a legal parameter substitution for $\alpha$ in $A$, and\\
\item there exists $\theta \in \eval{\addfactss{A}{\act\sigma}}$ such that:
\begin{inparaenum}[\it(i)]
\item $\theta$ and $\scmap$ agree on the common values in their
  domains (this enforces the deterministic semantics for services);
\item $\scmap' = \scmap \cup \theta$;
\item $\tup{A, \addfactss{A}{\act\sigma}\theta, \delfactss{A}{\act\sigma}, A'} \in
  \filter$, where $\addfactss{A}{\act\sigma}\theta$ denotes the \asd{ground}
  set of facts obtained by applying $\theta$ over the \asr{atoms}{facts} in $\addfactss{A}{\act\sigma}$;
\item $A'$ is $T$-consistent.
\end{inparaenum}
\end{inparaitem}

As a last preliminary notion towards the parametric execution semantics of
GKABs, we specify when a state $\tup{A,\scmap,\delta}$ is considered to be
\emph{final} by its program $\delta$, written $\final{\tup{A,\scmap,\delta}}$.
This is done by defining the set $\gfin$ of final states as follows:
\begin{compactenum}[1.]
  \small
\item $\final{\tup{A, \scmap, \gemptyprog}}$;
\item $\final{\tup{A, \scmap, \delta_1|\delta_2}}$ if
  $\final{\tup{A, \scmap, \delta_1}}$ or
  $\final{\tup{A, \scmap, \delta_2}}$;
\item $\final{\tup{A, \scmap, \delta_1;\delta_2}}$ if
  $\final{\tup{A, \scmap, \delta_1}}$ and
  $\final{\tup{A, \scmap, \delta_2}}$;
\item $\final{\tup{A, \scmap, \gif{\varphi}{\delta_1}{\delta_2}}}$ \\ if
  $\ask(\varphi, T, A) = \true$, and
  $\final{\tup{A, \scmap, \delta_1}}$;
\item $\final{\tup{A, \scmap, \gif{\varphi}{\delta_1}{\delta_2}}}$ \\if
  $\ask(\varphi, T, A) = \false$, and
  $\final{\tup{A, \scmap, \delta_2}}$;
\item $\final{\tup{A, \scmap, \gwhile{\varphi}{\delta}}}$ if
  $\ask(\varphi, T, A) = \false$;
\item $\final{\tup{A, \scmap, \gwhile{\varphi}{\delta}}}$ if
  $\ask(\varphi, T, A) = \true$, and
  $\final{\tup{A, \scmap, \delta}}$.
\end{compactenum}

\noindent
Now, given a filter relation $\filter$, we define the \emph{program
  execution relation} $\gprogtrans{\alpha\sigma, \filter}$, describing
how an atomic action with parameters simultaneously evolves the data-
and program-state:
\begin{compactenum}[1.\!\!\!]
\small
\item $\tup{A, \scmap, \gact{Q(\vec{p})}{\act(\vec{p})}}
  \gprogtrans{\alpha\sigma, \filter}
  \tup{A', \scmap', \gemptyprog}$, \\if
$\tup{\tup{A, \scmap}, \act\sigma, \tup{A', \scmap'}} \in \tell_\filter$;
\item $\tup{A, \scmap, \delta_1|\delta_2} \gprogtrans{\alpha\sigma,
    \filter} \tup{A', \scmap', \delta'}$, \\if $\tup{A, \scmap,
    \delta_1} \!\gprogtrans{\alpha\sigma, \filter}\! \tup{A', \scmap',
    \delta'}$ or $\tup{A, \scmap, \delta_2} \gprogtrans{\alpha\sigma,
    \filter} \tup{A', \scmap', \delta'}$;
\item $\tup{A, \scmap, \delta_1;\delta_2} \gprogtrans{\alpha\sigma,
    \filter} \tup{A', \scmap', \delta_1';\delta_2}$, \\if $\tup{A,
    \scmap, \delta_1} \gprogtrans{\alpha\sigma, \filter} \tup{A',
    \scmap', \delta_1'}$;
\item $\tup{A, \scmap, \delta_1;\delta_2} \gprogtrans{\alpha\sigma,
    \filter} \tup{A',
    \scmap', \delta_2'}$, \\
  if $\final{\tup{A, \scmap, \delta_1}}$, and $\tup{A, \scmap,
    \delta_2} \gprogtrans{\alpha\sigma, \filter} \tup{A', \scmap',
    \delta_2'}$;
\item $\tup{A, \scmap, \gif{\varphi}{\delta_1}{\delta_2}}
  \gprogtrans{\alpha\sigma, \filter}
  \tup{A', \scmap', \delta_1'}$, \\
  if $\ask(\varphi, T, A) = \true$, and $\tup{A, \scmap, \delta_1}
  \gprogtrans{\alpha\sigma, \filter} \tup{A', \scmap', \delta_1'}$;
\item $\tup{A, \scmap, \gif{\varphi}{\delta_1}{\delta_2}} \gprogtrans{\alpha\sigma, \filter} \tup{A', \scmap', \delta_2'}$,\\
  if $\ask(\varphi, T, A) = \false$, and $\tup{A, \scmap, \delta_2}
  \gprogtrans{\alpha\sigma, \filter} \tup{A', \scmap', \delta_2'}$;
\item $\tup{A, \scmap, \gwhile{\varphi}{\delta}} \gprogtrans{\alpha\sigma, \filter} \tup{A', \scmap', \delta';\gwhile{\varphi}{\delta}}$,\\
  if $\ask(\varphi, T, A) = \true$, and $\tup{A, \scmap, \delta}
  \gprogtrans{\alpha\sigma, \filter} \tup{A', \scmap', \delta'}$.
\end{compactenum}

Given a GKAB $\gkabsym = \tup{T,
  \initABox, \actset, \ginitprog}$ and a filter relation $\filter$, we finally
define the \emph{transition system of $\gkabsym$ w.r.t.~$\filter$},
written $\ts{\gkabsym}^{\filter}$, as
$\tup{\const,T,\stateset,s_0,\abox,\trans}$, where
\begin{inparaenum}[]
\item $s_0 = \tup{\initABox, \emptyset, \ginitprog}$, and
\item $\stateset$ and $\trans$ are defined by simultaneous induction as the
  smallest sets such that $s_0 \in \stateset$, and if $\tup{A, \scmap, \delta} \in \stateset$ and
    $\tup{A, \scmap, \delta} \gprogtrans{\alpha\sigma, \filter} \tup{A', \scmap', \delta'}$,
    then $\tup{A',\scmap',\delta'}\in\stateset$
    and $\tup{A, \scmap, \delta}\trans \tup{A', \scmap', \delta'}$.
\end{inparaenum}
By suitbably concretizing the filter relation, we can obtain a
plethora of
execution semantics.

\smallskip
\noindent
\textbf{Standard and Inconsistency-Aware Semantics.}
Given a GKAB $\gkabsym = \tup{T, \initABox, \actset, \ginitprog}$, we exploit filter relations to define its
standard execution semantics (reconstructing that of \cite{CKMSZ13}
for normal KABs), and three inconsistency-aware semantics that
incorporate the repair-based approaches reviewed in Section~\ref{sec:inconsistency}.
In particular, we introduce 4 filter relations $\filter_S$, $\filter_B$, $\filter_C$, $\filter_E$, as
follows. Given an ABox $A$, an atomic action $\act(\vec{p}) \in
\actset$, a legal parameter substitution $\sigma$
for $\act$ in $A$, and a service call evaluation $\theta \in
\eval{\addfactss{A}{\act\sigma}}$, let $F^+ =
\addfactss{A}{\act\sigma}\theta$ and $F^- =
\delfactss{A}{\act\sigma}$. We then have $\tup{A, F^+, F^-, A'} \in
\filter$, where
$
\begin{cases}
  A' = (A \setminus F^-) \cup F^+,
  & \text{if } \filter = \filter_S\\
  A' \in \arset{T, (A \setminus F^-) \cup F^+},
  & \text{if } \filter = \filter_B\\
  A' = \iarset{T, (A \setminus F^-) \cup F^+},
  & \text{if } \filter = \filter_C\\
  A' = \evol(T, A, F^+, F^-), & \text{if } \begin{array}[t]{@{}l@{}}
    \filter = \filter_E \text{ and }\\[-1mm]
    F^+ \text{ is $T$-consistent}
             \end{array}
\end{cases}
$
Filter $\filter_S$ gives rise to the \emph{standard execution
  semantics} for $\gkabsym$, since it just applies the update induced
by the ground atomic action $\alpha\sigma$ (giving priority to
additions over deletions). Filter $f_B$ gives rise to the
\emph{b-repair execution semantics} for $\gkabsym$, where inconsistent
ABoxes are repaired by non-deterministically picking a b-repair.
Filter $f_C$ gives rise to the \emph{c-repair execution semantics} for
$\gkabsym$, where inconsistent ABoxes are repaired by computing their
unique c-repair. Filter $f_E$ gives rise to the \emph{b-evol execution
  semantics} for $\gkabsym$, where for updates leading to inconsistent
ABoxes, their unique bold-evolution is computed.  We call the GKABs
adopting these semantics \emph{S-GKABs}, \emph{B-GKABs},
\emph{C-GKABs}, and \emph{E-GKABs}, respectively, and
we group the last three forms of GKABs under the umbrella of
\emph{inconsistency-aware GKABs (I-GKABs)}.

\smallskip
\noindent
\textbf{Transforming S-KABs to S-GKABs.}
We close this section by showing that our S-GKABs are able to capture
normal S-KABs in the literature \cite{BCMD*13,CKMSZ13}.
In particular, we show the following.

\begin{theorem}
\label{thm:stog}
Verification of \muladom properties over S-KABs can be recast as verification
over S-GKABs.
\end{theorem}
\begin{proof}[Proof sketch]
We provide a translation $\tkabs$ that, given an S-KAB $\kabsym = \tup{T,
 \initABox, \actset, \procset}$ with transition system $\ts{\kabsym}^{S}$,
generates an S-GKAB $\tkabs(\kabsym) = \tup{T, \initABox, \actset,
 \ginitprog}$. 
Program
$\ginitprog$ is obtained from $\procset$ as $ \ginitprog = \gwhile{\true}{(a_1|a_2|\ldots|a_{\card{\procset}})}$,
where, for each condition-action rule $\carule{Q_i(\vec{x})}{\act_i(\vec{x})} \in
\procset$, we have $a_i = \gact{Q_i(\vec{x})}{\act_i(\vec{x})}$.
The translation produces a program that continues
forever to non-deterministically pick an executable action with
parameters (as specified by $\procset$), or stops if no
action is executable. It can be then proven directly that for every
\muladom property $\Phi$,
  $\ts{\kabsym}^{S} \models
  \Phi$~iff~$\ts{\tkabs(\kabsym)}^{\filter_S} \models \Phi$.
%
%
\end{proof}


\section{Compilation of Inconsistency Management}
\label{sec:compilation}

This section provides a general account of inconsistency management in GKABs,
proving that all inconsistency-aware variants introduced in
Section~\ref{sec:gkab} can be reduced to S-GKABs.

\begin{theorem}
\label{thm:itos}
Verification of \muladom properties over I-GKABs can be recast as verification over S-GKABs.
\end{theorem}

The remainder of this section is devoted to prove this result, case by
case. Our general strategy is to show that S-GKABs are sufficiently expressive to
incorporate the repair-based approaches of
Section~\ref{sec:inconsistency}, so that an action executed
under a certain inconsistency
semantics can be compiled into a Golog program that applies the action
with the standard semantics,
and then explicitly handles the inconsistency, if needed.


We start by recalling that checking whether a \dllitea KB $\tup{T,A}$ is
inconsistent is FO rewritable, i.e., can be reduced to evaluating a boolean
query $\qunsatecq{T}$ over $A$ (interpreted as a database) \cite{CDLLR07}. To
express such queries compactly, we make use of the following abbreviations. For
role $R = P^-$, atom $R(x,y)$ denotes $P(y,x)$. For concept $B = \exists P$,
atom $B(x)$ denotes $P(x,\_)$, where `$\_$' stands for an anonymous
existentially quantified variable. Similarly, for $B = \exists P^-$, atom
$B(x)$ denotes $P(\_,x)$.

In particular, the boolean query $\qunsatecq{T}$ is:
\begin{center}
$
\qunsatecq{T} =\begin{array}[t]{@{}l@{}}
  \bigvee_{ \funct{R} \in T } \exists x,y,z.\qunsatf(\funct{R}, x,
  y, z)   \lor{}\\
\bigvee_{T \models B_1 \ISA \neg B_2} \exists x.\qunsatn(B_1 \ISA \neg B_2, x)  \lor{}\\
 \bigvee_{T \models R_1 \ISA \neg R_2} \exists x,y.\qunsatn(R_1 \ISA \neg R_2, x, y)
\end{array}
$
\end{center}
where:
\begin{compactitem}
\item $\qunsatf(\funct{R},x,y,z) = R(x,y) \land R(x,z) \land \neg [y =
z]$;
\item $\qunsatn(B_1 \ISA \neg B_2, x) = B_1(x) \land B_2(x)$;
\item  $\qunsatn(R_1 \ISA \neg R_2,x,y) = R_1(x,y) \land
R_2(x,y)$.
\end{compactitem}



\subsection{From B-GKABs to S-GKABs}\label{BGKABToSGKAB}

To encode B-GKABs into S-GKABs, we use a special fact $\temp$ to distinguish
\emph{stable} states, where an atomic action can be applied, from intermediate
states used by the S-GKABs to incrementally remove inconsistent facts from the
ABox.  Stable/repair states are marked by the absence/presence of
$\temp$. 
To set/unset $\temp$, we
define set $\actset_{rep} = \set{ \act^+_{rep}(), \act^-_{rep}()}$
of actions, where $\act^+_{rep}() :\set{\map{\true}{\add \set{\temp}}}$, and
$\act^-_{rep}() :\set{\map{\true}{\del \set{\temp}}}$.

Given a B-GKAB $\gkabsym = \tup{T, \initABox, \actset, \ginitprog}$, we
define the \emph{set $\actset_b^T$ of b-repair actions} and the set
$\setinvocation_b^T$ of \emph{b-repair atomic action invocations} as follows.
For each functionality assertion $\funct{R} \in T$, we include in
$\actset_b^T$ and $\setinvocation_b^T$ respectively:
\begin{compactitem}
\item   $\gact{\exists z.\qunsatf(\funct{R}, x, y, z)}{\act_{F}(x,y)} \in
  \setinvocation_b^T$, and 
\item  $\act_{F}(x,y)\!:\!\set{\map{R(x,z) \wedge \neg
      [z = y]}{\del \set{R( x, z)}}} \in \actset_b^T $
\end{compactitem}
This invocation repairs an inconsistency related to $\funct{R}$ by
removing all tuples causing the inconsistency, except one. 
For each negative concept inclusion $B_1 \ISA \neg B_2$ s.t.\ $T \models B_1
\ISA \neg B_2$, we include in $\actset_b^T$ and $\setinvocation_b^T$
respectively:
  \begin{compactitem}
  \item $\gact{\qunsatn(B_1 \ISA \neg B_2, x)}{\act_{B_1}(x)} \in
    \setinvocation_b^T$, and 
  \item     $\act_{B_1}(x):\set{\map{\true}{\del \set{B_1(x)}}} \in
    \actset_b^T$
  \end{compactitem}
This invocation repairs an inconsistency related to $B_1 \ISA \neg
B_2$ by removing an individual that is both in $B_1$ and $B_2$ from
$B_1$. Similarly for negative role inclusions.
Given $\setinvocation_b^T  = \set{a_1,\ldots,a_n}$, we then define the \emph{b-repair program}
\begin{center}
  $\delta_b^T  =\gwhile{\qunsatecq{T}}{(a_{1}|a_{2}|\ldots|a_n)}$,
\end{center}
Intuitively, $\delta_b^T$ iterates while the ABox is inconsistent, and
at each iteration, non-deterministically picks one of the sources of
inconsistency, and removes one or more facts causing it. Consequently,
the loop is guaranteed to terminate, in a state that corresponds to
one of the b-repairs of the initial ABox.

 With this machinery at hand, we are ready to define a translation
 $\tgkabb$ that, given $\gkabsym$, produces  S-GKAB
 $\tgkabb(\gkabsym) = \tup{T_p, \initABox, \actset \cup \actset_b^T
   \cup \actset_{rep}, \ginitprog'}$, where only the positive
 inclusion assertions $T_p$ of
 the original TBox $T$ are maintained (guaranteeing that
 $\tgkabb(\gkabsym)$ never encounters inconsistency). Program $\ginitprog'$ is obtained
 from program $\ginitprog$ of $\gkabsym$ by replacing each
 occurrence of an
atomic action invocation $\gact{Q(\vec{p})}{\act(\vec{p})}$ with
\begin{center}
  $
\gact{Q(\vec{p})}{\act(\vec{p})} ; \gact{\true}{\act^+_{rep}()} ;
 \delta^T_b ; \gact{\true}{\act^-_{rep}()}
$
\end{center}
This program concatenates the original action invocation with a corresponding
``repair'' phase. Obviously, this means that when an inconsistent ABox is
produced, a single transition in $\gkabsym$ corresponds to a sequence of
transitions in $\tgkabb(\gkabsym)$. Hence, we need to introduce a translation
$\tforb$ that takes a \muladom formula $\Phi$ over $\gkabsym$ and produces a
corresponding formula over $\tgkabb(\gkabsym)$. This is done by first
obtaining formula $\Phi' = \nnf(\Phi)$, where $\nnf(\Phi)$ denotes the
\emph{negation normal form} of $\Phi$. Then, every subformula of $\Phi$ of the
form $\DIAM{\Psi}$ becomes $\DIAM{\DIAM{\mu Z.((\temp \wedge \DIAM{Z}) \vee
  (\neg\temp \wedge \tforb(\Psi)))}}$, so as to translate a next-state
condition over $\gkabsym$ into reachability of the next stable state over
$\tgkabb(\gkabsym)$. Similarly for $\BOX{\Psi}$.

With these two translations at hand, we can show that $\ts{\gkabsym}^{\filter_B}
\models \Phi$ iff $\ts{\tgkabb(\gkabsym)}^{\filter_S} \models \tforb(\Phi)$.

\subsection{From C-GKABs to S-GKABs}\label{CGKABToSGKAB}
Making inconsistency management for C-GKABs explicit requires just a
single action, which removes all individuals that are involved in some
form of inconsistency. Hence, given a TBox $T$, we define a 0-ary
\emph{c-repair action $\act^T_c$}, where
$\eff{\act^T_c}$ is the smallest set containing the following effects:
\begin{compactitem}
\item for each assertion $\funct{R} \in T$,\\
$
  \map{\qunsatf(\funct{R}, x, y, z)}
  {\set{\del \set{R(x, y),R(x, z)}}}
$
\item for each assertion $B_1 \ISA \neg B_2$ s.t.\ $T
  \models B_1 \ISA \neg B_2$,\\
$
  \map{\qunsatn(B_1 \ISA \neg B_2, x)}
 {\set{\del \set{B_1(x), B_2(x)} }};
$
\item similarly for negative role inclusions.
\end{compactitem}
Notice that all effects are guarded by queries that extract only
individuals involved in an inconsistency. Hence, other facts are kept
unaltered, which also means that $\act^T_c$ is a no-op when applied
over a $T$-consistent ABox.
We define a translation $\tgkabc$ that, given a C-GKAB $\gkabsym = \tup{T,
 \initABox, \actset, \ginitprog}$, generates an S-GKAB $\tgkabc(\gkabsym) =
\tup{T_p, \initABox, \actset \cup \set{\act^T_c}, \ginitprog'}$, which, as for
B-GKABs, only maintains positive inclusion assertions of $T$. Program
$\ginitprog'$ is obtained from $\ginitprog$ by replacing each occurrence of
  an atomic action invocation of the form
  $\gact{Q(\vec{p})}{\act(\vec{p})}$ with
  $\gact{Q(\vec{p})}{\act(\vec{p})};\gact{\true}{\act^T_c()}$. This
  attests that each transition in $\gkabsym$ corresponds to a sequence
  of two transitions in $\tgkabc(\gkabsym)$: the first mimics the action
  execution, while the second computes the c-repair of the obtained
  ABox. 

  A \muladom property $\Phi$ over $\gkabsym$ can then be recast as a
  corresponding property over $\tgkabc(\gkabsym)$ that \asd{simply}
  substitutes each subformula $\DIAM{\Psi}$ of $\Phi$ with
  $\DIAM{\DIAM{\Psi}}$ (similarly for $\BOX{\Phi}$). By \asr{denoting
    this translation with $\tford$,} {denoting with $\tford$ this
    formula translation,} we \asr{get}{obtain
  that} $\ts{\gkabsym}^{\filter_C} \models \Phi$ iff
  $\ts{\tgkabc(\gkabsym)}^{\filter_S} \models \tford(\Phi)$.




\subsection{From E-GKABs to S-GKABs}\label{EGKABToSGKAB}
Differently from the case of B-GKABs and C-GKABs, E-GKABs pose two
challenges:
\begin{inparaenum}[\it (i)]
\item  when applying an atomic action (and managing the possibly
  arising inconsistency) it is necessary to
distinguish those facts that are newly introduced by the action from
those already present in the system;
\item the evolution semantics can be applied only if the facts to be added are
  consistent with the TBox, and hence an additional check is required to abort
  the action execution if this is not the case.
\end{inparaenum}
To this aim, given a TBox $T$, we duplicate concepts and roles in $T$,
introducing a fresh concept name $N^{n}$ for every concept name $N$ in $T$
(similarly for roles). The key idea is to insert those individuals that are
added to $N$ also in $N^{n}$, so as to trace that they are part of the update.

The first issue described above is then tackled by compiling the bold evolution
semantics into a 0-ary \emph{evolution action} $\act^T_e$, where
$\eff{\act^T_e}$ is the smallest set of effects containing:
\begin{compactitem}[$\bullet$]
\item for each assertion $\funct{R} \in T$,\\
$
  \map{\exists z.\qunsatf(\funct{R}, x, y, z) \land R^{n}(x, y)}{\set{\del \set{R(x, z)}}}
$
\item for each assertion $B_1 \!\ISA \neg B_2$ s.t.\ $T \models B_1 \ISA \neg B_2$,\\
$\map{\qunsatn(B_1 \ISA \neg B_2, x) \land B_1^{n}(x)}{\set{\del \set{B_2(x)}}} $;
%
%
%
\item similarly for negative role inclusion assertions;
\item for each concept name $N$, $\map{N^{n}(x)}{\set{\del \set{N^{n}(x)}}}$;
\item similarly for role names.
%
\end{compactitem}
These effects mirror those of Section~\ref{CGKABToSGKAB}, with the
difference that they asymmetrically remove old facts when
inconsistency arises. The last two bullets guarantee that the
content of concept and role names tracking the newly added facts are flushed.
We then define a translation $\tgkabe$ that, given an
E-GKAB $\gkabsym = \tup{T,
  \initABox, \actset, \ginitprog}$, generates an S-GKAB $\tgkabe(\gkabsym) = \tup{T_p \cup T^n, \initABox,
  \actset' \cup \set{\act^T_e}, \ginitprog'}$, where:
\begin{compactitem}[$\bullet$]
\item $T^n$ is obtained from $T$ by renaming each concept name $N$
in $T$ into $N^n$ (similarly for roles). In this way, the original
concepts/roles are only subject in $\tgkabe(\gkabsym)$ to the
positive inclusion assertions of $T$, while concepts/roles
tracking newly inserted facts are subject also to negative constraints. This blocks the
generation of the successor state when the facts to be added to the current ABox are $T$-inconsistent.
\item $\actset'$ is obtained by translating each action in $\act(\vec{p}) \in
  \actset$ into action $\act'(\vec{p})$, such that for each effect
  $\map{Q}{\add F^+, \del F^-} \in \eff{\act}$, we have $\map{Q}{\add
    F^+ \cup {F^+}^n,  \del F^-} \in \eff{\act'}$ where ${F^+}^n$
  duplicates $F^+$ by using the vocabulary for newly introduced facts.
%
\item $\ginitprog'$ is obtained from $\ginitprog$ by replacing each
  action invocation $\gact{Q(\vec{p})}{\act(\vec{p})}$
  with $\gact{Q(\vec{p})}{\act'(\vec{p})};\gact{\true}{\act^T_e()}$.
  \end{compactitem}
By exploiting the same \muladom translation used in
Section~\ref{CGKABToSGKAB}, we obtain that  $\ts{\gkabsym}^{\filter_E} \models \Phi$ iff
     $\ts{\tgkabe(\gkabsym)}^{\filter_S} \models \tford(\Phi)$.





\section{From Golog to Standard KABs}
\label{sec:gtos}

We close our tour by showing that S-GKABs can be compiled into the normal
S-KABs of \cite{BCMD*13,CKMSZ13}.

\begin{theorem}
\label{thm:gtos}
Verification of \muladom properties over S-GKABs can be recast as
verification over S-KABs.
\end{theorem}
\begin{proof}[Proof sketch]
  \asr{We}{As before, we} 
  introduce a translation from S-GKABs to S-KABs, and from \muladom
  properties over S-GKABs to corresponding properties over S-KABs, in
  such a way that verification in the first setting can be reduced to
  verification in the second setting. The translation is quite
  involved,
\asr{for space reasons, we refer to
  Section \ref{sec:translation-sgkab-to-skab} in the Appendix for details.}{ so for space reasons we
  only sketch it here, referring to the online appendix for details.}
%
\asd{ The key idea is to inductively interpret a Golog program as a
  structure consisting of nested processes, suitably composed through
  the Golog operators.  We mark the starting and ending point of each
  Golog subprogram, and use accessory facts in the ABox to track
  states corresponding to subprograms.  Each subprogram is then
  inductively translated into a set of condition-action rules encoding
  its entrance and termination conditions.}
%
\end{proof}
From Theorems~\ref{thm:stog} and \ref{thm:gtos}, we 
obtain that  S-KABs and S-GKABs are expressively equivalent.
From Theorems~\ref{thm:itos} and~\ref{thm:gtos}, we get
our second major result: inconsistency-management can be
compiled into an S-KAB by concatenating the two translations
from I-GKABs to S-GKABs, and then to S-KABs.
\begin{theorem}
\label{thm:itoverys}
 Verification of \muladom properties over I-GKABs can be recast as
verification over S-KABs.
\end{theorem}
Even more interesting is the fact that the semantic property of
\emph{run-boundedness} \cite{BCDDM13,BCMD*13} is preserved by all
translations presented in this paper. Intuitively, run-boundedness
requires that every run of the system cumulatively encounters at most
a bounded number of individuals. Unboundedly many individuals can
still be present in the overall system, provided that they do not
accumulate in the same run.  Thanks to the preservation of
run-boundedness, and to the compilation of I-GKABs into S-KABs, we
get:
\begin{theorem}
  Verification of \muladom properties over run-bounded I-GKABs is
  decidable, and reducible to standard $\mu$-calculus finite-state
  model checking.
\end{theorem}\label{thm:i-to-kab}
\begin{proof}[Proof sketch]
  \asr{The claim follows by combining the fact that all translations
    preserve run-boundedness, Theorem~\ref{thm:itoverys}, and the
    results in \cite{BCDDM13,BCMD*13} for run-bounded S-KABs.}{The
    translation from I-GKABs to S-GKABs preserves run-boundedness,
    since the actions introduced to manage inconsistency never inject
    new individuals, but only remove facts causing
    inconsistency. Run-boundedness is also preserved from S-GKABs to
    S-KABs, since only a bounded number of new individuals are
    introduced, when emulating the Golog program with condition-action
    rules. The claim follows by combining Theorem~\ref{thm:itoverys}
    with the results in \cite{BCDDM13,BCMD*13} for run-bounded
    S-KABs.}
\end{proof}


\section{Conclusion}

We introduced GKABs, which extend KABs with Golog-inspired high-level programs,
and provided a parametric execution semantics supporting an elegant treatment
of inconsistency. We have shown that verification of rich temporal properties
over (inconsistency-aware) GKABs can be recast as verification over standard
KABs, by encoding the semantics of inconsistency in terms of Golog programs and
specific inconsistency-management actions, and Golog programs into standard KAB
condition-action rules. An overview of our reductions is depicted below. Our
approach is very general, and can be seamlessly extended to account for other
mechanisms for handling inconsistency, and more in general data cleaning.
\begin{center}
  \includegraphics[width=0.30\textwidth]{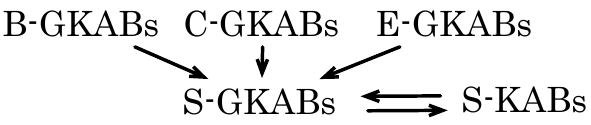}
\end{center}


{
\smallskip
\noindent
\textbf{Acknowledgments.} 
This research has been partially supported by the EU IP project Optique
(\emph{Scalable End-user Access to Big Data}), grant agreement
n.~FP7-318338, and by the UNIBZ internal project KENDO
(\emph{Knowledge-driven ENterprise Distributed cOmputing}).
}

{\small
\bibliographystyle{named}

\bibliography{main-bib}
}

\clearpage
\appendix
\section{Some Additional Basic Notions and Notation Conventions}


Given a function $f$, we often write $\tap{a \ra b}$ when $f(a) = b$
(i.e., $f$ maps $a$ to $b$). We write $\domain{f}$ to denote the
domain of $f$.

Given a substitution $\sigma$, we write $ x/c \in \sigma $ if
$\sigma(x) = c$, i.e., $\sigma$ maps $x$ into $c \in \const$ (or
sometimes we also say $\sigma$ substitutes $x$ with $c \in \const$).
We write $\sigma[x/c]$ to denote a new substitution obtained from
$\sigma$ such that $\sigma[x/c](x) = c$ and
$\sigma[x/c](y) = \sigma(y)$ (for $y \neq x$).

We call the set of concept and role names that appear in TBox $T$ a
\emph{vocabulary of TBox $T$}, denoted by $\voc(T)$.  W.l.o.g. given a
TBox $T$, we assume that $\voc(T)$ contains all possible concept and
role names. Notice that we can simply add an assertion
$N \sqsubseteq N$ (resp. $P \sqsubseteq P$) into the TBox $T$ in order
to add a concept name $N$ (resp. role name $P$) inside $\voc(T)$ such
that $\voc(T)$ contains all possible concept and role names, and
without changing the expected set of models of the TBox $T$ (hence,
preserving the deductive closures of $T$).
Moreover, we call an ABox \emph{$A$ is over $\voc(T)$} if it consists
of ABox assertions of the form either $N(o_1)$ or $P(o_1, o_2)$ where
$N, P \in \voc(T)$.

We now define some abbreviations for ABox assertions that we will use
later in order to have a compact presentation.

\begin{definition}[Abbreviations for ABox assertion]\label{def:abbreviation-abox-assertion}
  We define some notations to compactly express various ABox
  assertions as follows:
  \begin{compactitem}
  \item Given a TBox assertion $B \sqsubseteq B_1$ or $B \sqsubseteq \neg B_1$, an assertion
    $B(c)$ denotes
    \begin{compactitem}
    \item $N(c)$ if $B = N$,
    \item $P(c,c')$ if $B = \exists P$,
    \item $P(c',c)$ if $B = \exists P^-$,
    \end{compactitem}
    where `$c'$' is a constant;
  \item Given a TBox assertion $R \sqsubseteq R_1$ or
    $R \sqsubseteq \neg R_1$, an assertion $R(c_1, c_2)$ denotes
    \begin{compactitem}
    \item $P(c_1, c_2)$ if $R = P$,
    \item $P(c_2, c_1)$ if $R = P^-$.
    \end{compactitem}
  \end{compactitem}
  \ \
\end{definition}

Given a Golog program $\delta$ we define the notion of sub-programs of
$\delta$ as follows.

\begin{definition}[Sub-program]
  Given a program $\delta$, we define the notion of a
  \emph{sub-program} of $\delta$ inductively as follows:
  \begin{compactitem}
  \item $\delta$ is a sub-program of $\delta$,
  \item If $\delta$ is of the form $\delta_1|\delta_2$,
    $\delta_1;\delta_2$, or $\gif{\varphi}{\delta_1}{\delta_2}$, then
    \begin{compactitem}
    \item $\delta_1$ and $\delta_2$ both are sub-programs of $\delta$,
    \item each sub-program of $\delta_1$ is a sub-program of $\delta$,
    \item each sub-program of $\delta_2$ is a sub-program of $\delta$,
    \end{compactitem}
  \item  If $\delta$ is of the form $\gwhile{\varphi}{\delta_1}$
    \begin{compactitem}
    \item $\delta_1$ is a sub-program of $\delta$,
    \item each sub-program of $\delta_1$ is a sub-program of $\delta$,
    \end{compactitem}
  \end{compactitem}
  \ \
\end{definition}
\noindent
We say a program \emph{$\delta'$ occurs in $\delta$} if $\delta'$ is a
sub-program of $\delta$.

Given an action invocation $\gact{Q(\vec{p})}{\act(\vec{p})}$ and an
ABox $A$, when we have a substitution $\sigma$ 
is \emph{legal for $\act$ in $A$}, we often also say that $\sigma$
is a legal parameter assignment for $\act$ in $A$.




We now proceed to define the notion of a program execution trace as
well as the notion when such a trace is called terminating. Moreover,
we also define the notion of program execution result in the case of
terminating program execution trace.

\begin{definition}[Program Execution Trace]
  Let
  $\ts{\gkabsym}^{\filter} = \tup{\const, T, \stateset, s_0, \abox,
    \trans}$
  be the transition system of a GKAB
  $\gkabsym = \tup{T, \initABox, \actset, \ginitprog}$. Given a state
  $\tup{A_1, \scmap_1, \delta_1}$, a \emph{program execution trace
    $\pi$ induced by $\delta$ on $\tup{A_1, \scmap_1, \delta_1}$
    w.r.t.\ filter $\filter$} is a (possibly infinite) sequence of
  states of the form
\[
\pi = \tup{A_1, \scmap_1, \delta_1} \gexectrans \tup{A_2, \scmap_2,
  \delta_2} \gexectrans \tup{A_3, \scmap_3, \delta_3} \gexectrans
\cdots
\]
s.t.  $\tup{A_i, \scmap_i, \delta_i} \gprogtrans{\alpha_i\sigma_i,
  \filter} \tup{A_{i+1}, \scmap_{i+1}, \delta_{i+1}}$ for $i\geq1$.
\ \ 
\end{definition}

\begin{definition}[Terminating Program Execution Trace]
  Let
  $\ts{\gkabsym}^{\filter} = \tup{\const, T, \stateset, s_0, \abox,
    \trans}$
  be the transition system of a GKAB
  $\gkabsym = \tup{T, \initABox, \actset, \ginitprog}$. Given a state
  $\tup{A_1, \scmap_1, \delta_1}$, and a program execution trace $\pi$
induced by $\delta_1$ on $\tup{A_1, \scmap_1, \delta_1}$, we call
\emph{$\pi$ terminating} if
\begin{compactenum}[(1)]
\item $\tup{A_1, \scmap_1, \delta_1}$ is a final state, or
\item if $\tup{A_1, \scmap_1, \delta_1}$ is not a final state, then there
  exists a state $\tup{A_n, \scmap_n, \delta_n}$ such that we have the
  following finite program execution trace
\[
  \pi = 
  \tup{A_1,\scmap_1, \delta_1} \gexectrans \tup{A_2,\scmap_2,
    \delta_2} \gexectrans \cdots \gexectrans \tup{A_n, \scmap_n,
    \delta_n},
\]
where $\tup{A_i, \scmap_i, \delta_i}$ (for
$i \in \set{1,\ldots, n-1}$) are not final states, and
$\tup{A_n, \scmap_n, \delta_n}$ is a final state.
\end{compactenum}
In the situation (1) (resp.\ (2)), we call the ABox $A_1$ (resp.\
$A_n$) \emph{the result of executing $\delta_1$ on
  $\tup{A_1, \scmap_1, \delta_1}$} w.r.t.\ filter
$\filter$. Additionally, we also say that $\pi$ is the \emph{program
  execution trace that produces $A_1$ (resp.\ $A_n$)}.
\end{definition}
We write $\progres(A_1, \scmap_1, \delta_1)$ to denote the set of all
ABoxes that is the result of executing $\delta_1$ on
$\progres(A_1, \scmap_1, \delta_1)$ w.r.t.\ filter $\filter$. Note
that given a state $\tup{A_1, \scmap_1, \delta_1}$, it is possible to
have several terminating program execution traces.
%
%
Intuitively, a program execution trace is a sequence of states which
captures the computation of the program as well as the evolution of
the system states by the program. Additionally, it is terminating if
at some point it reaches a final state.

For a technical reason, we also reserve some fresh concept names
$\flagconceptname$, $\noopconceptname$ and $\tmpconceptname$ (i.e.,
they are outside of any TBox vocabulary), and they are not allowed to
be used in any temporal properties (i.e., in \muladom or \mula
formulas). We call them \emph{special marker concept names}.
Additionally, we make use the constants in $\const_0$ to populate
them.  We call \emph{special marker} an ABox assertion that is
obtained by applying either $\flagconceptname$, $\noopconceptname$ or
$\tmpconceptname$ to a constant in $\const_0$.  Additionally, we call
\emph{flag} a special marker formed by applying either concept name
$\flagconceptname$ or $\noopconceptname$ to a constant in $\const_0$.
Later on, we use flags as markers to impose a certain sequence of
action executions, and we use a special marker $\tmp$ (where
$\tmpconst \in \const_0$) to mark an \emph{intermediate state}.

\subsection{Inconsistency Management Related Notions}


In this section we introduce some notions related to inconsistency.
Below we introduce the notion of negative inclusion assertion
(resp. functionality assertions) violation.

\begin{definition}[Violation of a Negative Inclusion Assertion]
  Let $\tup{T,A}$ be a KB, and $T \models B_1 \sqsubseteq \neg B_2$.
  We say \emph{$B_1 \sqsubseteq \neg B_2$ is violated} if there exists
  a constant $c$ such that $\set{B_1(c), B_2(c)} \subseteq A$. In this
  situation, we also say that \emph{$B_1(c)$ (resp. $B_2(c)$) violates
    $B_1 \sqsubseteq \neg B_2$}. Similarly for roles.
\end{definition}

\begin{definition}[Violation of a Functionality Assertion]
  Let $\tup{T,A}$ be a KB, and $\funct{R} \in T$.  We say
  \emph{$\funct{R}$ is violated} if there exists constants
  $c, c_1, c_2$ such that $\set{R(c,c_1), R(c,c_2)} \subseteq A$ and
  $c_1 \neq c_2$. In this situation, we also say that \emph{$R(c,c_1)$
    (resp. $R(c,c_2)$) violates $\funct{R}$}.
\end{definition}

Next, we define the notion of a set of inconsistent ABox assertions as
follows.

\begin{definition}[Set of Inconsistent ABox Assertions]
  Given a KB $\tup{T, A}$, we define the set $\inc(A)$ containing all
  ABox assertions that participate in the inconsistencies w.r.t.\
  $T$
  as the smallest set satisfying the following:
  \begin{compactenum}
  \item For each negative inclusion assertion $B_1 \sqsubseteq \neg B_2$ s.t.\
    $T \models B_1 \sqsubseteq \neg B_2$, we have
    $B_1(c) \in \inc(A)$, if $B_1(c)$ 
    violates $B_1 \sqsubseteq \neg B_2$, 
%
  \item For each negative inclusion assertion $R_1
    \sqsubseteq \neg R_2$ s.t.\ $T \models R_1 \sqsubseteq \neg
    R_2$, we have $R_1(c_1, c_2) \in \inc(A)$ if $R_1(c_1,
    c_2)$ violates $R_1 \sqsubseteq \neg R_2$, 
%
  \item For each functional assertion $\funct{R} \in T$, we have
    $R(c_1, c_2) \in \inc(A)$, if $R(c_1, c_2)$ violates
    $\funct{R}$. 
%
  \end{compactenum}
\ \ 
\end{definition}

\begin{lemma}\label{lem:card-inc-abox}
  Given a TBox $T$ and an ABox $A$, we have $\card{\inc(A)} = 0$ if
  and only if $A$ is $T$-consistent.
\end{lemma}
\begin{proof}
  Trivially follows from the definition. Since there is no ABox
  assertion violating any functionality or negative inclusion
  assertions.
\end{proof}

\subsection{History Preserving $\mu$-calculus (\muladom)}

This section briefly explains the history preserving $\mu$-calculus
(\muladom) (defined in~\cite{BCMD*13}) as an additional explanation
w.r.t.\ the explanation in Section~\ref{sec:VerificationFormalism}.

The semantics of \muladom formulae is defined over transition systems
$\ts{} = \tup{\const, T,\Sigma,s_0,\abox,{\Rightarrow}}$.  Since \muladom
contains formulae with both individual and predicate free variables,
given a transition system $\ts{}$, we introduce:
\begin{compactenum}[1.]
\item An individual variable valuation $\vfo$, i.e., a mapping from
  individual variables $x$ to $\const{}$. 
\item A predicate variable valuation $\vso$, i.e., a mapping from the
  predicate variables $Z$ to a subset of $\Sigma$.
\end{compactenum}

Given an individual variable valuation $\vfo$, we write
$ x/c \in \vfo$ if $\vfo(x) = c$, i.e., $\vfo$ maps $x$ into
$c \in \const$ (or sometimes we also say $\vfo$ substitutes $x$ with
$c \in \const$).  We write $\vfo[x/c]$ to denote a new individual
variable valuation obtained from $\vfo$ such that $\vfo[x/c](x) = c$
and $\vfo[x/c](y) = \vfo (y)$ (for $y \neq x$). We use similar
notation for predicate variable valuations.

We assign meaning to \muladom formulas by associating to $\ts{}$,
$\vfo$ and $\vso$ an \emph{extension function} $\MODA{\cdot}$, which
maps \muladom formulas to subsets of $\Sigma$.
The extension function $\MODA{\cdot}$ is defined inductively as
follows:
\[
  \begin{array}{r@{\ }l@{\ }l@{\ }l}
    \MODA{Q} & = &\{s \in \Sigma\mid \Ans(Q\vfo, T, \abox(s)) = \mathit{true}\}\\
   \MODA{\exists x. \Phi} & =&\{s \in \Sigma \mid\exists d. d \in
   \adom{\abox(s)} \\
    &&\mbox{ and } s \in \MODAX{\Phi}{[x/d]}\}\\
%
%
    \MODA{Z}  & = & V(Z) \subseteq  \Sigma\\
    \MODA{\lnot \Phi} & = & \Sigma - \MODA{\Phi}\\
    \MODA{\Phi_1 \land \Phi_2}  & = & \MODA{\Phi_1}\cap\MODA{\Phi_2}\\
%
%
    \MODA{\DIAM{\Phi}}  & = &  \{ s \in \Sigma \mid\exists s'.\ s
    \Rightarrow s' \mbox{ and } s' \in \MODA{\Phi}\}\\
%
%
    \MODA{\mu Z.\Phi} & = &\bigcap\{ \E\subseteq  \Sigma \mid
    {\MODA{\Phi}}_{[Z/\E]} \subseteq\E \} \\
%
%
 \end{array}
\]
Beside the usual FOL abbreviations, we also make use of the following
ones: $\BOX{\Phi} = \lnot\DIAM{(\lnot\Phi)}$ and
$\nu Z.\Phi=\lnot\mu Z.\lnot\Phi[Z/\lnot Z]$.  Here, $Q\vfo$ stands
for the query obtained from $Q$ by substituting its free variables
according to $\vfo$.
%
%
When $\Phi$ is a closed formula, $\MODA{\Phi}$ does not depend on
$\vfo$ or $\vso$, and we denote the extension of $\Phi$ simply by
$\MOD{\Phi}$.  A closed formula $\Phi$ holds in a state $s \in \Sigma$
if $s \in \MOD{\Phi}$.  In this case, we write $\ts{},s \models \Phi$.
A closed formula $\Phi$ holds in $\ts{}$, briefly \emph{$\ts{}$
  satisfies $\Phi$}, 
if $\ts{},s_0\models \Phi$ (In this situation we write
$\ts{} \models \Phi$).
Given a GKAB $\gkabsym$, and a \muladom property $\Phi$, let
$\ts{\gkabsym}^\filter$ be the transition system of $\gkabsym$, we say
$\gkabsym$ satisfies $\Phi$ if $\ts{\gkabsym}^\filter$ satisfies
$\Phi$.

\subsection{S-KABs Execution Semantics}
As we need later in the proof, here we briefly review the execution
semantics of S-KAB that we consider as described in the literature~\cite{BCMD*13,CKMSZ13} by also combining the framework with the action
specification formalism in~\cite{MoCD14}. The execution semantics of
an S-KAB is defined in terms of a possibly infinite-state transition
system.  Formally, given an S-KAB
$\kabsym = \tup{T,A_0, \actset, \procset}$, we define its semantics by
the \emph{transition system}
$\ts{\kabsym}^S = \tup{\const, T, \stateset, s_0, \abox, \trans}$,
where:
\begin{inparaenum}[\it (i)]
\item $T$ is a \dllitea TBox;
\item $\stateset$ is a (possibly infinite) set of states;
\item $s_0 \in \stateset$ is the initial state;
\item $\abox$ is a function that, given a state $s\in\stateset$,
  returns an ABox associated to $s$;
\item $\trans \subseteq \Sigma\times\Sigma$ is a transition
  relation between pairs of states.
\end{inparaenum}
Intuitively, the transitions system $\ts{\kabsym}^S$ S-KAB $\kabsym$ captures
all possible evolutions of the system by the actions in accordance with the
available condition-action rules.
%
Each state $s \in \stateset$ of the transition system $\ts{\kabsym}^S$ is a tuple
$\tup{A, \scmap}$, where $A$ is an ABox and $\scmap$ is a service call map.

The semantics of an \emph{action execution} is as follows:
Given a state $s = \tup{A, \scmap}$,
let $\act \in \actset$ be an action of the form
$\act(\vec{p}):\set{e_1,\ldots,e_m}$ with
$e_i = \map{Q(\vec{x})}{\add F^+, \del F^-}$, and let $\sigma$ be a
\emph{parameter substitution} for $\vec{p}$ with values taken from
$\const$.
%
%
We say that \emph{$\act$ is executable in $A$ with a parameter
  substitution $\sigma$}, if there exists a condition-action rule
$Q(\vec{x})\mapsto\alpha(\vec{x}) \in \procset$
s.t.\ $\Ans(Q\sigma, T, A)$ is $\true$.  In that case we call $\sigma$
a \emph{legal parameter assignment} for $\act$.%
The result of the application of $\act$ to an ABox $A$ using a
parameter substitution $\sigma$ is captured by the following function:
\[
\begin{array}{@{}l@{}l@{}}
  \doo{T, A, \act\sigma}= 
&\left(A\ \setminus\ \bigcup_{e_i \text{ in } \eff{\act}}
  \bigcup_{\rho\in\Ans(Q\sigma,T,A)}
  F^-\sigma\rho \right)  \\
  &\cup \left(\bigcup_{e_i
  \text{ in } \eff{\act}} \bigcup_{\rho\in\Ans(Q\sigma,T,A)} F^+\sigma\rho\right)
\end{array}
\] 
where $e_i = \map{Q(\vec{x})}{\add F^+, \del F^-}$

Intuitively, the result of the evaluation of $\act$ is obtained by
first deleting from $A$ the assertions that is obtained from the
grounding of the facts in $F^-$ and then adds the new assertions that
is obtained from the grounding of the facts in $F^+$. The grounding of
the facts in $F^+$ and $F^-$ are obtained from all the certain answers
of the query $Q(\vec{x})$ over $\tup{T,A}$.

The result of $\doo{T, A, \act\sigma}$ is in general not a proper ABox,
because it could contain (ground) Skolem terms, attesting that in order to
produce the ABox, some service calls have to be issued.  We denote by
$\calls{\doo{T, A, \act\sigma}}$ the set of such ground service calls, and
by $\eval{T, A,\act\sigma}$ the set of substitutions that replace such
calls with concrete values taken from $\const$. Specifically,
$\eval{T, A,\act\sigma}$ is defined as
\begin{tabbing}
$\eval{T,A,\act\sigma}= \{ \theta\ |\ $ \=
                        $\theta \mbox { is a total function }$\+\\
                        $\theta: \calls{\doo{T, A, \act\sigma}} \ra \const \}$.
\end{tabbing}
%

With all these notions in place, we can now recall the execution
semantics of a KAB $\kabsym = \tup{T, A_0, \actset, \procset}$. To do
so, we first introduce a transition relation $\exec{\kabsym}$ that
connects pairs of ABoxes and service call maps due to action
execution. In particular,
$\tup{\tup{A, \scmap},\act\sigma, \tup{A', \scmap'}} \in \exec{\kabsym}$ if
the following holds:
\begin{compactenum}
\item $\act$ is \emph{executable} in $A$ 
  with parameter substitution $\sigma$;
\item there exists $\theta \in \eval{T, A,\act\sigma}$ s.t. $\theta$ and $\scmap$
  ``agree'' on the common values in their domains (in order to realize the
  deterministic service call semantics);
\item $A' = \doo{T, A, \act\sigma}\theta$;
\item $\scmap' = \scmap \cup \theta$ (i.e., updating the history of
  issued service calls).
\end{compactenum}
For more intuitive notation, we write
$\tup{A, \scmap} \exect{\act\sigma} \tup{A', \scmap'}$ to denote
$\tup{\tup{A, \scmap}, \act\sigma, \tup{A', \scmap'}} \in
\exec{\kabsym}$.


The transition system $\ts{\kabsym}^S$ of $\kabsym$ is then defined as
$\tup{\const, T, \stateset, s_0, \abox, \trans}$ where
\begin{compactitem}
\item $s_0 = \tup{\initABox,\emptyset}$, and
\item $\stateset$ and $\trans$ are defined by simultaneous induction as the
  smallest sets satisfying the following properties:
  \begin{compactenum}[\it (i)]
  \item $s_0 \in \stateset$;
  \item if $\tup{A,\scmap} \in \stateset$, then for all actions $\act \in
    \actset$,
    for all substitutions $\sigma$ for the parameters of $\act$ and
    for all $\tup{A',\scmap'}$ s.t.
    %
    $\tup{A,\scmap}  \exect{\act\sigma} \tup{A',\scmap'}$
    and $A'$ is $T$-consistent,
    we have $\tup{A',\scmap'}\in\stateset$, $\tup{A,\scmap}\trans
    \tup{A',\scmap'}$.
  \end{compactenum}
\end{compactitem}
A \emph{run} of $\ts{\kabsym}$ is a (possibly infinite) sequence $s_0s_1\cdots$
of states of $\ts{\kabsym}$ such that $s_i\trans s_{i+1}$, for all $i\geq 0$.

\section{From S-KABs to S-GKABs}

This section is devoted to present the proof of Theorem
\ref{thm:stog}. The core idea is to show that our translation $\tkabs$
transforms S-KABs into S-GKABs such that their transition systems are
``equal'' (in the sense that they have the same structure and each
corresponding state contains the same ABox and service call map).
%
%
As a consequence, they should satisfy the same \muladom formulas. 

Technically, to formalize the notion of ``equality'' between
transition systems, we introduce the notion of
E-Bisimulation. Furthermore, we show that two E-bisimilar transition
systems can not be distinguished by \muladom properties. Then, to
provide the proof of Theorem \ref{thm:stog} we simply need to show
that $\tkabs$ transforms S-KABs into S-GKABs such that their
transition systems are E-bisimilar.


\subsection{E-Bisimulation}

We now define the notion of \emph{E-Bisimulation} and
show that two E-bisimilar transition systems can not be distinguished
by a \muladom formula.

\begin{definition}[E-Bisimulation] \ \\
  Let $\ts{1} = \tup{\const, T, \Sigma_1, s_{01}, \abox_1, \trans_1}$
  and $\ts{2} = \tup{\const, T, \Sigma_2, s_{02}, \abox_2, \trans_2}$
  be transition systems, with
  $\adom{\abox_1(s_{01})} \subseteq \const$ and
  $\adom{\abox_2(s_{02})} \subseteq \const$.  An \emph{E-Bisimulation}
  between $\ts{1}$ and $\ts{2}$ is a relation
  $\B \subseteq \Sigma_1 \times\Sigma_2$ such that
  $\tup{s_1, s_2} \in \B$ implies that:
  \begin{compactenum}
  \item $\abox_1(s_1) = \abox_2(s_2)$
  \item for each $s_1'$, if $s_1 \Rightarrow_1 s_1'$ then there exist
    $s_2'$ with $ s_2 \Rightarrow_2 s_2' $ such that
    $\tup{s_1', s_2'}\in\B$.
  \item for each $s_2'$, if $ s_2 \Rightarrow_2 s_2' $ then there
    exists $s_1'$ with $s_1 \Rightarrow_1 s_1'$, such that
    $\tup{s_1', s_2'}\in\B$.
 \end{compactenum}
\ \ 
\end{definition}

\noindent
Let $\ts{1} = \tup{\const, T, \Sigma_1, s_{01}, \abox_1, \trans_1}$
and $\ts{2} = \tup{\const, T, \Sigma_2, s_{02}, \abox_2, \trans_2}$ be
transition systems,
a state $s_1 \in \Sigma_1$ is \emph{E-bisimilar} to
$s_2 \in \Sigma_2$, written $s_1 \ebsim s_2$, if there exists an
E-Bisimulation $\B$ between $\ts{1}$ and $\ts{2}$ such that
$\tup{s_1, s_2}\in\B$.
The transition system $\ts{1}$ is \emph{E-bisimilar} to $\ts{2}$,
written $\ts{1} \ebsim \ts{2}$, if there exists an E-Bisimulation $\B$
between $\ts{1}$ and $\ts{2}$ such that $\tup{s_{01}, s_{02}}\in\B$.
%


\begin{lemma}\label{lem:e-bisimilar-ts-satisfies-same-formula}
  Consider two transition systems
  $\ts{1} = \tup{\const,T,\stateset_1,s_{01},\abox_1,\trans_1}$ and
  $\ts{2} = \tup{\const,T,\stateset_2,s_{02},\abox_2,\trans_2}$ such
  that $\ts{1} \ebsim \ts{2}$.  For every $\muladom$ closed formula
  $\Phi$, we have:
  \[
  \ts{1} \models \Phi \textrm{ if and only if } \ts{2} \models \Phi.
  \]
\end{lemma}
\begin{proof}
  The claim easily follows since two E-bisimilar transition systems
  are essentially equal in terms of the structure and the ABoxes that
  are contained in each bisimilar state.
\end{proof}

\subsection{Reducing the Verification of S-KABs to S-GKABs}

To reduce the verification of \muladom over S-KABs as verification
over S-GKABs, in this subsection we show that the transition system of
an S-KAB and the transition system of its corresponding S-GKAB are
E-bisimilar. Then, by using the result from the previous subsection we
can easily recast the verification problem and hence achieve our purpose.

\begin{lemma}\label{lem:skab-to-sgkab-bisimilar-state}
  Let $\kabsym$ be an S-KAB with transition system
  $\ts{\kabsym}^S$, and let $\tkabs(\kabsym)$ be an
  S-GKAB with transition system $\ts{\tkabs(\kabsym)}^{\filter_S}$
  obtain through $\tkabs$.
  Consider 
  \begin{inparaenum}[]
  \item a state $\tup{A_k,\scmap_k}$ of $\ts{\kabsym}^S$ and
  \item a state $\tup{A_g,\scmap_g, \delta_g}$ of
    $\ts{\tkabs(\kabsym)}^{\filter_S}$.
  \end{inparaenum}
  If $A_k = A_g$, and $\scmap_k = \scmap_g$, then
  $\tup{A_k,\scmap_k} \ebsim \tup{A_g,\scmap_g, \delta_g}$.
\end{lemma}
\begin{proof}
Let
\begin{compactenum}
\item $\kabsym = \tup{T, \initABox, \actset, \procset}$, and \\
  $\ts{\kabsym}^S = \tup{\const, T, \stateset_k, s_{0k}, \abox_k,
    \trans_k}$,
\item $\tkabs(\kabsym) = \tup{T, \initABox, \actset, \ginitprog}$, and \\
  $\ts{\tkabs(\kabsym)}^{\filter_S} = \tup{\const, T, \stateset_g, s_{0g},
    \abox_g, \trans_g}$.
\end{compactenum}
To prove the lemma, we show that, for every state
$\tup{A_k', \scmap_k'}$ s.t.\
$\tup{A_k,\scmap_k} \trans_k \tup{A_k',\scmap_k'}$, there exists a state
$\tup{A'_g,\scmap'_g, \delta_g'}$ s.t.:
\begin{compactenum}
\item
  $\tup{A_g,\scmap_g, \delta_g} \trans_g \tup{A'_g,\scmap'_g,
    \delta_g'}$;
\item $A'_k = A_g'$;
\item $\scmap'_k = \scmap_g'$.
\end{compactenum}
By definition of $\ts{\kabsym}^S$, if
$\tup{A_k,\scmap_k} \trans \tup{A_k',\scmap_k'}$, then there exist
\begin{compactenum}
\item a condition action rule $\carule{Q(\vec{p})}{\act(\vec{p})}$, 
\item an action $\act \in \actset$ with parameters $\vec{p}$, 
\item an parameter substitution $\sigma$, and 
\item a substitution $\theta$.
\end{compactenum}
such that 
\begin{inparaenum}[\it (i)]
\item $\theta \in \eval{T,A_k,\act\sigma}$ and agrees with $\scmap_k$, 
\item $\act$ is executable in state $A_k$ with a parameter
  substitution $\sigma$,
\item $A_k' = \doo{T, A_k, \act\sigma}\theta$, and 
\item $\scmap_k' = \scmap_k \cup \theta$.
\end{inparaenum}

Now, since
$ \ginitprog = \gwhile{\true}{(a_1|a_2|\ldots|a_{\card{\procset}})}$,
and each $a_i$ is an action invocation obtained from a
condition-action rule in $\procset$, then there exists an action
invocation $a_i$ such that $a_i = \gact{Q(\vec{x})}{\act(\vec{x})}$.
Since $A_k = A_g$, and $\scmap_k = \scmap_g$, by considering how a
transition is created in the transition system of S-GKABs, it is easy
to see that there exists a state $\tup{A_g', \scmap_g', \delta_g'}$
such that
$\tup{A_g,\scmap_g, \delta_g} \trans_g \tup{A'_g,\scmap'_g,
  \delta_g'}$,
$A_g' = A_k'$, and $\scmap_g' = \scmap_k'$. Thus, the claim is proven.
\end{proof}

\begin{lemma}\label{lem:skab-to-sgkab-bisimilar-ts}
  Given an S-KAB $\kabsym$, we have
  $\ts{\kabsym}^S \ebsim \ts{\tkabs(\kabsym)}^{\filter_S}$
\end{lemma}
\begin{proof}
Let
\begin{compactenum}
\item $\kabsym = \tup{T, \initABox, \actset, \procset}$, and \\
  $\ts{\kabsym}^S = \tup{\const, T, \stateset_k, s_{0k}, \abox_k,
    \trans_k}$,
\item $\tkabs(\kabsym) = \tup{T, \initABox, \actset, \ginitprog}$, and \\
  $\ts{\tkabs(\kabsym)}^{\filter_S} = \tup{\const, T, \stateset_g, s_{0g},
    \abox_g, \trans_g}$.
\end{compactenum}
We have that $s_{0k} = \tup{A_0, \scmap_k}$ and
$s_{0g} = \tup{A_0, \scmap_g, \delta}$ where
$\scmap_k = \scmap_g = \emptyset$. Hence, by Lemma
\ref{lem:skab-to-sgkab-bisimilar-state}, we have
$s_{0k} \ebsim s_{0g}$. Therefore, by the definition of E-bisimulation
between two transition systems, we have
$\ts{\kabsym}^S \ebsim \ts{\tkabs(\kabsym)}^{\filter_S}$.
\end{proof}

Having Lemma~\ref{lem:skab-to-sgkab-bisimilar-ts} in hand, we can
easily show that the verification of \muladom over S-KABs can be
reduced to the verification of \muladom over S-GKABs by also making
use the result from the previous subsection.

\begin{theorem}\label{thm:skab-to-sgkab}
  Given an S-KAB $\kabsym$ and a closed $\muladom$ formula $\Phi$, we
  have $\ts{\kabsym}^S \models \Phi$ iff
  $\ts{\tkabs(\kabsym)}^{\filter_S} \models \Phi$.
\end{theorem} 
\begin{proof}
  By Lemma~\ref{lem:skab-to-sgkab-bisimilar-ts}, we have that
  $\ts{\kabsym}^S \ebsim \ts{\tkabs(\kabsym)}^{\filter_S}$.  Hence,
  the claim is directly follows from
  Lemma~\ref{lem:e-bisimilar-ts-satisfies-same-formula}.
\end{proof}

\smallskip
\noindent
\textbf{Proof of Theorem 1.} \\ The proof of is simply obtained since
we can translate S-KABs into S-GKABs using $\tkabs$ and then by making
use Theorem \ref{thm:skab-to-sgkab}, we basically reduce the
verification of S-KABs into S-GKABs.

\section{From S-GKABs to S-KABs}\label{sec:proof-sgkab-to-skab}

We dedicate this section to 
show that the verification of \muladom properties over
S-GKABs can be recast as verification over S-KABs which essentially
exhibit the proof of Theorem~\ref{thm:gtos}.
%
%
To this aim, technically we do the following:
\begin{compactenum}

\item We define a special bisimulation relation between two transition
  system namely \emph{jumping bisimulation}.

\item We define a generic translation $\tforj$ that takes a \muladom
  formula $\Phi$ in Negative Normal Form (NNF) as an input and
  produces a \muladom formula $\tforj(\Phi)$.

\item We show that two jumping bisimilar transition system can not be
  distinguished by any \muladom formula (in NNF) modulo the
  translation $\tforj$.

\item We define a generic translation $\tgkab$, that given an S-GKAB
  $\gkabsym$, produces an S-KAB $\tgkab(\gkabsym)$. The core idea of this
  translation is to transform the given program $\delta$ and the set
  of actions in S-GKAB $\gkabsym$ into a process (a set of
  condition-action rules) and a set of S-KAB actions, such that all
  possible sequence of action executions that is enforced by $\delta$
  can be mimicked by the process in S-KAB (which determines all possible
  sequence of action executions in S-KAB).

\item We show that the transition system of a GKAB $\gkabsym$ and the
  transition system of its corresponding S-KAB $\tgkab(\gkabsym)$
  (obtained through translation $\tgkab$) are bisimilar w.r.t.\ the
  jumping bisimulation relation.

\item Making use all of the ingredients above, we finally in the end
  show that a GKAB $\gkabsym$ satisfies a certain \muladom formula
  $\Phi$ if and only if its corresponding S-KAB $\tgkab(\gkabsym)$
  satisfies a \muladom formula $\tforj(\Phi)$.

\end{compactenum}

\subsection{Jumping Bisimulation (J-Bisimulation)}\label{sec:jumping-bisimulation}

As a start towards defining the notion of J-Bisimulation, we introduce
the notion of equality modulo flag between two ABoxes as follows:

\begin{definition}[Equal Modulo Special Markers]\label{def:equal-mod-markers}
  Given a TBox $T$, two ABoxes $A_1$ and $A_2$ over $\voc(T)$ that
  might contain special markers, 
  we say \emph{$A_1$ equal to $A_2$ modulo special markers}, written
  $A_1 \eqm A_2$ (or equivalently $A_2 \eqm A_1$),
  if the following hold:
  \begin{compactitem}
  \item For each concept name $N \in \voc(T)$ (i.e., $N$ is not a
    special marker concept name), we have a concept assertion
    $N(c) \in A_1$ if and only if a concept assertion $N(c) \in A_2$,
  \item For each role name $P \in \voc(T)$, we have a role assertion
    $P(c_1,c_2) \in A_1$ if and only if a role assertion
    $P(c_1,c_2) \in A_2$.
  \end{compactitem}
\ \ 
\end{definition}

\begin{lemma}\label{lem:equal-ABox-imply-equal-modulo-markers}
  $A_1 = A_2$ implies $A_1 \eqm A_2$.
\end{lemma}
\begin{proof}
  Trivially true from the definition of $A_1 \eqm A_2$ above (see
  Definition~\ref{def:equal-mod-markers}).
\end{proof}

\begin{lemma}\label{lem:ECQ-equal-ABox-modulo-markers}
  Given a GKAB $\gkabsym = \tup{T, \initABox, \actset, \ginitprog}$,
  two ABoxes $A_1$ and $A_2$ over $\voc(T)$ which might contain
  special markers, and an ECQ $Q$ over $\tup{T, \initABox}$ which does
  not contain any atoms whose predicates are special marker concept
  names. 
  If $A_1 \eqm A_2$, then $\Ans(Q, T, A_1) = \Ans(Q, T, A_2)$.
\end{lemma}
\begin{proof}
  Trivially hold since 
  without considering special markers, we have $A_1 = A_2$ (i.e., we
  have a concept assertion $N(c) \in A_1$ if and only if a concept
  assertion $N(c) \in A_2$, and we have a role assertion
  $P(c_1,c_2) \in A_1$ if and only if a role assertion
  $P(c_1,c_2) \in A_2$). Hence $\Ans(Q, T, A_1) = \Ans(Q, T, A_2)$.
\ \ 
\end{proof}

We now proceed to define the notion of \emph{jumping bisimulation} as
follows.
\begin{definition}[Jumping Bisimulation (J-Bisimulation)]
  Let $\ts{1} = \tup{\const, T, \Sigma_1, s_{01}, \abox_1, \trans_1}$
  and $\ts{2} = \tup{\const, T, \Sigma_2, s_{02}, \abox_2, \trans_2}$
  be transition systems, with
  $\adom{\abox_1(s_{01})} \subseteq \const$
  and $\adom{\abox_2(s_{02})} \subseteq \const$.
  A \emph{jumping bisimulation} (J-Bisimulation) between $\ts{1}$ and
  $\ts{2}$ is a relation $\B \subseteq \Sigma_1 \times\Sigma_2$ such
  that $\tup{s_1, s_2} \in \B$ implies that:
  \begin{compactenum}
  \item $\abox_1(s_1) \eqm \abox_2(s_2)$
  \item for each $s_1'$, if $s_1 \Rightarrow_1 s_1'$ then there exist
    $s_2'$, $t_1, \ldots ,t_n$ (for $n \geq 0$) with
    \[
    s_2 \Rightarrow_2 t_1 \Rightarrow_2 \ldots \Rightarrow_2 t_n \Rightarrow_2 s_2'
    \] 
   such that $\tup{s_1', s_2'}\in\B$,
    $\tmp \not\in \abox_2(s_2')$ and $\tmp \in \abox_2(t_i)$ for
    $i \in \set{1, \ldots, n}$.
  \item for each $s_2'$, if 
    \[
    s_2 \Rightarrow_2 t_1 \Rightarrow_2 \ldots \Rightarrow_2 t_n \Rightarrow_2 s_2'
    \] 
    (for $n \geq 0$) with $\tmp \in \abox_2(t_i)$ for
    $i \in \set{1, \ldots, n}$ and $\tmp \not\in \abox_2(s_2')$, then
    there exists $s_1'$ with $s_1 \Rightarrow_1 s_1'$, such that
    $\tup{s_1', s_2'}\in\B$.
 \end{compactenum}
\ \ 
\end{definition}

\noindent
Let $\ts{1} = \tup{\const, T, \Sigma_1, s_{01}, \abox_1, \trans_1}$
and $\ts{2} = \tup{\const, T, \Sigma_2, s_{02}, \abox_2, \trans_2}$ be
transition systems, 
a state $s_1 \in \Sigma_1$ is \emph{J-bisimilar} to
$s_2 \in \Sigma_2$, written $s_1 \jbsim s_2$, if there exists a
jumping bisimulation $\B$ between $\ts{1}$ and $\ts{2}$ such that
$\tup{s_1, s_2}\in\B$.
A transition system $\ts{1}$ is \emph{J-bisimilar} to $\ts{2}$,
written $\ts{1} \jbsim \ts{2}$, if there exists a jumping bisimulation
$\B$ between $\ts{1}$ and $\ts{2}$ such that
$\tup{s_{01}, s_{02}}\in\B$.

Now, we advance further to show that two J-bisimilar transition
systems can not be distinguished by any \muladom formula (in NNF)
modulo a translation $\tforj$ that is defined as follows:

\begin{definition}[Translation $\tforj$]\label{def:tforj}
  We define a \emph{translation $\tforj$} that transforms an arbitrary
  \muladom formula $\Phi$ (in NNF) into another \muladom formula
  $\Phi'$ inductively by recurring over the structure of $\Phi$ as
  follows:
\[
\begin{array}{@{}l@{}ll@{}}
  \bullet\ \tforj(Q) &=& Q \\

  \bullet\ \tforj(\neg Q) &=& \neg Q \\

  \bullet\ \tforj(\Q x.\Phi) &=& \Q x. \tforj(\Phi) \\

  \bullet\ \tforj(\Phi_1 \circ \Phi_2) &=& \tforj(\Phi_1) \circ \tforj(\Phi_2) \\

  \bullet\ \tforj(\circledcirc Z.\Phi) &=& \circledcirc Z. \tforj(\Phi) \\

  \bullet\ \tforj(\DIAM{\Phi}) &=& \DIAM{\mu Z.((\tmp \wedge \DIAM{Z})
                                   \vee \\
                                   &&\hspace*{21mm}(\neg \tmp \wedge \tforj(\Phi)))} \\

  \bullet\ \tforj(\BOX{\Phi}) &=& \BOX{\mu Z.((\tmp \wedge \BOX{Z} \wedge
                                  \DIAM{\top}) \vee \\ 
                     &&\hspace*{21mm} (\neg \tmp \wedge \tforj(\Phi)))}
\end{array}
\]
\noindent
where:
\begin{compactitem}
\item $\circ$ is a binary operator ($\vee, \wedge, \ra,$ or $\lra$),
\item $\circledcirc$ is least ($\mu$) or greatest ($\nu$) fix-point operator,
\item $\Q$ is forall ($\forall$) or existential ($\exists$)
  quantifier.
\end{compactitem}
\ \ 
\end{definition}
%

\begin{lemma}\label{lem:jumping-bisimilar-states-satisfies-same-formula}
  Consider two transition systems
  $\ts{1} = \tup{\const, T,\stateset_1,s_{01},\abox_1,\trans_1}$ and
  $\ts{2} = \tup{\const, T,\stateset_2,s_{02},\abox_2,\trans_2}$, with
  $\adom{\abox_1(s_{01})} \subseteq \const$ and
  $\adom{\abox_2(s_{02})} \subseteq \const$.  Consider two states
  $s_1 \in \stateset_1$ and $s_2 \in \stateset_2$ such that
  $s_1 \jbsim s_2$. Then for every formula $\Phi$ of $\muladom$ (in
  negation normal form), 
  and every valuations $\vfo_1$ and $\vfo_2$ that assign to each of
  its free variables a constant $c_1 \in \adom{\abox_1(s_1)}$ and
  $c_2 \in \adom{\abox_2(s_2)}$, such that $c_1 = c_2$, we have that
  \[
  \ts{1},s_1 \models \Phi \vfo_1 \textrm{ if and only if } \ts{2},s_2
  \models \tforj(\Phi) \vfo_2.
  \]
\end{lemma}
\begin{proof}
  The proof is then organized in three parts:
\begin{compactenum}[(1)]
\item We prove the claim for formulae of $\ladom$, obtained from
  $\muladom$ by dropping the predicate variables and the fixpoint
  constructs. $\ladom$ corresponds to a first-order variant of the
  Hennessy Milner logic, and its semantics does not depend on the
  second-order valuation.
\item We extend the results to the infinitary logic obtained by extending
  $\ladom$ with arbitrary countable disjunction.
\item We recall that fixpoints can be translated into this infinitary logic,
  thus proving that the theorem holds for $\muladom$.
\end{compactenum}

\smallskip
\noindent
\textbf{Proof for $\ladom$.}  We proceed by induction on the structure
of $\Phi$, without considering the case of predicate variable and of
fixpoint constructs, which are not part of $\ladom$.

\smallskip
\noindent
\textit{Base case:}
\begin{compactitem}
\item[\textbf{($\Phi = Q$)}.] Since $s_1 \jbsim s_2$, we have
  $\abox_1(s_1) \eqm \abox_2(s_2)$. Hence, since we also restrict that
  any \muladom formulas does not use special marker concept names, by
  Lemma~\ref{lem:ECQ-equal-ABox-modulo-markers}, we have
  $\Ans(Q, T, \abox_1(s_1)) = \Ans(Q, T, \abox_2(s_2))$.
  Hence, since $\tforj(Q) = Q$, for every valuations $\vfo_1$ and
  $\vfo_2$ that assign to each of its free variables a constant
  $c_1 \in \adom{\abox_1(s_1)}$ and $c_2 \in \adom{\abox_2(s_2)}$,
  such that $c_1 = c_2$, we have
  \[
  \ts{1},s_1 \models Q \vfo_1 \textrm{ if and only if } \ts{2},s_2
  \models \tforj(Q) \vfo_2.
  \]

\item[\textbf{($\Phi = \neg Q$)}.] Similar to the previous case.

\end{compactitem}

\smallskip
\noindent
\textit{Inductive step:}
\begin{compactitem}
\item[\textbf{($\Phi = \Psi_1 \wedge \Psi_2$)}.]  
  $\ts{1},s_1 \models (\Psi_1\wedge \Psi_2) \vfo_1$ if and only if
  either $\ts{1},s_1 \models \Psi_1 \vfo_1$ or
  $\ts{1},s_1 \models \Psi_2 \vfo_1$.  By induction hypothesis, we
  have for every valuations $\vfo_1$ and $\vfo_2$ that assign to each
  of its free variables a constant $c_1 \in \adom{\abox_1(s_1)}$ and
  $c_2 \in \adom{\abox_2(s_2)}$, such that $c_1 = c_2$, we have
\begin{compactitem}
\item
  $ \ts{1},s_1 \models \Psi_1 \vfo_1 \textrm{ if and only if }
  \ts{2},s_2 \models \tforj(\Psi_1) \vfo_2$, and also
\item
  $ \ts{1},s_1 \models \Psi_2 \vfo_1 \textrm{ if and only if }
  \ts{2},s_2 \models \tforj(\Psi_2) \vfo_2.  $
\end{compactitem}

Hence, $\ts{1},s_1 \models \Psi_1 \vfo_1$ and
$\ts{1},s_1 \models \Psi_2 \vfo_1$ if and only if
$\ts{2},s_2 \models \tforj(\Psi_1) \vfo_2$ and
$\ts{2},s_2 \models \tforj(\Psi_2) \vfo_2$. Therefore we have
$ \ts{1},s_1 \models (\Psi_1 \wedge \Psi_2) \vfo_1 \textrm{ if and
  only if } \ts{2},s_2 \models (\tforj(\Psi_1) \wedge \tforj(\Psi_2))
\vfo_2 $
Since
$\tforj(\Psi_1 \wedge \Psi_2) = \tforj(\Psi_1) \wedge \tforj(\Psi_2)$,
we have
\[
\ts{1},s_1 \models (\Psi_1 \wedge \Psi_2) \vfo_1 \textrm{ iff } \ts{2},s_2 \models \tforj(\Psi_1\wedge \Psi_2) \vfo_2
  \]
  The proof for the case of $\Phi = \Psi_1 \vee \Psi_2$,
  $\Phi = \Psi_1 \ra \Psi_2$, and $\Phi = \Psi_1 \lra \Psi_2$ can be
  done similarly.


\item[\textbf{($\Phi = \DIAM{\Psi}$)}.]  Assume
  $\ts{1},s_1 \models (\DIAM{\Psi}) \vfo_1$, where $\vfo_1$ is a
  valuation that assigns to each free variable of $\Psi$ a constant
  $c_1 \in \adom{\abox_1(s_1)}$. Then there exists $s_1'$ s.t.\
  $s_1 \trans_1 s_1'$ and $\ts{1},s_1' \models \Psi \vfo_1$.  Since
  $s_1 \jbsim s_2$, there exists $s_2'$, $t_1, \ldots ,t_n$ (for
  $n \geq 0$) with
    \[
    s_2 \Rightarrow_2 t_1 \Rightarrow_2 \ldots \Rightarrow_2 t_n
    \Rightarrow_2 s_2'
    \] 
    such that $s_1' \jbsim s_2'$, $\tmp \in \abox_2(t_i)$ for
    $i \in \set{1, \ldots, n}$, and $\tmp \not\in \abox_2(s_2')$.
    Hence, by induction hypothesis, for every valuations $\vfo_2$ that
    assign to each free variables $x$ of $\tforj(\Psi)$ a constant
    $c_2 \in \adom{\abox_2(s_2)}$, such that $c_1 = c_2$ and
    $x/c_1 \in \vfo_1$, we have
    $ \ts{2},s_2' \models \tforj(\Psi_1) \vfo_2.  $
%
%
%
%
%
    Consider that
    \[
    s_2\Rightarrow_2 t_1 \Rightarrow_2 \ldots
    \Rightarrow_2 t_n \Rightarrow_2 s_2'
    \] 
    (for $n \geq 0$), $\tmp \in \abox_2(t_i)$ for
    $i \in \set{1, \ldots, n}$, and $\tmp \not\in \abox_2(s_2')$. we
    therefore get
    \[
    \begin{array}{@{}l@{}l@{}}
      \ts{2},s_2 \models (\DIAM{\mu Z.((&\tmp \wedge \DIAM{Z})
      \vee \\
      &(\neg\tmp \wedge \tforj(\Psi)))})\vfo_2.
    \end{array}
    \]
    Since
    \[
    \begin{array}{l@{}l}
      \tforj(\DIAM{\Phi}) = \DIAM{\mu Z.((&\tmp \wedge \DIAM{Z}) \vee \\
      &(\neg \tmp \wedge \tforj(\Phi)))},
    \end{array}
    \]
    thus we have
    \[
    \ts{2},s_2 \models \tforj(\DIAM{\Phi})\vfo_2.
    \]

The other direction can be shown in a symmetric way.

\item[\textbf{($\Phi = \BOX{\Psi}$)}.]  The proof is similar to the
  case of $\Phi = \DIAM{\Psi}$

\item[\textbf{($\Phi = \exists x. \Psi$)}.]  Assume that
  $\ts{1},s_1 \models (\exists x. \Psi)\vfo'_1$, where $\vfo'_1$ is a
  valuation that assigns to each free variable of $\Psi$ a constant
  $c_1 \in \adom{\abox_1(s_1)}$. Then, by definition, there exists
  $c \in \adom{\abox_1(s_1)}$ such that
  $\ts{1},s_1 \models \Psi\vfo_1$, where $\vfo_1 = \vfo'_1[x/c]$. By
  induction hypothesis, for every valuation $\vfo_2$ that assigns to
  each free variable $y$ of $\tforj(\Psi)$ a constant
  $c_2 \in \adom{\abox_2(s_2)}$, such that $c_1 = c_2$ and
  $y/c_1 \in \vfo_1$, we have that
  $\ts{2},s_2 \models \tforj(\Psi) \vfo_2$. Additionally, we have
  $\vfo_2 = \vfo'_2[x/c']$, where $c' \in \adom{\abox_2(s_2)}$, and
  $c' = c$ because $\abox_2(s_2) = \abox_1(s_1)$.  Hence, we get
  $\ts{2},s_2 \models (\exists x. \tforj(\Psi))\vfo'_2$. Since
  $\tforj(\exists x.\Phi) = \exists x. \tforj(\Phi)$, thus we have
  $\ts{2},s_2 \models \tforj(\exists x. \Psi)\vfo'_2$

The other direction can be shown similarly.

\item[\textbf{($\Phi = \forall x. \Psi$)}.]  The proof is similar to
  the case of $\Phi = \exists x. \Psi$.


\end{compactitem}

\smallskip
\noindent
\textbf{Extension to arbitrary countable disjunction.}  Let $\Psi$ be
a countable set of $\ladom$ formulae. Given a transition system
$\ts{} = \tup{\const, T,\stateset,s_{0},\abox,\trans}$, the semantics
of $\bigvee \Psi$ is
$(\bigvee \Psi) _\vfo^{\ts{}} = \bigcup_{\psi \in \Psi}
(\psi)_\vfo^{\ts{}}$.
Therefore, given a state $s \in \Sigma$ we have
$\ts{}, s \models (\bigvee \Psi)\vfo$ if and only if there exists
$\psi \in \Psi$ such that $\ts{}, s \models \psi\vfo$. Arbitrary
countable conjunction can be obtained similarly.

Now, let $\ts{1} = \tup{\const, T,\stateset_1,s_{01},\abox_1,\trans_1}$
and $\ts{2} = \tup{\const, T,\stateset_2,s_{02},\abox_2,\trans_2}$.
Consider two states $s_1 \in \stateset_1$ and $s_2 \in
\stateset_2$ such that $s_1 \jbsim s_2$.
By induction hypothesis, we have for every valuations $\vfo_1$ and
$\vfo_2$ that assign to each of its free variables a constant
$c_1 \in \adom{\abox_1(s_1)}$ and $c_2 \in \adom{\abox_2(s_2)}$, such
that $c_2 = c_1$, we have that for every formula $\psi \in \Psi$, it
holds $\ts{1}, s_1 \models \psi \vfo_1$ if and only if
$\ts{2}, s_2 \models \tforj(\psi)\vfo_2$.
Given the semantics of $\bigvee \Psi$ above, this implies that
$\ts{1}, s \models (\bigvee \Psi) \vfo_1$ if and only if
$\ts{2}, s \models (\bigvee \tforj(\Psi)) \vfo_2$, where
$\tforj(\Psi) = \{\tforj(\psi) \mid \psi \in \Psi\}$. The proof is
then obtained by observing that
$\bigvee \tforj(\Psi) = \tforj(\bigvee \Psi)$.

\smallskip
\noindent
\textbf{Extension to full $\muladom$.}  In order to extend the result
to the whole \muladom, we resort to the well-known result stating that
fixpoints of the $\mu$-calculus can be translated into the infinitary
Hennessy Milner logic by iterating over \emph{approximants}, where the
approximant of index $\alpha$ is denoted by $\mu^\alpha Z.\Phi$
(resp.~$\nu^\alpha Z.\Phi$). This is a standard result that also holds
for \muladom. In particular, approximants are built as follows:
\[
\begin{array}{rl rl}
  \mu^0 Z.\Phi & = \false
  &  \nu^0 Z.\Phi & = \true\\
  \mu^{\beta+1} Z.\Phi & = \Phi[Z/\mu^\beta Z.\Phi]
  & \nu^{\beta+1} Z.\Phi & = \Phi[Z/\nu^\beta Z.\Phi]\\
  \mu^\lambda Z.\Phi & = \bigvee_{\beta < \lambda} \mu^\beta Z. \Phi &
  \nu^\lambda Z.\Phi & = \bigwedge_{\beta < \lambda} \nu^\beta Z. \Phi
\end{array}
\]
where $\lambda$ is a limit ordinal, and where fixpoints and their
approximants are connected by the following properties: given a
transition system $\ts{}$ and a state $s$ of $\ts{}$
\begin{compactitem}
\item $s \in \MODA{\mu Z.\Phi}$ if and only if there exists an ordinal
  $\alpha$ such that $s \in \MODA{\mu^\alpha Z.\Phi}$ and, for every
  $\beta < \alpha$, it holds that $s \notin \MODA{\mu^\beta Z.\Phi}$;
\item $s \notin \MODA{\nu Z.\Phi}$ if and only if there exists an
  ordinal $\alpha$ such that $s \notin \MODA{\nu^\alpha Z.\Phi}$ and,
  for every $\beta < \alpha$, it holds that $s \in \MODA{\nu^\beta
    Z.\Phi}$.
\end{compactitem}
\end{proof}

As a consequence, from Lemma
\ref{lem:jumping-bisimilar-states-satisfies-same-formula} above, we
can easily obtain the following lemma saying that two transition
systems which are J-bisimilar can not be distinguished by any \muladom
formula (in NNF) modulo the translation $\tforj$.

\begin{lemma}\label{lem:jumping-bisimilar-ts-satisfies-same-formula}
  Consider two transition systems
  $\ts{1} = \tup{\const,T,\stateset_1,s_{01},\abox_1,\trans_1}$ and
  $\ts{2} = \tup{\const,T,\stateset_2,s_{02},\abox_2,\trans_2}$ such
  that $\ts{1} \jbsim \ts{2}$.  For every $\muladom$ closed formula
  $\Phi$ (in NNF), we have:
  \[
  \ts{1} \models \Phi \textrm{ if and only if } \ts{2} \models
  \tforj(\Phi).
  \]
\end{lemma}
\begin{proof}
  Since by the definition we have $s_{01} \jbsim s_{02}$, we obtain
  the proof as a consequence of
  Lemma~\ref{lem:jumping-bisimilar-states-satisfies-same-formula} due
  to the fact that
  \[
  \ts{1}, s_{01} \models \Phi \textrm{ if and only if } \ts{2}, s_{02}
  \models \tforj(\Phi)
  \]
\ \ 
\end{proof}

\subsection{Transforming S-GKABs into S-KABs} \label{sec:translation-sgkab-to-skab}


As the first step towards defining a generic translation to compile
S-GKABs into S-KABs, we introduce the notion of program IDs as follows.

\begin{definition}[Golog Program with IDs]
  Given a set of actions $\actset$, a \emph{Golog program with
    ID} $\delta$ over $\actset$ is an expression formed by the
  following grammar:
  \[
  \begin{array}{@{}r@{\ }l@{\ }}
    \tup{id, \delta}  ::= &
                            \tup{id, \gemptyprog} ~\mid~
                            \tup{id, \gact{Q(\vec{p})}{\act(\vec{p})}} ~\mid~ \\
                          &\tup{id, \delta_1|\delta_2}  ~\mid~
                            \tup{id, \delta_1;\delta_2} ~\mid~ \\
                          &\tup{id, \gif{\varphi}{\delta_1}{\delta_2}} ~\mid~
                            \tup{id, \gwhile{\varphi}{\delta}}
  \end{array}
  \]
  where \emph{$id$ is a program ID} which is
  simply a string over some alphabets, and the rest of the things are
  the same as in usual Golog program defined before.  \ \
\end{definition}

\noindent
All notions related to golog program can be defined similarly for the
golog program with ID. We now step further to define a
formal translation that transforms a golog program into a golog
program with ID.
As for notation given program IDs $id$ and $id'$, we write $id.id'$ to
denote a string obtained by concatenating the strings $id$ and $id'$
consecutively.

\begin{definition}[Program ID Assignment]
  We define a translation $\tpid(\delta, id)$ that 
\begin{compactenum}
\item takes a program $\delta$ as well as a program ID $id$, and
\item produces a golog program with ID $\tup{id, \delta_{id}}$ such that
  each sub-program of $\delta$ is associated with a unique program ID
  and occurrence matters (i.e., for each sub-program $\delta'$ of
  $\delta$ such that $\delta'$ occurs more than once in $\delta$, each
  of them has a different program ID).
\end{compactenum}
The translation $\tpid(\delta, id)$ is formally defined as follows:

  \begin{compactitem}
  \item $\tpid(\gemptyprog, id) = \tup{id, \gemptyprog}$, 
%
  \item
    $\tpid(\gact{Q(\vec{p})}{\act(\vec{p})}, id) = \tup{id,
      \gact{Q(\vec{p})}{\act(\vec{p})}}$,
%
  \item $\tpid(\delta_1|\delta_2, id) = \tup{id, \tpid(id.id', \delta_1) |
      \tpid(id.id'', \delta_2)}$, \\ where
    $id'$ and $id''$ are fresh program IDs.
  \item $\tpid(\delta_1;\delta_2, id) = \tup{id, \tpid(id.id', \delta_1) ;
      \tpid(id.id'', \delta_2)}$, \\ where
    $id'$ and $id''$ are fresh program IDs.
  \item
    $\tpid(\gif{\varphi}{\delta_1}{\delta_2}, id) = \\ \tup{id,
      \gif{\varphi}{\tpid(id.id', \delta_1)}{\tpid(id.id'', \delta_2)}}$,
    \\ where $id'$ and $id''$ are fresh program IDs.
  \item
    $\tpid(\gwhile{\varphi}{\delta_1}, id) = \\ \tup{id,
      \gwhile{\varphi}{\tpid(id.id', \delta_1)}}$,
    \\ where $id'$ is a fresh program IDs.
  \end{compactitem}
  Given a program $\delta$, we say \emph{$\tup{id, \delta_{id}}$ is a
    program with ID w.r.t.\ $\delta$} if
  $\tpid(\delta, id) = \tup{id, \delta_{id}}$ where $id$ is a fresh
  program ID and $\delta_{id}$ is a program with ID.
\end{definition}

\begin{definition}[Program ID Retrieval function]
  Let $\delta$ be a program and $\tup{id, \delta_{id}}$ be its
  corresponding program with ID w.r.t.\ $\delta$, we define a function
  $\pid$ that
  \begin{compactenum}
  \item maps each sub-program of $\tup{id, \delta}$ into its unique
    id. I.e., for each sub-program $\tup{id', \delta'}$ of
    $\tup{id, \delta}$, we have $\pid(\tup{id', \delta'}) = id'$, and
  \item additionally, for a technical reason related to the
    correctness proof of our translation from S-GKABs to S-KABs, for
    each action invocation
    $\tup{id_\act, \gact{Q(\vec{p})}{\act(\vec{p})}}$, that is a
    sub-program of $\tup{id, \delta}$, we have
    $\pid(\tup{id_\act.\gemptyprog, \gemptyprog}) =
    id_\act.\gemptyprog$
    (where $id_\act.\gemptyprog$ is a new ID simply obtained by
    concatenating $id_\act$ with a string $\gemptyprog$).
\end{compactenum}
\ \ 
\end{definition}

For simplicity of the presentation, from now on we assume that every
program is associated with ID. 
Note that every program without ID can be transform into a program
with ID as above.
Moreover we will not write the ID that is attached to a (sub-)program,
and when it is clear from the context, we simply write
\emph{$\pid(\delta')$}, instead of $\pid(\tup{id, \delta'})$, to
denote the \emph{unique program ID of a sub-program $\delta'$ of
  $\delta$} that is based on its occurrence in $\delta$.

We now proceed to define a translation $\tgprog(\pre, \delta, \post)$,
that given a golog program $\delta$, as well as two flags $\pre$ and
$\post$, produces a process (set of condition-action rules) and a set
of actions that mimics the execution of the program $\delta$ starting
from a state $s$ with an ABox $A$ (i.e., $A = \abox(s)$) where
$\pre \in A$ and at the end of the execution of $\delta$, that changes
$A$ into $A'$, we have $\post \in A'$, but $\pre \notin A'$.
Intuitively, $\pre$ and $\post$ act as markers which indicate the
start and the end of the execution of the corresponding program
$\delta$. 
Formally, the translation $\tgprog$ 
is defined as follows:

\begin{definition}[Program Translation]\label{def:prog-translation}
  We define a translation $\tgprog$ that takes as inputs:
\begin{compactenum}
\item A program $\delta$ over a set of actions $\actset$,
\item Two flags (which will be used as markers indicating the start
  and the end of the execution of a program $\delta$).
\end{compactenum}
and produces as outputs:
\begin{compactenum}
\item $\ppre$
  is a function that maps a sub-program $\delta'$ of $\delta$
  to a flag (called \emph{start flag} of $\delta'$) that act as a
  marker indicating the start of the execution of $\delta'$,
\item $\ppost$ is a function that maps a sub-program $\delta'$ of
  $\delta$
  to a flag (called \emph{end flag} of $\delta'$) which act as a
  marker indicating the end of the execution of $\delta'$,
\item $\procset$ is a process (a set of condition-action rules),
\item $\actset'$ is a set of actions.
\end{compactenum}
I.e., $\tgprog(\pre,
\delta, \post) = \tup{\ppre, \ppost, \procset,
  \actset'}$, where $\pre$ and $\post$ are flags.
Formally, $\tgprog(\pre, \delta, \post)$ is inductively defined over
the structure of a program $\delta$ as follows:

 \begin{enumerate}

 \item For the case of $\delta = \gemptyprog$ (i.e., $\delta$ is an
   empty program):
   \[
   \tgprog(\pre, \gemptyprog, \post) = 
   \tup{ \ppre, \ppost, \set{\carule{\pre} {\act_\gemptyprog()}},
     \set{\act_\gemptyprog} },
   \] 
   where
   \begin{compactitem}[$\bullet$]
   \item $\ppre = \set{\tap{\pid(\gemptyprog) \ra \pre}}$,
   \item $\ppost = \set{\tap{\pid(\gemptyprog) \ra \post}}$,
   \item $\act_\gemptyprog$ is of the form\\ 
     $ \act_\gemptyprog():\set{\map{\true}{ \add \set{\post, \tmp},
         \del \set{\pre} }}; $
   \end{compactitem}


 \item For the case of $\delta = \gact{Q(\vec{p})}{\act(\vec{p})}$
   (i.e., $\delta$ is an action invocation) with
   $\pid(\gact{Q(\vec{p})}{\act(\vec{p})}) = id_\act$:
  \[
  \tgprog(\pre, \gact{Q(\vec{p})}{\act(\vec{p})}, \post) = \tup{\ppre,
    \ppost, \procset, \actset'},
  \]
  where
  \begin{compactitem}[$\bullet$]
  \item
    $\ppre = \set{\tap{\pid(\gact{Q(\vec{p})}{\act(\vec{p})}) \ra \pre}}
    \cup \ppre'$,
  \item
    $\ppost = \set{\tap{\pid(\gact{Q(\vec{p})}{\act(\vec{p})}) \ra \post}}
    \cup \ppost'$,
  \item
    $\procset = \set{\carule{Q(\vec{p}) \wedge \pre}{\act'(\vec{p})}}
    \cup \procset'$,
  \item $\actset' = \set{\act'} \cup \actset''$,
    where
  \begin{center}
    $\begin{array}{ll} \eff{\act'} = &\eff{\act} \cup \\
                                     &\set{\map{\true}{ \add
                                       \set{\post} }}
                                       \cup \\
                                     &\set{\map{\true}{ \del
                                       \set{\pre, \tmp} }}
                                       \cup \\
                                     &\set{\map{\noopconcept{x}}{\del
                                       \noopconcept{x} }},
  \end{array}
  $\end{center}
\item
  $\tgprog(\post, \gemptyprog, \post) = \tup{\ppre', \ppost',
    \procset', \actset''}$, \\where $\pid(\gemptyprog) = id_\act.\gemptyprog$
%
\end{compactitem}


\item For the case of $\delta = \delta_1 | \delta_2$ (i.e., $\delta$
  is a non-deterministic choice between programs):
  \[
  \tgprog(\pre, \delta_1|\delta_2, \post) = \tup{\ppre, \ppost, \procset, \actset},
  \]
  where
  \begin{compactitem}
  \item $ \procset = \set{
      \carule{\pre}{\gamma_{\delta_1}()},
      \carule{\pre}{\gamma_{\delta_2}()}}
    \cup \procset_1 \cup \procset_2$,
  \item $\actset = \actset_1~\cup~\actset_2~\cup~\set{\gamma_{\delta_1}, \gamma_{\delta_2}}$, where
    \begin{compactitem}
    \item $\gamma_{\delta_1}():\set{\map{\true} \\
        \hspace*{10mm}{\add \set{\flagconcept{c_1}, \tmp}, \del \set{\pre} }}$,
    \item $\gamma_{\delta_2}() : \set{\map{\true} \\
        \hspace*{10mm}{\add \set{\flagconcept{c_2}, \tmp}, \del \set{\pre} }}$,
    \end{compactitem}
  \item $\tgprog(\flagconcept{c_1}, \delta_1, \post) = \tup{\ppre_1, \ppost_1, \procset_1, \actset_1}$,
  \item $\tgprog(\flagconcept{c_2}, \delta_2, \post) = \tup{\ppre_2, \ppost_2, \procset_2, \actset_2}$,
  \item $c_1, c_2 \in \const_0$ are fresh constants;
  \end{compactitem}

%
\item
  $\tgprog(\pre, \delta_1;\delta_2, \post) = \tup{\ppre, \ppost,
    \procset_1 \cup \procset_2, \actset_1 \cup \actset_2}$, where
\begin{compactitem}
\item
  $\ppre = \set{\tap{\pid(\delta_1;\delta_2) \ra \pre}} \cup \ppre_1 \cup
  \ppre_2$,
\item
  $\ppost = \set{\tap{\pid(\delta_1;\delta_2) \ra \post}} \cup \ppost_1 \cup
  \ppost_2$,
\item $\tgprog(\pre, \delta_1, \flagconcept{c}) =
  \tup{\ppre_1, \ppost_1, \procset_1, \actset_1}$,
\item $\tgprog(\flagconcept{c}, \delta_2, \post) =
  \tup{\ppre_2, \ppost_2, \procset_2, \actset_2}$, 
\item $c \in \const_0$ is a fresh constant;
\end{compactitem}


\item
  $\tgprog(\pre, \gif{\varphi}{\delta_1}{\delta_2}, \post) =
  \tup{\ppre, \ppost, \procset,\actset}$, where
\begin{compactitem}
\item
  $\ppre = \set{\tap{\pid(\gif{\varphi}{\delta_1}{\delta_2}) \ra
      \pre}} \cup\\ 
  \hspace*{12mm}\ppre_1 \cup \ppre_2$,
\item $\ppost = \set{\tap{\pid(\gif{\varphi}{\delta_1}{\delta_2}) \ra
      \post}} \cup \\
  \hspace*{12mm}\ppost_1 \cup \ppost_2$,

\item
  $ \procset = \set{
    \carule{\pre \wedge \varphi}{\gamma_{if}()},
    \carule{\pre \wedge \neg \varphi}{\gamma_{else}()}}
  \cup \\ \hspace*{12mm}\procset_1 \cup \procset_2$,
\item $\actset = \actset_1~\cup~\actset_2~\cup~\set{\gamma_{if},
    \gamma_{else}}$, where
\begin{compactitem}
\item $\gamma_{if}():\set{\map{\true}\\
    \hspace*{10mm}{\add \set{\flagconcept{c_1}, \tmp}, \del
      \set{\pre} }}$,
\item
  $\gamma_{else}() : \set{\map{\true}\\
    \hspace*{10mm}{\add \set{\flagconcept{c_2}, \tmp}, \del \set{\pre}
    }}$,
\end{compactitem}
\item
  $\tgprog(\flagconcept{c_1}, \delta_1, \post) = \tup{\ppre_1,
    \ppost_1, \procset_1, \actset_1}$,

\item
  $\tgprog(\flagconcept{c_2}, \delta_2, \post) = \tup{\ppre_2,
    \ppost_2, \procset_2, \actset_2}$,

\item $c_1, c_2 \in \const_0$ are fresh constants;
\end{compactitem}


\item $\tgprog(\pre, \gwhile{\varphi}{\delta}, \post) =
  \tup{\ppre, \ppost, \procset, \actset}$, where 
\begin{compactitem}
\item
  $\ppre = \set{ \tap{\pid(\gwhile{\varphi}{\delta}) \ra \pre}} \cup
  \ppre'$

\item
  $\ppost = \set{ \tap{\pid(\gwhile{\varphi}{\delta}) \ra \post}} \cup
  \ppost'$

\item $ \procset~=~\procset'\cup~ \procset_{loop}$, where
  $\procset_{loop}$ contains:
\begin{compactitem}
\item
  $ \carule{\pre \wedge \varphi \wedge \neg \noopconcept{noop}
  }{\gamma_{doLoop}()}$,
\item
  $ \carule{\pre \wedge (\neg \varphi \vee \noopconcept{noop}) }
  {\gamma_{endLoop}()}$,
\end{compactitem}

\item $\actset~=~\actset'~\cup \actset_{loop}$, where $\actset_{loop}$
  contains the following:
\begin{compactitem}
\item $\gamma_{doLoop}(): \set{\true  \rightsquigarrow\\
    \hspace*{2mm} \add \set{\flagconcept{lStart}, \noopconcept{noop},
      \tmp }, \\
    \hspace*{2mm}\del \set{\pre}}$,
\item $\gamma_{endLoop}(): \set{\true  \rightsquigarrow \\
    \hspace*{2mm} \add \set{\post, \tmp }, \\
    \hspace*{2mm}\del \set{\pre, \noopconcept{noop}}}$,
\end{compactitem}

\item
  $\tgprog(\flagconcept{lStart}, \delta, \pre) =
  \tup{\ppre', \ppost', \procset', \actset'}$,


\item $noop, lStart \in \const_0$ are fresh constants.
\end{compactitem}
\end{enumerate}
\ \ 
\end{definition}

For compactness reason, we often simply write $\ppre(\delta)$ to
abbreviate the notation $\ppre(\pid(\delta))$ that essentially returns
the start flag of a program with program ID $\pid(\delta)$. Similarly
for $\ppost(\delta)$.

\begin{lemma}\label{lem:program-pre-post}
  Given a program $\delta$ over a set $\actset$ of actions.  We have
  $ \tgprog(\pre, \delta, \post) = \tup{\ppre, \ppost, \procset,
    \actset}$
  if and only if
  $\ppre(\delta) = \pre \mbox{ and } \ppost(\delta) = \post $
\end{lemma}
\begin{proof}
  Directly follows from the definition of $\tgprog$.
\end{proof}

Having $\tgprog$ in hand, we define a translation $\tgkab$ that
compile S-GKABs into S-KABs as follows.

\begin{definition}[Translation from S-GKABs to S-KABs]
  We define a translation $\tgkab$ that takes an S-GKAB
  $\gkabsym = \tup{T, \initABox, \actset, \ginitprog}$ as the input
  and produces an S-KAB
  $\tgkab(\gkabsym) = \tup{T, \initABox', \actset', \procset'}$ s.t.\
  \begin{compactitem}
  \item $\initABox' = \initABox \cup \set{\flagconcept{start}}$, and
  \item
    $\tgprog(\flagconcept{start}, \ginitprog, \flagconcept{end}) =
    \tup{\ppre, \ppost, \procset', \actset'}$.
  \end{compactitem}
\ \ 
\end{definition}

To show some properties of the translation $\tgkab$ which
transform S-GKAB to S-KAB above, we first introduce several
preliminaries below. As the first step, we define the notion when a
state of an S-GKAB is mimicked by a state of an S-KAB as follows.

\begin{definition}\label{def:mimic-state}
  Let $\gkabsym = \tup{T, \initABox, \actset, \delta}$ be a normalized
  S-GKAB with transition system $\ts{\gkabsym}^{\filter_S}$, and
  $\tgkab(\gkabsym) = \tup{T, \initABox', \actset', \procset'}$ be an
  S-KAB with transition system $\ts{\tgkab(\gkabsym)}^S$ obtained from
  $\gkabsym$ through $\tgkab$ s.t.\
\begin{inparaenum}[(i)]
\item $\initABox' = \initABox \cup \set{\flagconcept{start}}$, and
\item $\tgprog(\flagconcept{start}, \ginitprog, \flagconcept{end}) =
  \tup{\ppre, \ppost, \procset', \actset'}$.
\end{inparaenum}
Consider two states $\tup{A_g,\scmap_g, \delta_g}$ of
$\ts{\gkabsym}^{\filter_S}$ and $\tup{A_k,\scmap_k}$ of
$\ts{\tgkab(\gkabsym)}^S$. 
We say \emph{$\tup{A_g,\scmap_g, \delta_g}$ is mimicked by
  $\tup{A_k,\scmap_k}$} (or equivalently \emph{$\tup{A_k,\scmap_k}$
  mimics $\tup{A_g,\scmap_g, \delta_g}$}), written
$\tup{A_g,\scmap_g, \delta_g} \mimic \tup{A_k,\scmap_k}$, if
\begin{compactenum}
\item $A_k \eqm A_g$,
\item $\scmap_k = \scmap_g$, and 
\item $\ppre(\delta_g) \in A_k$.
\end{compactenum}
\ \ 
\end{definition}

Next, we define the notion of temp adder/deleter action as follows.

\begin{definition}[Temp Marker Adder Action]\label{def:tmp-adder-action}
  Let $\gkabsym$ be an S-GKAB and
  $\tgkab(\gkabsym) = \tup{T, \initABox', \actset', \procset'}$ be the
  corresponding S-KAB obtained from $\gkabsym$ via $\tgkab$.  
  An action $\act \in \actset$ is \emph{a temp adder action of
    $\tgkab(\gkabsym)$} if there exists an effect $e \in \eff{\act}$
  of the form $ \map{[q^+]\land Q^-}{\add \facta, \del \factd} $ such
  that $\tmp \in \facta$.
  We write $\actsettmpa$ to denote the set of temp adder actions
  of $\tgkab(\gkabsym)$.
\end{definition}

\begin{definition}[Temp Marker Deleter Action]\label{def:tmp-deleter-action}
  Let $\gkabsym$ be an S-GKAB and
  $\tgkab(\gkabsym) = \tup{T, \initABox', \actset', \procset'}$ be the
  corresponding S-KAB obtained from $\gkabsym$ via $\tgkab$.  
  An action $\act \in \actset$ is \emph{a temp deleter action of
    $\tgkab(\gkabsym)$} if there exists an effect $e \in \eff{\act}$
  of the form $ \map{[q^+]\land Q^-}{\add \facta, \del \factd} $ such
  that $\tmp \in \factd$.
  We write $\actsettmpd$ to denote the set of temp deleter
  actions of $\tgkab(\gkabsym)$.
\end{definition}

\noindent
Roughly speaking, a temp adder action is an action that adds the
ABox assertion $\tmp$. Similarly, a temp deleter action is an
action that removes the ABox assertion $\tmp$.

\begin{lemma}\label{lem:action-set-separation}
  Let $\gkabsym$ be an S-GKAB,
  $\tgkab(\gkabsym) = \tup{T, \initABox', \actset', \procset'}$ be the
  corresponding S-KAB obtained from $\gkabsym$ via $\tgkab$, and
  $\actsettmpa$ (resp.\ $\actsettmpd$) be a set of temp adder
  (resp.\ deleter) actions of $\tgkab(\gkabsym)$. We have that
  $\actset' = \actsettmpa \uplus \actsettmpd$.
\end{lemma}
\begin{proof}
  Trivially true by observing Definitions~\ref{def:prog-translation},
  \ref{def:tmp-adder-action}, \ref{def:tmp-deleter-action}.
\end{proof}

\begin{lemma}\label{lem:temp-state-produced-by-temp-act}
  Let $\gkabsym$ be an S-GKAB,
  $\tgkab(\gkabsym) = \tup{T, \initABox', \actset', \procset'}$ be the
  corresponding S-KAB (with transition system $\ts{\tgkab(\gkabsym)}^S$)
  obtained from $\gkabsym$ via $\tgkab$, and $\actsettmpa$ be a set of
  temp adder actions of $\tgkab(\gkabsym)$.
  Consider a state $\tup{A_k, \scmap_k}$ of $\ts{\tgkab(\gkabsym)}^S$,
  if there exists a state $\tup{A_k', \scmap_k'}$ such that
  $\tup{A_k, \scmap_k} \exect{\act\sigma} \tup{A_k', \scmap_k'}$, and
  $\tmp \in A_k'$ then 
  $\sigma$ is an empty substitution,
  $\act \in \actsettmpa$,
  $\act$ does not involve any service calls,
  $A_k' \eqm A_k$ and $\scmap_k' = \scmap_k$.
\end{lemma}
\begin{proof}
  Since $\tmp \in A_k'$, then by Definition~\ref{def:tmp-adder-action} and
  Lemma~\ref{lem:action-set-separation} we must have
  $\act \in \actsettmpa$. By the definition of translation $\tgprog$
  (see Definition~\ref{def:prog-translation}), any actions in $\actsettmpa$ does
  not involve service calls and only do a manipulation on special
  markers. Thus, it is easy to see that $A_k' \eqm A_k$ and
  $\scmap_k' = \scmap_k$.
\end{proof}

\begin{lemma}\label{lem:non-temp-state-produced-by-normal-action}
  Let $\gkabsym = \tup{T, \initABox, \actset, \ginitprog}$ be an S-GKAB,
  $\tgkab(\gkabsym) = \tup{T, \initABox', \actset', \procset'}$ be the
  corresponding S-KAB (with transition system $\ts{\tgkab(\gkabsym)}^S$)
  obtained from $\gkabsym$ via $\tgkab$, and $\actsettmpa$ be a set of
  temp adder actions of $\tgkab(\gkabsym)$.
  Consider a state $\tup{A_k, \scmap_k}$ of $\ts{\tgkab(\gkabsym)}^S$,
  if there exists a state $\tup{A_k', \scmap_k'}$ such that
  $\tup{A_k, \scmap_k} \exect{\act'\sigma} \tup{A_k', \scmap_k'}$, and
  $\tmp \not\in A_k'$ then $\act' \in \actsettmpd$, and there
  exists action invocation $\gact{Q(\vec{p})}{\act(\vec{p})}$ in the
  sub-proram of $\ginitprog$ such that $\act'$ is obtained from the
  translation of $\gact{Q(\vec{p})}{\act(\vec{p})}$ via $\tgprog$.
\end{lemma}
\begin{proof}
  Since $\tmp \not\in A_k'$, then by Definition~\ref{def:tmp-deleter-action}
  and Lemma~\ref{lem:action-set-separation} we must have
  $\act' \in \actsettmpd$. By the definition of translation $\tgprog$
  (see Definition~\ref{def:prog-translation}), $\act'$ must be obtained from the
  translation of an action invocation
  $\gact{Q(\vec{p})}{\act(\vec{p})}$ in the sub-proram of
  $\ginitprog$.
\end{proof}

The following lemma shows that given two action invocations that has
different program ID, we have that their start flags are
different. I.e., any actions invocations that occur in a different
place inside a certain program will have different start
flag. This claim is formalized below.

\begin{lemma}\label{lem:action-invocation-unique-start-flag}
  Let $\gkabsym = \tup{T, \initABox, \actset, \ginitprog}$ be an
  S-GKAB,
  $\tgkab(\gkabsym) = \tup{T, \initABox', \actset', \procset'}$ be the
  corresponding S-KAB (with transition system $\ts{\tgkab(\gkabsym)}^S$)
  obtained from $\gkabsym$ via $\tgkab$, and $\actsettmpa$ be a set of
  temp adder actions of $\tgkab(\gkabsym)$.
  Consider two action invocations
  $\gact{Q_1(\vec{x})}{\act_1(\vec{x})}$ and
  $\gact{Q_2(\vec{y})}{\act_2(\vec{y})}$ that are sub-programs of
  $\ginitprog$. We have that
  $\pid(\gact{Q_1(\vec{x})}{\act_1(\vec{x})}) \neq
  \pid(\gact{Q_2(\vec{y})}{\act_2(\vec{y})})$
  if and only if
  $\ppre(\pid(\gact{Q_1(\vec{x})}{\act_1(\vec{x})})) \neq
  \ppre(\pid(\gact{Q_2(\vec{y})}{\act_2(\vec{y})}))$.
\end{lemma}
\begin{proof}
  Trivially true by observing the definition of translation $\tgprog$
  (see Definition~\ref{def:prog-translation}).
\end{proof}

We now progress to show a property of translation $\tgkab$ that is
related to the final states of S-GKABs transition system. Essentially,
we show that given a final state $s_g = \tup{A_g,\scmap_g, \delta_g}$
of an S-GKAB transition system and a state $s_k$ of its corresponding
S-KAB transition system such that those two states are J-bisimilar
(i.e., $s_g \jbsim s_k$), we have that there exists a state $s_k'$
that is reachable from $s_k$ (possibly through some intermediate
states) and we have that $\ppost(\delta_g)$ is in the ABox that is
contained in $s_k'$. Formally this claim is stated below.

\begin{lemma}\label{lem:final-state-add-transition}
  Given
  \begin{inparaenum}[]
  \item an S-GKAB $\gkabsym$ (with a transition system
    $\ts{\gkabsym}^{\filter_S}$), and
  \item an S-KAB $\tgkab(\gkabsym)$ (with a transition system
    $\ts{\tgkab(\gkabsym)}^S$) that is obtained from $\gkabsym$ through
    $\tgkab$.
\end{inparaenum}
Consider the states
\begin{inparaenum}[]
\item $\tup{A_g,\scmap_g, \delta_g}$ of $\ts{\gkabsym}^{\filter_S}$
  and
\item $\tup{A_k,\scmap_k}$ of $\ts{\tgkab(\gkabsym)}^S$.
\end{inparaenum}
If 
\begin{inparaitem}[]
\item $\final{\tup{A_g, \scmap_g, \delta_g}}$, and
\item $\tup{A_g,\scmap_g, \delta_g} \mimic \tup{A_k,\scmap_k}$,
\end{inparaitem}
then there exists states $\tup{A_i, m_k}$ and actions $\act_i$ (for
$i \in \set{1, \ldots, n}$, and $n \geq 0$)
such that
\begin{compactitem}
\item
  $\tup{A_k, m_k} \exect{\act_1\sigma} \tup{A_1,
    m_k}\exect{\act_2\sigma} \cdots \exect{\act_n\sigma} \tup{A_n,
    m_k}$ \\(with an empty substitution $\sigma$),
\item $\tmp \in A_i$ (for $i \in \set{1, \ldots, n}$),
\item $\ppost(\delta_g) \in A_n$, 
\item $\ppre(\delta_g) \not\in A_n$ (if
  $\ppost(\delta_g) \neq \ppre(\delta_g)$),
\item $A_n \eqm A_g$, and
\item if $\noopconcept{c} \in A_k$ (where $c \in \const_0$), then
  $\noopconcept{c} \in A_i$ (for $i \in \set{1, \ldots, n}$).
\end{compactitem}
\end{lemma}
\begin{proof}
  Let
  \begin{compactitem}
  \item $\gkabsym = \tup{T, \initABox, \actset, \delta}$, and \\
    $\ts{\gkabsym}^{\filter_S} = \tup{\const, T, \stateset_g, s_{0g},
      \abox_g, \trans_g}$),

  \item $\tgkab(\gkabsym) = \tup{T, \initABox', \actset', \procset'}$,
    and \\
    $\ts{\tgkab(\gkabsym)}^S = \tup{\const, T, \stateset_k, s_{0k},
      \abox_k, \trans_k}$
    where
    \begin{compactitem}[$\bullet$]
    \item $\initABox' = \initABox \cup \set{\flagconcept{start}}$, and
    \item
      $\tgprog(\flagconcept{start}, \ginitprog, \flagconcept{end}) = \tup{\ppre, \ppost, \procset', \actset'}$.
    \end{compactitem}
\end{compactitem}
We show the claim by induction over the definition of final states 
as follows:

\smallskip
\noindent
\textbf{Base case:}
\begin{compactitem}
\item[\textbf{[$\delta_g = \gemptyprog$]}.] Since
  $\tup{A_g,\scmap_g, \gemptyprog} \mimic \tup{A_k,\scmap_k}$, then by
  Definition~\ref{def:mimic-state} we have
  $\ppre(\gemptyprog) \in A_k$.  By the definition of $\tgprog$, we
  have a 0-ary action $\act_\gemptyprog()$ where
  \begin{compactitem}
  \item $\carule{\ppre(\gemptyprog)} {\act_\gemptyprog()}$, and
  \item
    $\eff{\act_\gemptyprog} = \set{\map{\true}\\ \hspace*{10mm}{\add
        \set{\ppost(\gemptyprog), \tmp}, \del \set{\ppre(\gemptyprog)}
      }}$
  \end{compactitem}
  Hence, by observing how an action is executed and the result of an
  action execution is constructed, we easily obtain that there exists
  $\tup{A_1, m_k}$ such that
  \begin{compactitem}
  \item
    $\tup{A_k, m_k} \exect{\act_{\gemptyprog}\sigma} \tup{A_1, m_k}$
    (with an empty substitution $\sigma$),
  \item $\tmp \in A_1$,
  \item $\ppost(\gemptyprog) \in A_1$,
  \item $\ppre(\gemptyprog) \not\in A_1$ 
    (if $\ppre(\gemptyprog) \neq \ppost(\gemptyprog)$), and
  \item $A_1 = A_g$.
  \end{compactitem}
  Additionally, it is also true that if $\noopconcept{c} \in A_k$ (for
  a constant $c \in \const_0$), then $\noopconcept{c} \in A_1$,
  because, by the definition of $\tgprog$, the action
  $\act_{\gemptyprog}$ does not delete any concept made by concept
  names $\noopconceptname$ and only actions that are obtained from the
  translation of an action invocation delete such kind of concept
  assertions. Therefore the claim is proven for this case.
\end{compactitem}

\smallskip
\noindent
\textbf{Inductive cases:}
\begin{itemize}

\item[\textbf{[$\delta_g = \delta_1|\delta_2$]}.] Since
  $\final{\tup{A_g, \scmap_g, \delta_1|\delta_2}}$, then by
  the definition of final states we have either
  \begin{compactenum}[\bf (1)]
  \item $\final{\tup{A_g, \scmap_g, \delta_1}}$, or
  \item $\final{\tup{A_g, \scmap_g, \delta_2}}$.
  \end{compactenum}
  For compactness of the proof, here we only show the case
  \textbf{(1)}. The case \textbf{(2)} can be done similarly.
  Since
  $\tup{A_g,\scmap_g, \delta_1|\delta_2} \mimic \tup{A_k,\scmap}$,
  then $\ppre(\delta_1|\delta_2) \in A_k$.
  By the definition of $\tgkab$, 
  we have
  \begin{compactitem}
  \item
    $\carule{\ppre(\delta_1|\delta_2)}{\gamma_{\delta_1}()} \in
    \procset$,
  \item
    $\gamma_{\delta_1}():\set{\map{\true} {\add \set{\ppre(\delta_1), \tmp},
        \\ \hspace*{23.5mm} \del \set{ \ppre(\delta_1|\delta_2) } }}$,
  \item $\ppost(\delta_1|\delta_2) = \ppost(\delta_1)$.
  \end{compactitem}
Then, by induction hypothesis, and also by observing how an action is
executed as well as the result of an action execution is constructed,
it is easy to see that the claim is proven.

\item[\textbf{[$\delta_g = \delta_1;\delta_2$]}.] Since
  $\final{\tup{A_g, \scmap_g, \delta_1;\delta_2}}$, then by the
  definition of final states we have that
  $\final{\tup{A_g, \scmap_g, \delta_1}}$ and
  $\final{\tup{A_g, \scmap_g, \delta_2}}$.
  Since
  $\tup{A_g,\scmap_g, \delta_1;\delta_2} \mimic \tup{A_k,\scmap_k}$,
  then $\ppre(\delta_1;\delta_2) \in A_k$.  By the definition of
  $\tgkab$ we have that  $\ppre(\delta_1; \delta_2) = \ppre(\delta_1)$,
  $\ppost(\delta_1) = \ppre(\delta_2)$, and
  $\ppost(\delta_2) = \ppost(\delta_1; \delta_2)$.
  By induction hypothesis, there exists states $\tup{A_i, m_k}$, and
  actions $\act_i$, (for $i \in \set{1, \ldots, l}$, and $n \geq 0$)
such that
\begin{compactitem}
\item
  $\tup{A_k, m_k} \exect{\act_1\sigma} \tup{A_1,
    m_k}\exect{\act_2\sigma} \cdots
  \exect{\act_l\sigma} \tup{A_l, m_k}$
\\
  (with an empty substitution $\sigma$),
\item $\tmp \in A_i$ (for $i \in \set{1, \ldots, l}$),
\item $\ppost(\delta_1) \in A_l$,
\item $\ppre(\delta_1) \not\in A_l$ (if
  $\ppre(\delta_1) \neq \ppost(\delta_1)$),
\item $A_l \eqm A_g$,
\item if $\noopconcept{c} \in A_k$ (where $c \in \const_0$), then
  $\noopconcept{c} \in A_i$ (for $i \in \set{1, \ldots, l}$).
\end{compactitem}

Now, since $A_l \eqm A_g$, $\scmap_k = \scmap_g$,
$\ppre(\delta_2) \in A_l$, then we have
$\tup{A_g, \scmap_g, \delta_2} \mimic \tup{A_l, \scmap_k}$.  Hence, by
induction hypothesis again, there exists states $\tup{A_i, m_k}$, and
actions $\act_i$ (for $i \in \set{l+1, \ldots, n}$, and $n \geq 0$)
such that
\begin{compactitem}
\item $\tup{A_l, m_k}\exect{\act_{l+1}\sigma} \tup{A_{l+1},
    m_k}\exect{\act_{l+2}\sigma} \cdots $ \\
\hspace*{40mm} $\cdots\exect{\act_n\sigma} \tup{A_n, m_k}$\\
  (with an empty substitution
  $\sigma$),
\item $\tmp \in A_i$ (for $i \in \set{l+1, \ldots, n}$),
\item $\ppost(\delta_2) \in A_n$,
\item $\ppre(\delta_2) \not\in A_n$ (if $\ppre(\delta_2) \neq \ppost(\delta_2)$),
\item $A_n \eqm A_g$,
\item if $\noopconcept{c} \in A_l$ (where $c \in \const_0$), \\ then
  $\noopconcept{c} \in A_i$ (for $i \in \set{l+1, \ldots, n}$).
\end{compactitem}
Therefore, it is easy to see that the claim is proven.

\item[\textbf{[$\delta_g = \gif{\varphi}{\delta_1}{\delta_2}$]}.]
  Since
  $\final{\tup{A_g, \scmap_g, \gif{\varphi}{\delta_1}{\delta_2}}}$,
  then by the definition of final states we have either
\begin{compactenum}[\bf (1)]
\item $\final{\tup{A_g, \scmap_g, \delta_1}}$ and
  $\ask(\varphi, T, A) = \true$, or
\item $\final{\tup{A_g, \scmap_g, \delta_2}}$ and
  $\ask(\varphi, T, A) = \false$.
\end{compactenum}
For compactness of the proof, here we only show the case
\textbf{(1)}. The case \textbf{(2)} can be done similarly. 
Now, since
$\tup{A_g,\scmap_g, \gif{\varphi}{\delta_1}{\delta_2}} \mimic
\tup{A_k,\scmap_k}$, then $\ppre(\gif{\varphi}{\delta_1}{\delta_2})
\in A_k$. 
By the definition of $\tgkab$, 
we have
\begin{compactitem}
\item $\carule{\varphi \wedge \ppre(\gif{\varphi}{\delta_1}{\delta_2})}{\gamma_{if}()} \in \procset$,
\item $\gamma_{if}():\set{\map{\true} {\add \set{\ppre(\delta_1),
        \tmp},$ \\
      \hspace*{22.5mm}
      $\del \set{\ppre(\gif{\varphi}{\delta_1}{\delta_2})} }}$,
\item $\ppost(\gif{\varphi}{\delta_1}{\delta_2}) = \ppost(\delta_1)$.
\end{compactitem}

Then, by induction hypothesis, and also by observing how an action is
executed as well as the result of an action execution is constructed,
it is easy to see that the claim is proven.

\item[\textbf{[$\delta_g = \gwhile{\varphi}{\delta}$]}.]  Since
  $\final{\tup{A_g, \scmap_g, \gwhile{\varphi}{\delta}}}$,
  then by the definition of final states, we have either
\begin{compactenum}[\bf (1)]
\item $\ask(\varphi, T, A) = \false$, or
\item $\final{\tup{A_g, \scmap_g, \delta}}$ and
  $\ask(\varphi, T, A) = \true$.
\end{compactenum}
%
%
\begin{compactitem}

\item[\textbf{Proof for the case (1)}:] Now, since
  \[
  \tup{A_g,\scmap_g, \gwhile{\varphi}{\delta}} \mimic
  \tup{A_k,\scmap_k},
  \]
  then $\ppre(\gwhile{\varphi}{\delta}) \in A_k$.  By the definition
  of $\tgkab$, we have
%
%
%
\begin{compactitem}[$\bullet$]
\item
  $ \carule{\ppre(\gwhile{\varphi}{\delta}) \wedge (\neg \varphi \vee
    \noopconcept{noop}) }\\{\gamma_{endLoop}()}$, 
\item
  $\gamma_{endLoop}(): \set{\true  \rightsquigarrow \\
    \hspace*{10mm} \add \set{\ppost(\gwhile{\varphi}{\delta}), \tmp }, \\
    \hspace*{10mm} \del \set{\ppre(\gwhile{\varphi}{\delta}),
      \noopconcept{noop}}}$,
\end{compactitem}

Then, by induction hypothesis, and also by observing how an action is
executed as well as the result of an action execution is constructed,
it is easy to see that the claim is proved.

\item[\textbf{Proof for the case (2)}:] Now, since
  \[
  \tup{A_g,\scmap_g, \gwhile{\varphi}{\delta}} \mimic
  \tup{A_k,\scmap_k},
  \]
  then $\ppre(\gwhile{\varphi}{\delta}) \in A_k$.  By the definition
  of $\tgkab$, 
  we have
\smallskip
\begin{itemize}[$\bullet$]
\item
  $ \carule{\ppre(\gwhile{\varphi}{\delta}) \wedge \varphi \wedge \neg
    \noopconcept{noop} }\\{\gamma_{doLoop}()}$,

\item
  $ \carule{\ppre(\gwhile{\varphi}{\delta}) \wedge (\neg \varphi \vee
    \noopconcept{noop}) }\\{\gamma_{endLoop}()}$, 

\item 
$\gamma_{endLoop}(): \set{\true  \rightsquigarrow \\
    \hspace*{10mm} \add \set{\ppost(\gwhile{\varphi}{\delta}), \tmp }, \\
    \hspace*{10mm} \del \set{\ppre(\gwhile{\varphi}{\delta}),
      \noopconcept{noop}}}$,

\item
  $\gamma_{doLoop}(): \set{\true  \rightsquigarrow\\
    \hspace*{10mm} \add \set{\ppre(\delta), \noopconcept{noop},
      \tmp }, \\
    \hspace*{10mm} \del \set{\ppre(\gwhile{\varphi}{\delta}) }}$.

\end{itemize}
\smallskip


Hence, it is easy to see that we have
\[
\tup{A_k,\scmap_k} \exect{\gamma_{doLoop} \sigma} \tup{A'_k,\scmap_k}
\]
where $\sigma$ is an empty substitution, and
$\set{\tmp, \ppre(\delta), \noopconcept{noop}} \subseteq A_k'$. Hence
$\tup{A_g, \scmap_g, \delta} \mimic \tup{A_k',\scmap_k}$. Since
$\final{\tup{A_g, \scmap_g, \delta}}$ and
$\tup{A_g, \scmap_g, \delta} \mimic \tup{A_k',\scmap_k}$, by induction
hypothesis, then there exists states $\tup{A_i, m_k}$, and actions
$\act_i$ (for $i \in \set{1, \ldots, n}$, and $n \geq 0$)
such that
\smallskip
\begin{itemize}[$\bullet$]
\item
  $\tup{A_k', m_k} \exect{\act_1\sigma} \tup{A_1,
    m_k}\exect{\act_2\sigma} \cdots $ \\
  \hspace*{40mm}$\cdots \exect{\act_n\sigma} \tup{A_n, m_k}$ \\(with an empty
  substitution $\sigma$),
\item $\tmp \in A_i$ (for $i \in \set{1, \ldots, n}$),
\item $\ppost(\delta) \in A_n$, 
\item $\ppre(\delta) \not\in A_n$ (if
  $\ppost(\delta) \neq \ppre(\delta)$),
\item $A_n \eqm A_g$, and
\item if $\noopconcept{c} \in A'_k$ (where $c \in \const_0$), then
  $\noopconcept{c} \in A_i$ (for $i \in \set{1, \ldots, n}$).
\end{itemize}
\smallskip Hence we have
\[
\set{\ppost(\delta), \noopconcept{noop}, \tmp} \subseteq A_n. 
\]
Now,
since by the definition of $\tgprog$ we have that
$\ppost(\delta) = \ppre(\gwhile{\varphi}{\delta})$, then the action
$\gamma_{endLoop}$ is executable in $A_n$ (notice that we do not care
whether $\ask(\varphi, T, A) = \false$, or
$\ask(\varphi, T, A) = \true$ because $\noopconcept{noop} \in
A_n$). Hence we have
\[
\tup{A_n,\scmap_k} \exect{\gamma_{endLoop} \sigma} \tup{A'_n,\scmap_k}
\]
with $\set{\tmp, \ppost(\gwhile{\varphi}{\delta})} \subseteq A'_n$,
and $\noopconcept{noop} \not\in A'_n$ (which is fine since
$\noopconcept{noop} \not\in A_k$). Thus we have that the claim
is proven.
Intuitively, the idea for the proof of this case is that since
$\final{\tup{A_g, \scmap_g, \delta}}$, there is no action
executed and no one removes the flag made by concept name
$\noopconceptname$. In that situation, for the second iteration, no
matter whether $\varphi$ (the guard of the loop) is hold or not, we
can exit the loop and additionally keeping all assertions in the ABox
(except the special markers) stay the same. Essentially it reflects
the situation that in the corresponding S-GKAB, there is no transition
was made (since $\final{\tup{A_g, \scmap_g, \delta}}$).
\end{compactitem}
\end{itemize}
\ \ 
\end{proof}

\subsection{Reducing the Verification of S-GKABs Into S-KABs}

We exploit the property of J-Bisimulation in order to show that the
verification of \muladom properties over S-GKABs can be reduce to the
verification of S-KABs.
Essentially, we show that given an S-GKAB $\gkabsym$, its transition
system $\ts{\gkabsym}^{\filter_S}$ is J-bisimilar to the transition
system $\ts{\tgkab(\gkabsym)}^{S}$ of S-KAB $\tgkab(\gkabsym)$ that is
obtained from $\gkabsym$ via the translation $\tgkab$.
Consequently, we have that both transition systems
$\ts{\gkabsym}^{\filter_S}$ and $\ts{\tgkab(\gkabsym)}^{S}$ can not be
distinguished by any \muladom (in NNF) modulo the translation
$\tforj$.


As a start, below we show that given a state $s_1$ of an S-GKAB
transition system, and a state $s_2$ of its corresponding S-KAB
transition system such that $s_2$ mimics $s_1$, we have that if $s_1$
reaches $s_1'$ in one step, then it implies that there exists $s_2'$
reachable from $s_2$ (possibly through some intermediate states
$s^t_1, \ldots, s^t_n$ that contain $\tmp$) and $s_2'$ mimics $s_1'$.

\begin{lemma}\label{lem:prog-exec-bsim}
  Let
  \begin{inparaitem}[]
  \item $\gkabsym$ be an S-GKAB with transition system
    $\ts{\gkabsym}^{\filter_S}$,
  \item $\tgkab(\gkabsym)$ be an S-KAB (obtained from $\gkabsym$ through
    $\tgkab$) with transition system $\ts{\tgkab(\gkabsym)}^S$.
  \end{inparaitem}
  Consider two states
\begin{inparaenum}[]
\item $\tup{A_g,\scmap_g, \delta_g}$ of $\ts{\gkabsym}^{\filter_S}$,
  and
\item $\tup{A_k,\scmap_k}$ of $\ts{\tgkab(\gkabsym)}^S$ 
\end{inparaenum}
such that $\tup{A_g,\scmap_g, \delta_g} \mimic \tup{A_k,\scmap_k} $.
For every state $\tup{A'_g,\scmap'_g, \delta_g'}$ such that
$ \tup{A_g,\scmap_g, \delta_g} \gprogtrans{\alpha\sigma, \filter_S}
\tup{A'_g,\scmap'_g, \delta_g'}$
(for a certain action $\act$, a legal parameter assignment $\sigma$
and a service call substitution $\theta$),
there exist states $\tup{A'_k, m'_k}$, $\tup{A_i^t, m_k}$ (for
$i \in \set{1, \ldots, n}$, where $n \geq 0$), and actions $\act'$, $\act_i$
(for $i \in \set{1, \ldots, n}$, where $n \geq 0$)
such that
\begin{compactitem}
\item
  $\tup{A_k, m_k} \exect{\act_1\sigma_e} \tup{A_1^t, m_k}
  \exect{\act_2\sigma_e} \cdots$ \\ \hspace*{15mm} $\cdots \exect{\act_n\sigma_e} \tup{A_n^t,
    m_k} \exect{\act'\sigma} \tup{A'_k, m'_k}$ \\ where
  \begin{compactitem}
  \item $\sigma_e$ is an empty substitution, 
  \item $\act'$ is obtained from $\act$ through $\tgprog$,
  \item $\tmp \in A_i^t$ (for $i \in \set{1, \ldots, n}$), $\tmp \not\in A_k'$, 
  \end{compactitem}
\item 
  $\tup{A'_g,\scmap'_g, \delta'_g} \mimic \tup{A'_k,\scmap'_k} $.
\end{compactitem}
\end{lemma}
\begin{proof}
Let 
  \begin{compactitem}[$\bullet$]
  \item $\gkabsym = \tup{T, \initABox, \actset, \delta}$, and \\ 
    $\ts{\gkabsym}^{\filter_S} = \tup{\const, T, \stateset_g, s_{0g},
      \abox_g, \trans_g}$,
  \item $\tgkab(\gkabsym) = \tup{T, \initABox', \actset', \procset'}$
    and \\
    $\ts{\tgkab(\gkabsym)}^S = \tup{\const, T, \stateset_k, s_{0k},
      \abox_k, \trans_k}$ where
    \begin{compactenum}
    \item $\initABox' = \initABox \cup \set{\flagconcept{start}}$, and
    \item
      $\tgprog(\flagconcept{start}, \ginitprog, \flagconcept{end}) =
      \tup{\ppre, \ppost, \procset', \actset'}$.
    \end{compactenum}
\end{compactitem}
We prove by induction over the structure of $\delta$.

\smallskip
\noindent
\textbf{Base case:}
\begin{enumerate}
\item[\textbf{[$\delta_g = \gemptyprog$]}.] Since
  $\final{\tup{A_g,\scmap_g, \gemptyprog}}$, then there does
  not exists $\tup{A'_g,\scmap'_g, \delta'_g}$ such that
  \[
  \tup{A_g,\scmap_g, \delta_g} \gprogtrans{\act\sigma, \filter_S}
  \tup{A'_g,\scmap'_g, \delta'_g},
  \]
  Hence, we do not need to show anything.

\item[\textbf{[$\delta_g = \gact{Q(\vec{p})}{\act(\vec{p})}$]}.] For
  compactness of the presentation, let
  $a = \gact{Q(\vec{p})}{\act(\vec{p})}$.  Since
  $\tup{A_g,\scmap_g, a} \mimic \tup{A_k,\scmap_k} $, then
  $\ppre(a) \in A_k$, by 
  the definition of $\tgkab$, we have:
  \begin{compactitem}
  \item
    $\carule{Q(\vec{p}) \wedge \ppre(a)}{\act'(\vec{p})} \in
    \procset'$,
  \item $\act' \in \actset'$. 
  \end{compactitem}
Since 
\[
\tup{A_g,\scmap_g, a}
\gprogtrans{\act\sigma, \filter_S} 
\tup{A'_g,\scmap'_g, \gemptyprog},
\]
then $\sigma \in \ask(Q, T, A_g)$. Since $A_k \eqm A_g$, and $Q$ does
not use any special marker concept names, by Lemma
\ref{lem:ECQ-equal-ABox-modulo-markers} we have
$\ask(Q, T, A_g) = \Ans(Q, T, A_k)$ and hence
$\sigma \in \ask(Q, T, A_k)$.
Now, since $\ppre(a) \in A_k$, then $\act'$ is executable in $A_k$
with legal parameter assignment $\sigma$. Additionally, considering
\[
\begin{array}{l@{}l} \eff{\act'} = &\eff{\act} \\
                                   &\cup \set{\map{\true}{ \add \set{\ppost(a)} }}\\
                                   &\cup \set{\map{\true}{ \del
                                     \set{\ppre(a),
                                     \tmp} }}\\
                                   &\cup
                                     \set{\map{\noopconcept{x}}{\del
                                     \noopconcept{x} }},
  \end{array}
\]
Then it is easy to see that we have
$\addfactss{A_g}{\act\sigma} = \addfactss{A_k}{\act'\sigma}$, and
hence
$\calls{\addfactss{A_g}{\act\sigma}} =
\calls{\addfactss{A_k}{\act'\sigma}}$.
Thus 
we have $\theta \in \calls{\addfactss{A_k}{\act'\sigma}}$.
Now, since $\scmap_g' = \theta \cup \scmap_g$, $\scmap_k = \scmap_g$
and $\theta \in \calls{\addfactss{A_k}{\act'\sigma}}$, we can
construct $\scmap_k' = \theta \cup \scmap_k$. Therefore it is easy to
see that there exists $\tup{A'_k, \scmap'_k}$, such that
\[
\tup{A_k, \scmap_k} \exect{\act'\sigma} \tup{A'_k, \scmap'_k}
\]
(with service call substition $\theta$) and $A'_g \eqm A_k'$ (by
considering how $A_k'$ is constructed), $\scmap'_g = \scmap'_k$. By
the definition of $\tgprog$ (in the translation of an action
invocation) we also have $\ppre(\gemptyprog) \in A_k'$. Thus the claim
is proven.
\end{enumerate}

\smallskip
\noindent
\textbf{Inductive case:}
\begin{enumerate}
\item[\textbf{[$\delta_g = \delta_1|\delta_2$]}.] Since 
\[
\tup{A_g,\scmap_g, \delta_1|\delta_2} \gprogtrans{\act\sigma,
  \filter_S} \tup{A'_g,\scmap'_g, \delta'},
\]
then, there are two cases, that is either
\begin{compactenum}[\bf (1)]
\item
  $\tup{A_g,\scmap_g, \delta_1} \gprogtrans{\act\sigma, \filter_S}
  \tup{A'_g,\scmap'_g, \delta'}$, or
\item
  $\tup{A_g,\scmap_g, \delta_2} \gprogtrans{\act\sigma, \filter_S}
  \tup{A'_g,\scmap'_g, \delta'}$.
\end{compactenum}
Here we only give the derivation for the first case, the second case
is similar. Since
$\tup{A_g,\scmap_g, \delta_1|\delta_2} \mimic \tup{A_k,\scmap_k} $,
then $A_g \eqm A_k$, $\scmap_g = \scmap_k$, and
$\ppre(\delta_1|\delta_2) \in A_k$.  By the definition of $\tgkab$ and
Lemma~\ref{lem:program-pre-post}, we have
\begin{compactitem}[$\bullet$]
\item
  $\carule{\ppre(\delta_1|\delta_2)}{\gamma_{\delta_1}()} \in
  \procset'$
\item $\gamma_{\delta_1} \in \actset'$, where \\
$ \gamma_{\delta_1}():\set{\map{\true} \\ \hspace*{10mm}{\add \set{\ppre(\delta_1),
        \tmp}, \del \set{\delta_1|\delta_2} }}, $
\end{compactitem}
Since $\ppre(\delta_1|\delta_2) \in A_k$, it is easy to see that 
\[
\tup{A_k, m_k} \exect{\gamma_{\delta_1}\sigma_t}\tup{A_t, m_k}
\]
where $\sigma_t$ is an empty substitution,
$\set{\ppre(\delta_1), \tmp} \in A_t$, and $A_t \eqm A_k$.
Since
$A_t \eqm A_k$ and $A_g \eqm A_k$, it is easy to see that
$A_g \eqm A_t$. Since $A_g \eqm A_t$, $\scmap_g = \scmap_k$, and
$\ppre(\delta_1) \in A_t$, then we have
$\tup{A_g,\scmap_g, \delta_1} \mimic \tup{A_t,\scmap_k} $. 
Therefore, since
$\tup{A_g, \scmap_g, \delta_1} \gprogtrans{\act\sigma, \filter_S}
\tup{A_g', \scmap_g', \delta_1'}$
and $\tup{A_g,\scmap_g, \delta_1} \mimic \tup{A_t,\scmap_k} $, by
induction hypothesis, it is easy to see that the claim is proven for
this case.

\item[\textbf{[$\delta_g = \delta_1;\delta_2$]}.] There are two cases:
\begin{compactenum}[\bf (1)]
\item
  $\tup{A_g,\scmap_g, \delta_1;\delta_2} \gprogtrans{\act\sigma,
    \filter_S} \tup{A'_g,\scmap'_g, \delta_1';\delta_2}, $
\item
  $\tup{A_g,\scmap_g, \delta_1;\delta_2} \gprogtrans{\act\sigma,
    \filter_S} \tup{A'_g,\scmap'_g, \delta_2'}, $
\end{compactenum}

\smallskip
\noindent
\textbf{Case (1).} Since
\[
\tup{A_g,\scmap_g, \delta_1;\delta_2} \gprogtrans{\act\sigma,
  \filter_S} \tup{A'_g,\scmap'_g, \delta_1';\delta_2},
\]
then we have
\[
\tup{A_g,\scmap_g, \delta_1} \gprogtrans{\act\sigma, \filter_S}
\tup{A'_g,\scmap'_g, \delta_1'},
\]
Since
$\tup{A_g,\scmap_g, \delta_1;\delta_2} \mimic \tup{A_k,\scmap_k} $,
then $A_g \eqm A_k$, $\scmap_g = \scmap_k$, and
$\ppre(\delta_1;\delta_2) \in A_k$.  By the definition of $\tgkab$ and
Lemma~\ref{lem:program-pre-post}, it is easy to see that
$\ppre(\delta_1;\delta_2) = \ppre(\delta_1)$, and hence because
$\ppre(\delta_1;\delta_2) \in A_k$, we have $\ppre(\delta_1) \in A_k$.
Now, since $A_g \eqm A_k$, $\scmap_g = \scmap_k$,
$\ppre(\delta_1) \in A_k$, then we have
$\tup{A_g,\scmap_g, \delta_1} \mimic \tup{A_k,\scmap_k} $. Thus, since
we also have
\[
\tup{A_g,\scmap_g, \delta_1} \gprogtrans{\act\sigma, \filter_S}
\tup{A'_g,\scmap'_g, \delta_1'},
\]
by using induction hypothesis we have that the claim is proven.

\smallskip
\noindent
\textbf{Case (2).}  Since
$\tup{A_g,\scmap_g, \delta_1;\delta_2} \gprogtrans{\act\sigma,
  \filter_S} \tup{A'_g,\scmap'_g, \delta_2'}, $,
then we have $\final{\tup{A_g,\scmap_g, \delta_1}}$, and
\[
\tup{A_g,\scmap_g, \delta_2} \gprogtrans{\act\sigma, \filter_S}
\tup{A'_g,\scmap'_g, \delta_2'},
\]

Since $\final{\tup{A_g,\scmap_g, \delta_1}}$ and
$\tup{A_g,\scmap_g, \delta_g} \mimic \tup{A_k,\scmap_k}$, by
Lemma~\ref{lem:final-state-add-transition}, there exist states
$\tup{A_i, m_k}$ and actions $\act_i$ (for $i \in \set{1, \ldots, n}$, and
$n \geq 0$)
such that
\begin{compactitem}
\item
  $\tup{A_k, m_k} \exect{\act_1\sigma} \tup{A_1,
    m_k}\exect{\act_2\sigma} \cdots $ \\ \hspace*{35mm} $\cdots \exect{\act_n\sigma} \tup{A_n,
    m_k}$ \\ (with empty an empty substitution $\sigma$),
\item $\tmp \in A_i$ (for $i \in \set{1, \ldots, n}$),
\item $\ppost(\delta_1) \in A_n$, $\ppre(\delta_1) \not\in A_n$, 
   and 
\item $A_n \eqm A_g$.
\end{compactitem}
Since
$\tup{A_g,\scmap_g, \delta_1;\delta_2} \mimic \tup{A_k,\scmap_k} $,
then $A_g \eqm A_k$, $\scmap_g = \scmap_k$, and
$\ppre(\delta_1;\delta_2) \in A_k$.  
By the definition of $\tgkab$ and
Lemma~\ref{lem:program-pre-post}, it is easy to see that
\begin{compactitem}
\item $\ppre(\delta_1;\delta_2) = \ppre(\delta_1)$, 
\item $\ppost(\delta_1) = \ppre(\delta_2)$, 
\item $\ppost(\delta_1;\delta_2) = \ppost(\delta_2)$, 
\end{compactitem}
Hence, because $\ppost(\delta_1) \in A_n$, and
$\ppost(\delta_1) = \ppre(\delta_2)$, we have
$\ppre(\delta_2) \in A_k$. 
Now, since $A_g \eqm A_n$, $\scmap_g = \scmap_k$,
$\ppre(\delta_2) \in A_n$, then we have
$\tup{A_g,\scmap_g, \delta_2} \mimic \tup{A_n,\scmap_k} $. Thus, since
we also have
\[
\tup{A_g,\scmap_g, \delta_2} \gprogtrans{\act\sigma, \filter_S}
\tup{A'_g,\scmap'_g, \delta_2'},
\]
by using induction hypothesis we have that the claim is proven.

\item[\textbf{[$\delta_g = \gif{\varphi}{\delta_1}{\delta_2}$]}.]
  There are two cases:
\begin{compactenum}[\bf (1)]
\item
  $\tup{A_g, \scmap_g, \gif{\varphi}{\delta_1}{\delta_2}}
  \gprogtrans{\alpha\sigma, \filter_S} \tup{A_g', \scmap_g',
    \delta_1'}$,
\item
  $\tup{A_g, \scmap_g, \gif{\varphi}{\delta_1}{\delta_2}}
  \gprogtrans{\alpha\sigma, \filter_S} \tup{A_g', \scmap_g',
    \delta_2'}$.
\end{compactenum}
Here we only consider the first case. The second case is similar. 

\noindent
\smallskip \textbf{Case (1).} Since
\[
\tup{A_g, \scmap_g, \gif{\varphi}{\delta_1}{\delta_2}}
\gprogtrans{\alpha\sigma, \filter_S} \tup{A_g', \scmap_g', \delta_1'},
\] 
then we have
\[
\tup{A_g, \scmap_g, \delta_1} \gprogtrans{\act\sigma, \filter_S}
\tup{A_g', \scmap_g', \delta_1'}.
\] 
with $\ask(\varphi, T, A_g) = \true$.  
Since
$\tup{A_g,\scmap_g, \gif{\varphi}{\delta_1}{\delta_2}} \mimic
\tup{A_k,\scmap_k} $,
then $A_g \eqm A_k$, $\scmap_g = \scmap_k$, and
$\ppre(\gif{\varphi}{\delta_1}{\delta_2}) \in A_k$.
By the definition of $\tgkab$ and Lemma~\ref{lem:program-pre-post}, we have
\begin{compactitem}
\item
  $\carule{\ppre(\gif{\varphi}{\delta_1}{\delta_2}) \wedge
    \varphi}{\gamma_{if}()} \in \procset'$
\item $\gamma_{if} \in \actset'$, where \\
  $
\begin{array}{l@{}l}
\gamma_{if}():\set{\map{\true} {&\add \set{\ppre(\delta_1), \tmp}, \\
      &\del \set{\ppre(\gif{\varphi}{\delta_1}{\delta_2})} }},
\end{array}
$
\end{compactitem}
Since $A_k \eqm A_g$, $\ask(\varphi, T, A_g) = \true$, and $\varphi$
does not use any special marker concept names, by Lemma
\ref{lem:ECQ-equal-ABox-modulo-markers} we have
$\ask(\varphi, T, A_k) = \true$.
Now, since $\ppre(\gif{\varphi}{\delta_1}{\delta_2}) \in A_k$, and
$\ask(\varphi, T, A_k) = \true$, it is easy to see that
\[
\tup{A_k, m_k} \exect{\gamma_{if}\sigma_t}\tup{A_t, m_k}
\]
where $\sigma_t$ is an empty substitution,
$\set{\ppre(\delta_1), \tmp} \in A_t$, and $A_t \eqm A_k$. Since
$A_t \eqm A_k$ and $A_g \eqm A_k$, it is easy to see that
$A_g \eqm A_t$. Since $A_g \eqm A_t$, $\scmap_g = \scmap_k$, and
$\ppre(\delta_1) \in A_t$, then we have
$\tup{A_g,\scmap_g, \delta_1} \mimic \tup{A_t,\scmap_k} $. 
Thus, since
$\tup{A_g, \scmap_g, \delta_1} \gprogtrans{\act\sigma, \filter_S}
\tup{A_g', \scmap_g', \delta_1'}$
and $\tup{A_g,\scmap_g, \delta_1} \mimic \tup{A_t,\scmap_k} $, by
induction hypothesis, it is easy to see that the claim is proven for
this case.

\item[\textbf{[$\delta_g = \gwhile{\varphi}{\delta}$]}.]  Since
\[  
\tup{A_g, \scmap_g,
  \gwhile{\varphi}{\delta}} \gprogtrans{\act\sigma,
  \filter_S} \tup{A_g', \scmap_g', \delta';\gwhile{\varphi}{\delta}},
\]
then we have $\ask(\varphi, T, A) = \true$ and
\[
\tup{A_g, \scmap_g, \delta} \gprogtrans{\act\sigma, \filter_S}
\tup{A_g', \scmap_g', \delta'}.
\]
Since
$\tup{A_g,\scmap_g, \gwhile{\varphi}{\delta}} \mimic
\tup{A_k,\scmap_k}$,
then $A_g \eqm A_k$, $\scmap_g = \scmap_k$, and
$\ppre(\gwhile{\varphi}{\delta}) \in A_k$.
By the definition of $\tgkab$ and Lemma~\ref{lem:program-pre-post}, we have
\begin{compactitem}[$\bullet$]
\item
  $ \carule{(\ppre(\gwhile{\varphi}{\delta}) \vee \ppost(\delta)) \wedge \varphi \wedge \neg
    \noopconcept{noop} }{\gamma_{doLoop}()} \in \procset'$,
\item $\gamma_{doLoop}(): \set{\true  \rightsquigarrow\\
    \hspace*{10mm} \add \set{\ppre(\delta), \noopconcept{noop},
      \tmp }, \\
    \hspace*{10mm} \del \set{ \ppre(\gwhile{\varphi}{\delta}),
      \ppost(\delta) }}$,
\end{compactitem}
Since $A_k \eqm A_g$, $\ask(\varphi, T, A_g) = \true$, and $\varphi$
does not use any special marker concept names, by Lemma
\ref{lem:ECQ-equal-ABox-modulo-markers} we have
$\ask(\varphi, T, A_k) = \true$. Additionally, it is easy to see from
the definition of $\tgprog$ that $\noopconcept{noop} \not\in A_k$.
Now, since $\ppre(\gif{\varphi}{\delta_1}{\delta_2}) \in A_k$,
$\ask(\varphi, T, A_k) = \true$, and $\noopconcept{noop} \not\in A_k$,
it is easy to see that
\[
\tup{A_k, m_k} \exect{\gamma_{doLoop}\sigma_t}\tup{A_t, m_k}
\]
%
%
where $\sigma_t$ is an empty substitution,
$\set{\ppre(\delta), \tmp} \in A_t$, and $A_t \eqm A_k$. Since
$A_t \eqm A_k$ and $A_g \eqm A_k$, it is easy to see that
$A_g \eqm A_t$. Since $A_g \eqm A_t$, $\scmap_g = \scmap_k$, and
$\ppre(\delta) \in A_t$, then we have
$\tup{A_g,\scmap_g, \delta} \mimic \tup{A_t,\scmap_k} $.
Thus, since
$\tup{A_g, \scmap_g, \delta} \gprogtrans{\act\sigma, \filter_S}
\tup{A_g', \scmap_g', \delta'}$
and $\tup{A_g,\scmap_g, \delta} \mimic \tup{A_t,\scmap_k} $, by
induction hypothesis, 
there exist states $\tup{A'_k, m'_k}$, $\tup{A_i^t, m_k}$ (for
$i \in \set{1, \ldots, n}$, where $n \geq 0$), and actions $\act'$, $\act_i$
(for $i \in \set{1, \ldots, n}$, where $n \geq 0$)
such that
\begin{compactitem}
\item
  $\tup{A_t, m_k} \exect{\act_1\sigma_e} \tup{A_1^t, m_k}
  \exect{\act_2\sigma_e} \cdots $ \\ \hspace*{20mm} $\cdots \exect{\act_n\sigma_e} \tup{A_n^t,
    m_k} \exect{\act'\sigma} \tup{A'_k, m'_k}$ \\ where
  \begin{compactitem}
  \item $\sigma_e$ is an empty substitution, 
  \item $\act'$ is obtained from $\act$ through $\tgprog$,
  \item $\tmp \in A_i^t$ (for $i \in \set{1, \ldots, n}$), $\tmp \not\in A_k'$, 
  \end{compactitem}
\item 
  $\tup{A'_g,\scmap'_g, \delta'_g} \mimic \tup{A'_k,\scmap'_k} $.
\end{compactitem}
The proof for this case is then completed by also observing that by
the definition of program execution relation (on the case of while
loops), we have that we repeat the while loop at the end of the
execution of program $\delta$, and this situation is captured in the
definition of $\tgprog$ by having that
$\ppost(\delta) = \ppre(\gwhile{\varphi}{\delta})$.
\end{enumerate}
\ \ 
\end{proof}

We now proceed to show another crucial lemma for showing the
bisimulation between S-GKAB transition system and the transition
system of its corresponding S-KAB that is obtained via
$\tgprog$. Basically, we show that given a state $s_1$ of an S-GKAB
transition system, and a state $s_2$ of its corresponding S-KAB
transition system such that $s_2$ mimics $s_1$, 
we have that if $s_2$ reaches $s_2'$ (possibly through some
intermediate states $s^t_1, \ldots, s^t_n$ that contains $\tmp$), then
$s_1$ reaches $s_1'$ in one step and $s_1'$ is mimicked by $s_2'$.

\begin{lemma}\label{lem:kab-transition-bsim}
  Let
  \begin{inparaitem}[]
  \item $\gkabsym  = \tup{T, \initABox, \actset, \delta}$ be an S-GKAB with transition system
    $\ts{\gkabsym}^{\filter_S}$,
  \item $\tgkab(\gkabsym) = \tup{T, \initABox', \actset', \procset'}$
    be an S-KAB (obtained from $\gkabsym$ through $\tgkab$) with
    transition system $\ts{\tgkab(\gkabsym)}^S$,
  \end{inparaitem}
  where
  \begin{compactenum}
  \item $\initABox' = \initABox \cup \set{\flagconcept{start}}$, and
  \item
    $\tgprog(\flagconcept{start}, \ginitprog, \flagconcept{end}) =
    \tup{\ppre, \ppost, \procset', \actset'}$.
  \end{compactenum}
    Consider two states
  \begin{inparaenum}[]
  \item $\tup{A_g,\scmap_g, \delta_g}$ of $\ts{\gkabsym}^{\filter_S}$,
    and
  \item $\tup{A_k,\scmap_k}$ of $\ts{\tgkab(\gkabsym)}^S$ 
  \end{inparaenum}
  such that $\tup{A_g,\scmap_g, \delta_g} \mimic \tup{A_k,\scmap_k}$.
  For every state $\tup{A'_k, m'_k}$
%
  such that
  \begin{compactitem}
  \item there exist $\tup{A_i^t, m_k}$ (for $i \in \set{1,\ldots, n}$,
    $n \geq 0$), and
  \item 
    $\tup{A_k, m_k} \exect{\act_1\sigma_e} \tup{A_1^t, m_k}
    \exect{\act_2\sigma_e} \cdots $ \\ \hspace*{20mm} $\cdots \exect{\act_n\sigma_e} \tup{A_n^t,
      m_k} \exect{\act'\sigma} \tup{A'_k, m'_k}$ \\ where
    \begin{compactitem}
    \item $\sigma_e$ is an empty substitution,
    \item $\act_i \in \actsettmpa$ (for $i \in \set{1, \ldots, n}$),
    \item $\act' \in \actsettmpd$,
    \item $\carule{Q(\vec{p})}{\act'(\vec{p})} \in \procset'$, 
    \item if $n = 0$, then $\sigma \in \Ans(Q,T,A_k)$, otherwise
      $\sigma \in \Ans(Q,T,A_n^t)$,
    \item $\tmp \in A_i^t$ (for $i \in \set{1,\ldots, n}$),
      $\tmp \not\in A_k'$,
    \end{compactitem}
  \end{compactitem}
  then there exists a state $\tup{A'_g,\scmap'_g, \delta_g'}$ such
  that
  \begin{compactitem}
  \item
    $ \tup{A_g,\scmap_g, \delta_g} \gprogtrans{\act\sigma,
      \filter_S} \tup{A'_g,\scmap'_g, \delta_g'}, $
  \item $\act'$ is obtained from the translation of a certain action
    invocation $\gact{Q(\vec{p})}{\act(\vec{p})}$ via $\tgprog$, 
  \item 
    $\tup{A'_g,\scmap'_g, \delta'_g} \mimic \tup{A'_k,\scmap'_k} $.
  \end{compactitem}
\ \ 
\end{lemma}
\begin{proof}
  Let
\begin{compactitem}
\item
  $\ts{\gkabsym}^{\filter_S}\!=\!\tup{\const, T, \stateset_g, s_{0g},
    \abox_g, \trans_g}$, and
\item
  $\ts{\tgkab(\gkabsym)}^S\!=\!\tup{\const, T, \stateset_k, s_{0k},
    \abox_k, \trans_k}$.
\end{compactitem}
%
  We prove the claim by induction over the structure of $\delta$.

\smallskip
\noindent
\textbf{Base case:}
\begin{enumerate}
\item[\textbf{[$\delta_g = \gemptyprog$]}.] Since
  $\delta_g = \gemptyprog$, by the definition of $\tgprog$, there must
  not exist $\tup{A'_k, m'_k}$ and $\tup{A_i^t, m_k}$ (for $i \in
  \set{1,\ldots, n}$, and $n \geq 0$) such that \\
$
    \tup{A_k, m_k} \exect{\act_1\sigma_e} \tup{A_1^t, m_k}
    \exect{\act_2\sigma_e} \cdots \\ \hspace*{20mm} \cdots \exect{\act_n\sigma_e} \tup{A_n^t,
    m_k} \exect{\act'\sigma} \tup{A'_k, m'_k},
$ \\
  where
    \begin{compactitem}
    \item $\sigma_e$ is an empty substitution,
    \item $\act_i \in \actsettmpa$ (for $i \in \set{1,\ldots, n}$),
    \item $\act' \in \actsettmpd$,
    \item $\carule{Q(\vec{p})}{\act'(\vec{p})} \in \procset'$, 
    \item if $n = 0$, then $\sigma \in \Ans(Q,T,A_k)$, otherwise
      $\sigma \in \Ans(Q,T,A_n^t)$,
    \item $\tmp \in A_i^t$ (for $i \in \set{1,\ldots, n}$),
      $\tmp \not\in A_k'$,
    \end{compactitem}

    The intuition is that the translation $\tgprog$ translates empty
    programs into actions that only add $\tmp$. 

  \item[\textbf{[$\delta_g = \gact{Q(\vec{p})}{\act(\vec{p})}$]}.]
    Assume that there exist states $\tup{A'_k, m'_k}$ and 
    $\tup{A_i^t, m_k}$ (for $i \in \set{1,\ldots, n}$, and $n \geq 0$) such
    that \\
    $\tup{A_k, m_k} \exect{\act_1\sigma_e} \tup{A_1^t, m_k}
    \exect{\act_2\sigma_e} \cdots \\ \hspace*{20mm} \cdots
    \exect{\act_n\sigma_e} \tup{A_n^t, m_k} \exect{\act'\sigma}
    \tup{A'_k, m'_k}$
  where
    \begin{compactitem}
    \item $\sigma_e$ is an empty substitution,
    \item $\act_i \in \actsettmpa$ (for $i \in \set{1,\ldots, n}$),
    \item $\act' \in \actsettmpd$,
    \item $\carule{Q(\vec{p})}{\act'(\vec{p})} \in \procset'$, 
    \item if $n = 0$, then $\sigma \in \Ans(Q,T,A_k)$, otherwise
      $\sigma \in \Ans(Q,T,A_n^t)$,
    \item $\tmp \in A_i^t$ (for $i \in \set{1,\ldots, n}$),
      $\tmp \not\in A_k'$,
    \end{compactitem}
    Additionally, w.l.o.g., let $\theta$ be the corresponding
    substitution that evaluates service calls in the transition
    \[
    \tup{A_n^t, m_k} \exect{\act'\sigma} \tup{A'_k, m'_k}
    \]
    Now, since
    $\tup{A_g,\scmap_g, \gact{Q(\vec{p})}{\act(\vec{p})} } \mimic
    \tup{A_k,\scmap_k} $,
    then $\ppre(\gact{Q(\vec{p})}{\act(\vec{p})}) \in A_k$.  Moreover,
    since we also have $\delta_g = \gact{Q(\vec{p})}{\act(\vec{p})}$,
    by the definition of $\tgkab$ and
    Lemma~\ref{lem:action-invocation-unique-start-flag}, we have that
    $\act'$ must be obtained from $\gact{Q(\vec{p})}{\act(\vec{p})}$
    and hence we have that
    $\carule{Q(\vec{p}) \wedge
      \ppre(\gact{Q(\vec{p})}{\act(\vec{p})})}{\act'(\vec{p})} \in
    \procset'$.

    Now, by our assumption above and by
    Lemma~\ref{lem:temp-state-produced-by-temp-act}, we have that
    $A_k \eqm A_1^t$, $A_i^t \eqm A_{i+1}^t$ (for
    $i \in \set{1,\ldots, n}$), and $A_n^t \eqm A_k'$, and hence we
    have $A_g \eqm A_n^t$. Since $A_g \eqm A_n^t$, and $Q$ does not
    use any special marker concept names, by
    Lemma~\ref{lem:ECQ-equal-ABox-modulo-markers} we have
    $\ask(Q, T, A_g) = \Ans(Q, T, A_n^t)$ and hence
    $\sigma \in \ask(Q, T, A_g)$.
  Additionally, considering
  \[
  \begin{array}{l@{}l} \eff{\act'} = &\eff{\act} \\
                                  &\cup \set{\map{\true}{ \add
                                    \set{\ppost(\gact{Q(\vec{p})}{\act(\vec{p})})}
                                    }}\\
                                  &\cup \set{\map{\true}{ \del
                                    \set{\ppre(\gact{Q(\vec{p})}{\act(\vec{p})}),\\
                                  &\hspace*{22mm}\tmp} }}\\
                                  &\cup
                                    \set{\map{\noopconcept{x}}{\del
                                    \noopconcept{x} }},
  \end{array}
  \]
  Then it is easy to see that we have
\[
 \addfactss{A_g}{\act\sigma} =\addfactss{A_k}{\act'\sigma}
  \setminus \ppost(\gact{Q(\vec{p})}{\act(\vec{p})}), 
\]
and hence
  \[
  \calls{\addfactss{A_g}{\act\sigma}} = \calls{\addfactss{A_k}{\act'\sigma}}
  \]
  and $\theta \in \calls{\addfactss{A_g}{\act\sigma}}$.
   Since $\scmap_k' = \theta \cup \scmap_k$, $\scmap_k = \scmap_g$ and
   $\theta \in \calls{\addfactss{A_g}{\act\sigma}}$, we can
   construct $\scmap_g' = \theta \cup \scmap_g$. 
   Thus, it is easy to see that there exists
   $\tup{A'_g,\scmap'_g, \gemptyprog}$ such that
  \[
  \tup{A_g,\scmap_g, \gact{Q(\vec{p})}{\act(\vec{p})}  } \gprogtrans{\act\sigma, \filter_S}
  \tup{A'_g,\scmap'_g, \gemptyprog},
  \]
  (with service call substition $\theta$) and $A'_g \eqm A_k'$ (by
  considering how $A_k'$ is constructed), $\scmap'_g = \scmap'_k$.
  Moreover, by the definition of $\tgprog$ (in the translation of an
  action invocation) we also have $\ppre(\gemptyprog) \in A_k'$
  (because $\ppost(\gact{Q(\vec{p})}{\act(\vec{p})}) = \ppre(\gemptyprog)$). Thus
  the claim is proven.
\end{enumerate}

\smallskip
\noindent
\textbf{Inductive case:}
\begin{enumerate}
\item[\textbf{[$\delta_g = \delta_1|\delta_2$]}.]  Assume that there
  exist states $\tup{A'_k, m'_k}$ and $\tup{A_i^t, m_k}$ (for
  $i \in \set{1,\ldots, n}$, and $n \geq 0$) such that\\
  $\tup{A_k, m_k} \exect{\act_1\sigma_e} \tup{A_1^t, m_k}
  \exect{\act_2\sigma_e} \cdots$
  \\
  \hspace*{25mm}$\cdots \exect{\act_n\sigma_e} \tup{A_n^t, m_k}
  \exect{\act'\sigma} \tup{A'_k, m'_k}$
  where
    \begin{compactitem}
    \item $\sigma_e$ is an empty substitution,
    \item $\act_i \in \actsettmpa$ (for $i \in \set{1,\ldots, n}$), 
    \item $\act' \in \actsettmpd$, 
    \item $\carule{Q(\vec{p})}{\act'(\vec{p})} \in \procset'$, 
    \item if $n = 0$, then $\sigma \in \Ans(Q,T,A_k)$, otherwise
      $\sigma \in \Ans(Q,T,A_n^t)$,
    \item $\tmp \in A_i^t$ (for $i \in \set{1,\ldots, n}$), 
      $\tmp \not\in A_k'$,
    \end{compactitem}
    By the definition of $\tgprog$ on the translation of a program of
    the form $\delta_1 | \delta_2$ and
    $\gact{Q(\vec{p})}{\act(\vec{p})}$, there exists
    $j \in \set{1,\ldots,n}$ such that
    $\ppre(\delta_1|\delta_2) \in A_{j-1}^t$, and either
    \begin{compactenum}
    \item $\act_j = \gamma_{\delta_1}$, and
      $\ppre(\delta_1) \in A_j^t$, or
    \item $\act_j = \gamma_{\delta_2}$, and
      $\ppre(\delta_2) \in A_j^t$.
    \end{compactenum}
    where $\gamma_{\delta_1}$ and $\gamma_{\delta_2}$ are the actions
    obtained from the translation of $\delta_1 | \delta_2$ by
    $\tgprog$, and it might be the case that $A_k = A_{j-1}^t$ (if
    $j = 1$). Now, by our assumption above and by
    Lemma~\ref{lem:temp-state-produced-by-temp-act}, we have that
    $A_k \eqm A_j^t$, and hence it is easy to see that we also have
    $A_g \eqm A_j^t$. Thus, essentially we have \\
      $\tup{A_{j-1}^t, m_k} \exect{\act_j\sigma_e} \tup{A_{j}^t, m_k}
      \exect{\act_{j+1}\sigma_e} \cdots $ \\ \hspace*{25mm}$\cdots \exect{\act_n\sigma_e}
      \tup{A_n^t, m_k} \exect{\act'\sigma} \tup{A'_k, m'_k}$,
%
  and 
  \begin{compactenum}
  \item if $\act_j = \gamma_{\delta_1}$, then
    $\tup{A_g, \scmap_g, \delta_1} \mimic \tup{A_j^t,\scmap_k}$
    (because $A_g \eqm A_j^t$, $\scmap_g = \scmap_k$,
    $\ppre(\delta_1) \in A_j^t$), otherwise
  \item if $\act_j = \gamma_{\delta_2}$, then
    $\tup{A_g, \scmap_g, \delta_2} \mimic \tup{A_j^t,\scmap_k}$
    (because $A_g \eqm A_j^t$, $\scmap_g = \scmap_k$,
    $\ppre(\delta_2) \in A_j^t$).
  \end{compactenum}
  Therefore by induction hypothesis, 
  it is easy to see that the claim is proven by also considering
  the definition of program execution relation.

\item[\textbf{[$\delta_g = \delta_1;\delta_2$]}.]  Assume that there
  exist states $\tup{A'_k, m'_k}$ and $\tup{A_i^t, m_k}$ (for
  $i \in \set{1,\ldots, n}$, and $n \geq 0$) such that \\
    $\tup{A_k, m_k} \exect{\act_1\sigma_e} \tup{A_1^t, m_k}
    \exect{\act_2\sigma_e} \cdots $ \\ \hspace*{25mm}$\cdots \exect{\act_n\sigma_e} \tup{A_n^t,
      m_k} \exect{\act'\sigma} \tup{A'_k, m'_k}$
  where
    \begin{compactitem}
    \item $\sigma_e$ is an empty substitution,
    \item $\act_i \in \actsettmpa$ (for $i \in \set{1,\ldots, n}$), 
    \item $\act' \in \actsettmpd$, 
    \item $\carule{Q(\vec{p})}{\act'(\vec{p})} \in \procset'$, 
    \item if $n = 0$, then $\sigma \in \Ans(Q,T,A_k)$, otherwise
      $\sigma \in \Ans(Q,T,A_n^t)$,
    \item $\tmp \in A_i^t$ (for $i \in \set{1,\ldots, n}$),
      $\tmp \not\in A_k'$,
    \end{compactitem}
    By the definition of $\tgprog$ on the translation of
    $\delta_1 ; \delta_2$ and $\gact{Q(\vec{p})}{\act(\vec{p})}$, as
    well as Lemma~\ref{lem:final-state-add-transition}, then there are two cases:

%
    \begin{compactenum}[\bf (a)]
    \item there exists $j \in \set{1,\ldots,n}$ such that
      $\ppre(\delta_1) \in A_j^t$, and
      $\ppost(\delta_1) \not\in A_l^t$ for $l \in \set{j+1,\ldots,n}$
      (capturing the case when $\tup{A_g, \scmap_g, \delta_1}$ is not
      a final state).
    \item there exists $j \in \set{1,\ldots,n-1}$ and
      $l \in \set{j+1,\ldots,n}$ such that
      $\ppre(\delta_1) \in A_j^t$, $\ppost(\delta_1) \in A_{l}^t$,
      $\ppre(\delta_2) \in A_{l}^t$,
      $\ppost(\delta_1) = \ppre(\delta_2)$ (capturing the case when
      $\final{\tup{A_g, \scmap_g, \delta_1}}$).
    \end{compactenum}
%
%
    Now, by our assumption above and by
    Lemma~\ref{lem:temp-state-produced-by-temp-act}, we have that
    \begin{compactenum}
    \item[- \textbf{For the case} \textbf{(a)}:] $A_k \eqm A_j^t$, and hence it is
      easy to see that we also have $A_g \eqm A_j^t$. Thus we have
      that $\tup{A_g, \scmap_g, \delta_1} \mimic
      \tup{A_j^t,\scmap_k}$.
    \item[- \textbf{For the case} \textbf{(b)}:] $A_k \eqm A_l^t$, and hence it is
      easy to see that we also have $A_g \eqm A_l^t$. Thus we have
      that $\tup{A_g, \scmap_g, \delta_2} \mimic
      \tup{A_l^t,\scmap_k}$. 
    \end{compactenum}
    Therefore by induction hypothesis, it is easy to see that the
    claim is proven by also considering the definition of program
    execution relation.

  \item[\textbf{[$\delta_g = \gif{\varphi}{\delta_1}{\delta_2}$]}.]
    Assume that there exist states $\tup{A'_k, m'_k}$ and
    $\tup{A_i^t, m_k}$ (for $i \in \set{1,\ldots, n}$, and $n \geq 0$)
    such that \\
    $\tup{A_k, m_k} \exect{\act_1\sigma_e} \tup{A_1^t, m_k}
    \exect{\act_2\sigma_e} \cdots $ \\ \hspace*{25mm}$\cdots \exect{\act_n\sigma_e} \tup{A_n^t,
      m_k} \exect{\act'\sigma} \tup{A'_k, m'_k}$
  where
    \begin{compactitem}
    \item $\sigma_e$ is an empty substitution,
    \item $\act_i \in \actsettmpa$ (for $i \in \set{1,\ldots, n}$), 
    \item $\act' \in \actsettmpd$, 
    \item $\carule{Q(\vec{p})}{\act'(\vec{p})} \in \procset'$, 
    \item if $n = 0$, then $\sigma \in \Ans(Q,T,A_k)$, otherwise
      $\sigma \in \Ans(Q,T,A_n^t)$,
    \item $\tmp \in A_i^t$ (for $i \in \set{1,\ldots, n}$),
      $\tmp \not\in A_k'$,
    \end{compactitem}
    By the definition of $\tgprog$ on the translation of a program of
    the form $\gif{\varphi}{\delta_1}{\delta_2}$ and
    $\gact{Q(\vec{p})}{\act(\vec{p})}$, there exists
    $j \in \set{1,\ldots,n}$ such that
    $\ppre(\gif{\varphi}{\delta_1}{\delta_2}) \in A_{j-1}^t$ and either
    \begin{compactenum}
    \item $\act_j = \gamma_{if}$, $\ppre(\delta_1) \in A_j^t$, and
      $\Ans(\varphi, T, A_{j-1}^t)$, or
    \item $\act_j = \gamma_{else}$, $\ppre(\delta_2) \in A_j^t$, and
      $\Ans(\varphi, T, A_{j-1}^t)$.
    \end{compactenum}
    where $\gamma_{if}$ and $\gamma_{else}$ are the actions obtained
    from the translation of $\gif{\varphi}{\delta_1}{\delta_2}$ by
    $\tgprog$ and it might be the case that $A_k = A_{j-1}^t$ (if
    $j = 1$). Now, by our assumption above and by
    Lemma~\ref{lem:temp-state-produced-by-temp-act}, we have that
    $A_k \eqm A_j^t$, and hence it is easy to see that we also have
    $A_g \eqm A_j^t$. Thus, essentially we have \\
      $\tup{A_{j-1}^t, m_k} \exect{\act_j\sigma_e} \tup{A_{j}^t, m_k}
      \exect{\act_{j+1}\sigma_e} \cdots$ \\ \hspace*{25mm} $\cdots \exect{\act_n\sigma_e}
      \tup{A_n^t, m_k} \exect{\act'\sigma} \tup{A'_k, m'_k}$,
  and
  \begin{compactenum}
  \item if $\act_j = \gamma_{if}$, then
    $\tup{A_g, \scmap_g, \delta_1} \mimic \tup{A_j^t,\scmap_k}$
    (because $A_g \eqm A_j^t$, $\scmap_g = \scmap_k$,
    $\ppre(\delta_1) \in A_j^t$), otherwise
  \item if $\act_j = \gamma_{else}$, then
    $\tup{A_g, \scmap_g, \delta_2} \mimic \tup{A_j^t,\scmap_k}$
    (because $A_g \eqm A_j^t$, $\scmap_g = \scmap_k$,
    $\ppre(\delta_2) \in A_j^t$).
  \end{compactenum}
  Therefore by induction hypothesis, 
  it is easy to see that the claim is proven by also considering
  the definition of program execution relation.

\item[\textbf{[$\delta_g = \gwhile{\varphi}{\delta_1}$]}.]  Assume
  that there exist states $\tup{A'_k, m'_k}$ and $\tup{A_i^t, m_k}$ (for
  $i \in \set{1,\ldots, n}$, and $n \geq 0$) such that \\
    $\tup{A_k, m_k} \exect{\act_1\sigma_e} \tup{A_1^t, m_k}
    \exect{\act_2\sigma_e} \cdots $ \\ \hspace*{25mm} $\cdots \exect{\act_n\sigma_e} \tup{A_n^t,
      m_k} \exect{\act'\sigma} \tup{A'_k, m'_k}$
  where
    \begin{compactitem}
    \item $\sigma_e$ is an empty substitution,
    \item $\act_i \in \actsettmpa$ (for $i \in \set{1,\ldots, n}$), 
    \item $\act' \in \actsettmpd$, 
    \item $\carule{Q(\vec{p})}{\act'(\vec{p})} \in \procset'$, 
    \item if $n = 0$, then $\sigma \in \Ans(Q,T,A_k)$, otherwise
      $\sigma \in \Ans(Q,T,A_n^t)$,
    \item $\tmp \in A_i^t$ (for $i \in \set{1,\ldots, n}$),
      $\tmp \not\in A_k'$,
    \end{compactitem}
    By the definition of $\tgprog$ on the translation of a program of
    the form $\gwhile{\varphi}{\delta_1}$ and
    $\gact{Q(\vec{p})}{\act(\vec{p})}$, there exists
    $j \in \set{1,\ldots,n}$ such that
    \begin{compactitem}
    \item $\act_j
      = \gamma_{doLoop}$
      ($\gamma_{doLoop}$
      is the action obtained during the translation of
      $\gwhile{\varphi}{\delta_1}$ by $\tgprog$),
    \item $\ppre(\gwhile{\varphi}{\delta_1})
      \in A_{j-1}^t$ (where $A_{j-1}^t = A_k$ when $j = 1$), and
    \item $\ppre(\delta_1) \in A_{j}^t$.
    \end{compactitem}
    Now, by our assumption above and by
    Lemma~\ref{lem:temp-state-produced-by-temp-act}, we have that $A_k
    \eqm A_j^t$, and hence it is easy to see that $A_g \eqm
    A_j^t$. Thus, essentially we have \\
    $\tup{A_{j-1}^t, m_k} \exect{\act_j\sigma_e} \tup{A_{j}^t, m_k}
    \exect{\act_{j+1}\sigma_e} \cdots $ \\ \hspace*{25mm}$\cdots \exect{\act_n\sigma_e}
    \tup{A_n^t, m_k} \exect{\act'\sigma} \tup{A'_k, m'_k}$,
  and $\tup{A_g, \scmap_g, \delta_1} \mimic \tup{A_j^t,\scmap_k}$
  (because $A_g \eqm A_j^t$, $\scmap_g = \scmap_k$,
  $\ppre(\delta_1) \in A_j^t$).
  Therefore by induction hypothesis, it is easy to see that the claim
  is proven by also considering the definition of program execution
  relation.
\end{enumerate}
\end{proof}

Now we will show that given a state $s_g$ of an S-GKAB transition
system and a state $s_k$ of its corresponding S-KAB transition system
such that $s_g$ is mimicked by $s_k$, then we have $s_g$ and $s_k$ are
J-bisimilar. Formally this claim is stated and shown below.

\begin{lemma}\label{lem:sgkab-to-kab-bisimilar-state}
  Let $\gkabsym$ be an S-GKAB with transition system
  $\ts{\gkabsym}^{\filter_S}$, and let $\tgkab(\gkabsym)$ be an S-KAB
  with transition system $\ts{\tgkab(\gkabsym)}^S$ obtained from
  $\gkabsym$ through $\tgkab$.  Consider
\begin{inparaenum}[]
\item a state $\tup{A_g,\scmap_g, \delta_g}$ of
  $\ts{\gkabsym}^{\filter_S}$ and
\item a state $\tup{A_k,\scmap_k}$ of $\ts{\tgkab(\gkabsym)}^S$.
\end{inparaenum}
If $\tup{A_g,\scmap_g, \delta_g} \mimic \tup{A_k,\scmap_k} $ then
$\tup{A_g,\scmap_g, \delta_g} \jbsim \tup{A_k,\scmap_k}$.
\end{lemma}
\begin{proof}
Let 
\begin{compactitem}
\item $\gkabsym = \tup{T, \initABox, \actset, \ginitprog}$ and \\
  $\ts{\gkabsym}^{\filter_S} = \tup{\const, T, \stateset_g, s_{0g},
    \abox_g, \trans_g}$, 
\item $\tgkab(\gkabsym) = \tup{T, \initABox', \actset', \procset'}$,
  and \\
  $\ts{\tgkab(\gkabsym)}^S = \tup{\const, T, \stateset_k, s_{0k},
    \abox_k, \trans_k}$
\end{compactitem}

\noindent
We have to show:
\begin{enumerate}[\bf (1)]
\item for every state
$\tup{A'_g,\scmap'_g, \delta'_g}$ such that
$\tup{A_g,\scmap_g, \delta_g} \trans_g \tup{A'_g,\scmap'_g,
  \delta'_g},$
there exist states $\tup{A'_k,\scmap'_k}$,
$\tup{A_1^t,\scmap_k} \ldots \tup{A_n^t,\scmap_k}$ (for $n \geq 0$)
with \\
$
\hspace*{5mm} \tup{A_k,\scmap_k} \trans_k \tup{A_1^t,\scmap_k} \trans_k \cdots \\ \hspace*{25mm} \cdots
\trans_k \tup{A_n^t,\scmap_k} \trans_k \tup{A'_k,\scmap'_k}
$\\
such that:
\begin{compactenum}
\item $\tmp \not\in A'_k$, $\tmp \in A^t_i$ for
  $i \in \set{1, \ldots, n}$, and
\item 
  $\tup{A'_g,\scmap'_g, \delta'_g} \mimic \tup{A'_k,\scmap'_k} $.
\end{compactenum}

\item for every state $\tup{A'_k,\scmap'_k}$ such that there exist
  states $\tup{A_1^t,\scmap_1} \ldots \tup{A_n^t,\scmap_n}$ (for
  $n \geq 0$)
  and \\
  $ \hspace*{5mm}\tup{A_k,\scmap_k} \trans_k \tup{A_1^t,\scmap_1}
  \trans_k \cdots $
  \\
  \hspace*{20mm}$\cdots \trans_k \tup{A_n^t,\scmap_n} \trans_k
  \tup{A'_k,\scmap'_k}
  $\\
  where $\tmp \not\in A'_k$, and $\tmp \in A^t_i$ for
  $i \in \set{1, \ldots, n}$,
  then there exists a state $\tup{A'_g,\scmap'_g, \delta'_g}$ with
  $\tup{A_g,\scmap_g, \delta_g} \trans_g \tup{A'_g,\scmap'_g,
    \delta'_g}$
  such that 
  $\tup{A'_g,\scmap'_g, \delta'_g} \mimic \tup{A'_k,\scmap'_k} $.

\end{enumerate}

\begin{enumerate}

\item[\textbf{Proof for (1):}] 
%
%
Assume
$ \tup{A_g,\scmap_g, \delta_g} \trans \tup{A'_g,\scmap'_g, \delta'_g},
$ then by the definition of GKABs transition system we have
\[
\tup{A_g,\scmap_g, \delta_g} \gprogtrans{\act\sigma, \filter_S}
\tup{A'_g,\scmap'_g, \delta'_g}.
\]
Additionally, 
it is easy to see that $A'_g$ is $T$-consistent.
By Lemma~\ref{lem:prog-exec-bsim}, there exist states $\tup{A_i^t, m_k}$,
and actions $\act_i$ (for $i \in \set{1, \ldots, n}$, where $n \geq 0$)
such that
\begin{compactitem}
\item
  $\tup{A_k, m_k} \exect{\act_1\sigma_e} \tup{A_1^t, m_k}
  \exect{\act_2\sigma_e} \cdots $ \\ \hspace*{15mm} $\cdots \exect{\act_n\sigma_e} \tup{A_n^t, m_k}
  \exect{\act'\sigma} \tup{A'_k, m'_k}$ \\where
  \begin{compactitem}
  \item $\sigma_e$ is an empty substitution, 
  \item $\act'$ is obtained from $\act$ through $\tgprog$,
  \item $\tmp \in A_i^t$ (for $i \in \set{1, \ldots, n}$), $\tmp \not\in A_k'$
  \end{compactitem}
\item 
  $\tup{A'_g,\scmap'_g, \delta'_g} \mimic \tup{A'_k,\scmap'_k} $,
\end{compactitem}
Additionally, since $A'_g$ is $T$-consistent and $A'_g \eqm A_k'$ then
$A'_k$ is $T$-consistent. 
As a consequence, we have that the claim is easily proven, since by
the definition of S-KABs standard transition systems, we have
\begin{center}
$
  \hspace*{-15mm}\tup{A_k, m_k} \trans_k \tup{A_1^t, m_k} \trans_k
  \cdots$ \\ \hspace*{25mm} $\cdots \trans_k \tup{A_n^t, m_k} \trans_k \tup{A'_k,
  m'_k}
$
\end{center}
where 
\begin{compactenum}
\item $\tmp \not\in A'_k$, and $\tmp \in A^t_i$ for
  $i \in \set{1, \ldots, n}$, and
\item 
  $\tup{A'_g,\scmap'_g, \delta'_g} \mimic \tup{A'_k,\scmap'_k} $,
\end{compactenum}

\item[\textbf{Proof for (2):}] 
%
  Assume \\
  $ \hspace*{5mm}\tup{A_k,\scmap_k} \trans_k \tup{A_1^t,\scmap_k}
  \trans_k \cdots \\ \hspace*{25mm} \cdots \trans_k
  \tup{A_n^t,\scmap_k} \trans_k \tup{A'_k,\scmap'_k}
  $\\
  where $n > 0$, $\tmp \not\in A'_k$, and $\tmp \in A^t_i$ for
  $i \in \set{1, \ldots, n}$. the definition of S-KABs transition
  systems, then we have \\ 
  \hspace*{5mm}$ \tup{A_k, m_k} \exect{\act_1\sigma_1} \tup{A_1^t, m_1}
  \exect{\act_2\sigma_2} \cdots $
  \\
  \hspace*{20mm}$\cdots \exect{\act_n\sigma_n} \tup{A_n^t, m_n}
  \exect{\act'\sigma} \tup{A'_k, m'_k}.  $ \\
  For some actions $\act'$, $\act_i$ (for $i \in \set{1,\ldots, n}$), and
  substitutions $\sigma'$, $\sigma_i$ (for $i \in \set{1,\ldots, n}$). 
  Let $\actsettmpa$ (resp.\ $\actsettmpd$) be the set of temp adder
  (resp.\ deleter) actions of $\tgkab(\gkabsym)$, since
  $\tmp \not\in A'_k$, and $\tmp \in A^t_i$ for
  $i \in \set{1, \ldots, n}$, by
  Lemmas~\ref{lem:temp-state-produced-by-temp-act} and
  \ref{lem:non-temp-state-produced-by-normal-action}, we have that
\begin{compactitem}
\item $\sigma_i$ is an empty substitution (for
  $i \in \set{1,\ldots, n}$).
\item $\act_i \in \actsettmpa$ (for $i \in \set{1,\ldots, n}$), and
\item $\act' \in \actsettmpd$, 
\item there exists an action invocation
  $\gact{Q(\vec{p})}{\act(\vec{p})}$ that is a sub-program of
  $\ginitprog$ such that $\act'$ is obtained from the translation of
  $\gact{Q(\vec{p})}{\act(\vec{p})}$ by $\tgprog$,
\item $\act_i$ (for $i \in \set{1,\ldots, n}$) does not involve any
  service calls, and hence $m_i = m_{i+1}$ (for $i \in \set{1,\ldots, n-1}$).
\end{compactitem}

Therefore, by Lemma~\ref{lem:kab-transition-bsim}, then there exists a
state $\tup{A'_g,\scmap'_g, \delta_g'}$ such that
\begin{compactitem}
\item
  $ \tup{A_g,\scmap_g, \delta_g} \gprogtrans{\alpha\sigma, \filter_S}
  \tup{A'_g,\scmap'_g, \delta_g'}, $
  \item $\act'$ is obtained from the translation of a certain action
    invocation $\gact{Q(\vec{p})}{\act(\vec{p})}$ via $\tgprog$.
  \item 
  $\tup{A'_g,\scmap'_g, \delta'_g} \mimic \tup{A'_k,\scmap'_k} $.
\end{compactitem}
Since
$\tup{A_g,\scmap_g, \delta_g} \gprogtrans{\act\sigma, \filter_S}
\tup{A'_g,\scmap'_g, \delta'_g}$,
by the definition of GKABs transition systems, we have that
$ \tup{A_g,\scmap_g, \delta_g} \trans \tup{A'_g,\scmap'_g, \delta'_g}
$.
Thus it is easy that the claim is proven since we also have that
$\tup{A'_g,\scmap'_g, \delta'_g} \mimic \tup{A'_k,\scmap'_k} $.
\end{enumerate}
\ \ 
\end{proof}
%

Having Lemma~\ref{lem:sgkab-to-kab-bisimilar-state} in hand, we can
easily show that given an S-GKAB, its transition system is J-bisimilar
to the transition system of its corresponding S-KAB that is obtained
via the translation $\tgkab$.

\begin{lemma}\label{lem:sgkab-to-kab-bisimilar-ts}
  Given an S-GKAB $\gkabsym$, we have
  $\ts{\gkabsym}^{\filter_S} \jbsim \ts{\tgkab(\gkabsym)}^{S} $
\end{lemma}
\begin{proof}
Let
\begin{compactenum}
\item $\gkabsym = \tup{T, \initABox, \actset, \ginitprog}$, and \\
  $\ts{\gkabsym}^{\filter_S} = \tup{\const, T, \stateset_g, s_{0g},
    \abox_g, \trans_g}$,
\item $\tgkab(\gkabsym) = \tup{T, \initABox', \actset', \procset'}$,
  and \\
  $\ts{\tgkab(\gkabsym)}^S = \tup{\const, T, \stateset_k,
    s_{0k}, \abox_k, \trans_k}$ 
\end{compactenum}
We have that $s_{0g} = \tup{A_0, \scmap_g, \delta_g}$ and
$s_{0k} = \tup{A'_0, \scmap_k}$ where
$\scmap_g = \scmap_k = \emptyset$.
Since $A_0' = A_0 \cup \set{\flagconcept{start}}$, and
$\flagconceptname$ is a special vocabulary outside the vocabulary of
$T$, hence $A'_0 \eqm A_0$. 
Now, by Lemma~\ref{lem:program-pre-post}, we have
$\ppre(\ginitprog) = \flagconcept{start}$ and
$\ppost(\ginitprog) = \flagconcept{end}$. 
Furthermore, since $\flagconcept{start} \in A_0'$, then we have
$s_{0g} \mimic s_{0k}$. Hence by
Lemma~\ref{lem:sgkab-to-kab-bisimilar-state}, we have
$s_{0g} \jbsim s_{0k}$. Therefore,
we have
$\ts{\gkabsym}^{\filter_S} \jbsim \ts{\tgkab(\gkabsym)}^S $.
\end{proof}

Having all of these machinery in hand, we are now ready to show that
the verification of \muladom properties over S-GKABs can be recast as
verification over S-KAB as follows.

\begin{theorem}\label{thm:sgkab-to-kab}
  Given an S-GKAB $\gkabsym$ and a closed $\muladom$ property $\Phi$
  in NNF,
\begin{center}
  $\ts{\gkabsym}^{\filter_S} \models \Phi$ if and only if
  $\ts{\tgkab(\gkabsym)}^S \models \tforj(\Phi)$
\end{center}
\end{theorem}
\begin{proof}
  By Lemma~\ref{lem:sgkab-to-kab-bisimilar-ts}, we have that
  $\ts{\gkabsym}^{\filter_S} \jbsim \ts{\tgkab(\gkabsym)}^S$.
  Hence, by
  Lemma~\ref{lem:jumping-bisimilar-ts-satisfies-same-formula}, we have
  that for every $\muladom$ property $\Phi$
\[
\ts{\gkabsym}^{\filter_S} \models \Phi \textrm{ if and only if }
\ts{\tgkab(\gkabsym)}^S \models \tforj(\Phi)
\]
\end{proof}

\subsubsection{Proof of Theorem~\ref{thm:gtos}.} The proof of Theorem
\ref{thm:gtos} is then essentially a consequence of Theorem
\ref{thm:sgkab-to-kab}.

\section{From B-GKABs to S-GKABs}\label{sec:correctnessBGKABtoSGKAB}

This section aims to show that the verification of \muladom properties
over B-GKABs can be recast as verification over S-GKABs.
%
%
%
Formally, given a B-GKAB $\gkabsym$ and a $\muladom$ formula $\Phi$,
we show that $\ts{\gkabsym}^{\filter_B} \models \Phi$ if and only if
$\ts{\tgkabb(\gkabsym)}^{\filter_S} \models \tforb(\Phi)$ (This claim
is formally stated and proven in Theorem~\ref{thm:bgkab-to-sgkab}).
To this aim, our approach is as follows: we first introduce a notion
of Leaping bisimulation (L-Bisimulation) and show that two
L-bisimilar transition systems can not be distinguished by any
\muladom properties modulo the translation $\tforb$. 
%
Then, we show that the b-repair program is always terminate and
produces the same result as the result of b-repair over a knowledge
base.
Using those results, we show that given a B-GKAB, its transition
system is L-bisimilar to the transition of its corresponding S-GKAB
that is obtained through the translation $\tgkabb$. As a consequence,
using the property of L-bisimulation, we have that they can not be
distinguished by any \muladom properties modulo the translation
$\tforb$.

\subsection{Leaping Bisimulation (L-Bisimulation)}\label{sec:leaping-bisimulation}

We define the notion of \emph{leaping bisimulation} as follows.
\begin{definition}[Leaping Bisimulation (L-Bisimulation)]
  Let $\ts{1} = \tup{\const, T, \Sigma_1, s_{01}, \abox_1, \trans_1}$
  and $\ts{2} = \tup{\const, T, \Sigma_2, s_{02}, \abox_2, \trans_2}$
  be transition systems, with
  $\adom{\abox_1(s_{01})} \subseteq \const$ and
  $\adom{\abox_2(s_{02})} \subseteq \const$.  A \emph{leaping
    bisimulation} (L-Bisimulation) between $\ts{1}$ and $\ts{2}$ is a
  relation $\B \subseteq \Sigma_1 \times\Sigma_2$ that
  $\tup{s_1, s_2} \in \B$ implies that:
  \begin{compactenum}
  \item $\abox_1(s_1) = \abox_2(s_2)$
  \item for each $s_1'$, if $s_1 \Rightarrow_1 s_1''$ then there exist
    $s_2'$, $s_2''$, $t_1, \ldots ,t_n$ (for $n \geq 0$) with
    \[
    s_2 \Rightarrow_2 s_2' \Rightarrow_2 t_1 \Rightarrow_2 \ldots
    \Rightarrow_2 t_n \Rightarrow_2 s_2''
    \] 
   such that $\tup{s_1'', s_2''}\in\B$,
    $\temp \not\in \abox_2(s_2'')$ and $\temp \in \abox_2(t_i)$ for
    $i \in \set{1, \ldots, n}$.
  \item for each $s_2''$, if 
    \[
    s_2 \Rightarrow_2 s_2'\Rightarrow_2 t_1 \Rightarrow_2 \ldots
    \Rightarrow_2 t_n \Rightarrow_2 s_2''
    \] 
    (for $n \geq 0$) with $\temp \in \abox_2(t_i)$ for
    $i \in \set{1, \ldots, n}$ and $\temp \not\in \abox_2(s_2'')$, then
    there exists $s_1''$ with $s_1 \Rightarrow_1 s_1''$, such that
    $\tup{s_1'', s_2''}\in\B$.
 \end{compactenum}
\ \ 
\end{definition}

\noindent
Let $\ts{1} = \tup{\const, T, \Sigma_1, s_{01}, \abox_1, \trans_1}$
and $\ts{2} = \tup{\const, T, \Sigma_2, s_{02}, \abox_2, \trans_2}$ be
transition systems,
a state $s_1 \in \Sigma_1$ is \emph{L-bisimilar} to
$s_2 \in \Sigma_2$, written $s_1 \lbsim s_2$, if there exists an
L-bisimulation $\B$ between $\ts{1}$ and $\ts{2}$ such that
$\tup{s_1, s_2}\in\B$.
A transition system $\ts{1}$ is \emph{L-bisimilar} to $\ts{2}$,
written $\ts{1} \lbsim \ts{2}$, if there exists an L-bisimulation $\B$
between $\ts{1}$ and $\ts{2}$ such that $\tup{s_{01}, s_{02}}\in\B$.

Now, we advance further to show that two transition systems which are
L-bisimilar can not be distinguished by any \muladom formula (in NNF)
modulo the translation $\tforb$ which is defined in detail as follows:

\begin{definition}[Translation $\tforb$]\label{def:tforb}
  We define a \emph{translation $\tforb$} that transforms an arbitrary
  \muladom formula $\Phi$ (in NNF) into another \muladom formula
  $\Phi'$ inductively by recurring over the structure of $\Phi$ as
  follows:
\[
\small
\begin{array}{@{}l@{}ll@{}}
  \bullet\ \tforb(Q) &=& Q \\

  \bullet\ \tforb(\neg Q) &=& \neg Q \\

  \bullet\ \tforb(\Q x.\Phi) &=& \Q x. \tforb(\Phi) \\

  \bullet\ \tforb(\Phi_1 \circ \Phi_2) &=& \tforb(\Phi_1) \circ \tforb(\Phi_2) \\

  \bullet\ \tforb(\circledcirc Z.\Phi) &=& \circledcirc Z. \tforb(\Phi) \\

  \bullet\ \tforb(\DIAM{\Phi}) &=& \\
                     &&\hspace*{-20mm}\DIAM{\DIAM{\mu Z.((\temp \wedge \DIAM{Z})
                        \vee (\neg \temp \wedge \tforb(\Phi)))}} \\

  \bullet\ \tforb(\BOX{\Phi}) &=& \\
                     &&\hspace*{-20mm}\BOX{\BOX{\mu Z.((\temp \wedge \BOX{Z} \wedge
                        \DIAM{\top}) \vee 
                        (\neg \temp \wedge \tforb(\Phi)))}}
\end{array}
\]
\noindent
where:
\begin{compactitem}
\item $\circ$ is a binary operator ($\vee, \wedge, \ra,$ or $\lra$),
\item $\circledcirc$ is least ($\mu$) or greatest ($\nu$) fix-point operator,
\item $\Q$ is forall ($\forall$) or existential ($\exists$)
  quantifier.
\end{compactitem}
\ \ 
\end{definition}

\begin{lemma}\label{lem:leaping-bisimilar-states-satisfies-same-formula}
  Consider two transition systems
  $\ts{1} = \tup{\const, T,\stateset_1,s_{01},\abox_1,\trans_1}$ and
  $\ts{2} = \tup{\const, T,\stateset_2,s_{02},\abox_2,\trans_2}$, with
  $\adom{\abox_1(s_{01})} \subseteq \const$ and
  $\adom{\abox_2(s_{02})} \subseteq \const$.  Consider two states
  $s_1 \in \stateset_1$ and $s_2 \in \stateset_2$ such that
  $s_1 \lbsim s_2$. Then for every formula $\Phi$ of $\muladom$ (in
  negation normal form), and every valuations $\vfo_1$ and $\vfo_2$
  that assign to each of its free variables a constant
  $c_1 \in \adom{\abox_1(s_1)}$ and $c_2 \in \adom{\abox_2(s_2)}$,
  such that $c_1 = c_2$, we have that
  \[
  \ts{1},s_1 \models \Phi \vfo_1 \textrm{ if and only if } \ts{2},s_2
  \models \tforb(\Phi) \vfo_2.
  \]
\end{lemma}
\begin{proof}
  The proof is then organized in three parts:
\begin{compactenum}[(1)]
\item We prove the claim for formulae of $\ladom$, obtained from
  $\muladom$ by dropping the predicate variables and the fixpoint
  constructs. $\ladom$ corresponds to a first-order variant of the
  Hennessy Milner logic, and its semantics does not depend on the
  second-order valuation.
\item We extend the results to the infinitary logic obtained by extending
  $\ladom$ with arbitrary countable disjunction.
\item We recall that fixpoints can be translated into this infinitary logic,
  thus proving that the theorem holds for $\muladom$.
\end{compactenum}

\smallskip
\noindent
\textbf{Proof for $\ladom$.}  We proceed by induction on the structure
of $\Phi$, without considering the case of predicate variable and of
fixpoint constructs, which are not part of $\ladom$.

\smallskip
\noindent
\textit{Base case:}
\begin{compactitem}
\item[\textbf{($\Phi = Q$)}.] Since $s_1 \lbsim s_2$, we have
  $\abox_1(s_1) = \abox_2(s_2)$. Hence, we have
  $\Ans(Q, T, \abox_1(s_1)) = \Ans(Q, T, \abox_2(s_2))$.
  Hence, since $\tforb(Q) = Q$, for every valuations $\vfo_1$ and
  $\vfo_2$ that assign to each of its free variables a constant
  $c_1 \in \adom{\abox_1(s_1)}$ and $c_2 \in \adom{\abox_2(s_2)}$,
  such that $c_1 = c_2$, we have
  \[
  \ts{1},s_1 \models Q \vfo_1 \textrm{ if and only if } \ts{2},s_2
  \models \tforb(Q) \vfo_2.
  \]

\item[\textbf{($\Phi = \neg Q$)}.] Similar to the previous case.

\end{compactitem}

\smallskip
\noindent
\textit{Inductive step:}
\begin{compactitem}
\item[\textbf{($\Phi = \Psi_1 \wedge \Psi_2$)}.]  
  $\ts{1},s_1 \models (\Psi_1\wedge \Psi_2) \vfo_1$ if and only if
  either $\ts{1},s_1 \models \Psi_1 \vfo_1$ or
  $\ts{1},s_1 \models \Psi_2 \vfo_1$.  By induction hypothesis, we
  have for every valuations $\vfo_1$ and $\vfo_2$ that assign to each
  of its free variables a constant $c_1 \in \adom{\abox_1(s_1)}$ and
  $c_2 \in \adom{\abox_2(s_2)}$, such that $c_1 = c_2$, we have
\begin{compactitem}
\item
  $ \ts{1},s_1 \models \Psi_1 \vfo_1 \textrm{ if and only if }
  \ts{2},s_2 \models \tforb(\Psi_1) \vfo_2$, and also
\item
  $ \ts{1},s_1 \models \Psi_2 \vfo_1 \textrm{ if and only if }
  \ts{2},s_2 \models \tforb(\Psi_2) \vfo_2.  $
\end{compactitem}

Hence, $\ts{1},s_1 \models \Psi_1 \vfo_1$ and
$\ts{1},s_1 \models \Psi_2 \vfo_1$ if and only if
$\ts{2},s_2 \models \tforb(\Psi_1) \vfo_2$ and
$\ts{2},s_2 \models \tforb(\Psi_2) \vfo_2$. Therefore we have
$ \ts{1},s_1 \models (\Psi_1 \wedge \Psi_2) \vfo_1 \textrm{ if and
  only if } \ts{2},s_2 \models (\tforb(\Psi_1) \wedge \tforb(\Psi_2))
\vfo_2 $. 
Since
$\tforb(\Psi_1 \wedge \Psi_2) = \tforb(\Psi_1) \wedge \tforb(\Psi_2)$,
we have
\[
\ts{1},s_1 \models (\Psi_1 \wedge \Psi_2) \vfo_1 \textrm{ iff } \ts{2},s_2 \models \tforb(\Psi_1\wedge \Psi_2) \vfo_2
  \]
  The proof for the case of $\Phi = \Psi_1 \vee \Psi_2$,
  $\Phi = \Psi_1 \ra \Psi_2$, and $\Phi = \Psi_1 \lra \Psi_2$ can be
  done similarly.


\item[\textbf{($\Phi = \DIAM{\Psi}$)}.]  Assume
  $\ts{1},s_1 \models (\DIAM{\Psi}) \vfo_1$, where $\vfo_1$ is a
  valuation that assigns to each free variable of $\Psi$ a constant
  $c_1 \in \adom{\abox_1(s_1)}$. Then there exists $s_1''$ such that
  $s_1 \trans_1 s_1''$ and $\ts{1},s_1'' \models \Psi \vfo_1$.  Since
  $s_1 \lbsim s_2$, there exists $s_2'$, $s_2''$, $t_1, \ldots ,t_n$
  (for $n \geq 0$) with
    \[
    s_2 \Rightarrow_2 s_2' \Rightarrow_2 t_1 \Rightarrow_2 \cdots \Rightarrow_2 t_n
    \Rightarrow_2 s_2''
    \] 
    such that $s_1'' \lbsim s_2''$, $\temp \in \abox_2(t_i)$ for
    $i \in \set{1, \ldots, n}$, and $\temp \not\in \abox_2(s_2'')$.
    Hence, by induction hypothesis, for every valuations $\vfo_2$ that
    assign to each free variables $x$ of $\tforb(\Psi)$ a constant
    $c_2 \in \adom{\abox_2(s_2)}$, such that $c_1 = c_2$ and
    $x/c_1 \in \vfo_1$, we have
    $ \ts{2},s_2'' \models \tforb(\Psi_1) \vfo_2.  $
%
%
%
%
%
    Considering that
    \[
    s_2\Rightarrow_2 s_2'\Rightarrow_2 t_1 \Rightarrow_2 \cdots
    \Rightarrow_2 t_n \Rightarrow_2 s_2''
    \] 
    (for $n \geq 0$), $\temp \in \abox_2(t_i)$ for
    $i \in \set{1, \ldots, n}$, and $\temp \not\in \abox_2(s_2'')$. We
    therefore get
    \[
    \begin{array}{@{}l@{}l@{}}
      \ts{2},s_2 \models (\DIAM{\DIAM{\mu Z.((&\temp \wedge \DIAM{Z})
      \vee \\
      &(\neg\temp \wedge \tforb(\Psi)))}})\vfo_2.
    \end{array}
    \]
    Since
    \[
    \begin{array}{l@{}l}
      \tforb(\DIAM{\Phi}) = \DIAM{\DIAM{\mu Z.((&\temp \wedge \DIAM{Z}) \vee \\
      &(\neg \temp \wedge \tforb(\Phi)))}},
    \end{array}
    \]
    thus we have
    \[
    \ts{2},s_2 \models \tforb(\DIAM{\Phi})\vfo_2.
    \]

The other direction can be shown in a symmetric way.

\item[\textbf{($\Phi = \BOX{\Psi}$)}.]  The proof is similar to the
  case of $\Phi = \DIAM{\Psi}$

\item[\textbf{($\Phi = \exists x. \Psi$)}.]  Assume that
  $\ts{1},s_1 \models (\exists x. \Psi)\vfo'_1$, where $\vfo'_1$ is a
  valuation that assigns to each free variable of $\Psi$ a constant
  $c_1 \in \adom{\abox_1(s_1)}$. Then, by definition, there exists
  $c \in \adom{\abox_1(s_1)}$ such that
  $\ts{1},s_1 \models \Psi\vfo_1$, where $\vfo_1 = \vfo'_1[x/c]$. By
  induction hypothesis, for every valuation $\vfo_2$ that assigns to
  each free variable $y$ of $\tforb(\Psi)$ a constant
  $c_2 \in \adom{\abox_2(s_2)}$, such that $c_1 = c_2$ and
  $y/c_1 \in \vfo_1$, we have that
  $\ts{2},s_2 \models \tforb(\Psi) \vfo_2$. Additionally, we have
  $\vfo_2 = \vfo'_2[x/c']$, where $c' \in \adom{\abox_2(s_2)}$, and
  $c' = c$ because $\abox_2(s_2) = \abox_1(s_1)$.  Hence, we get
  $\ts{2},s_2 \models (\exists x. \tforb(\Psi))\vfo'_2$. Since
  $\tforb(\exists x.\Phi) = \exists x. \tforb(\Phi)$, thus we have
  $\ts{2},s_2 \models \tforb(\exists x. \Psi)\vfo'_2$

The other direction can be shown similarly.

\item[\textbf{($\Phi = \forall x. \Psi$)}.]  The proof is similar to
  the case of $\Phi = \exists x. \Psi$.


\end{compactitem}

\smallskip
\noindent
\textbf{Extension to arbitrary countable disjunction.}  Let $\Psi$ be
a countable set of $\ladom$ formulae. Given a transition system
$\ts{} = \tup{\const, T,\stateset,s_{0},\abox,\trans}$, the semantics
of $\bigvee \Psi$ is
$(\bigvee \Psi) _\vfo^{\ts{}} = \bigcup_{\psi \in \Psi}
(\psi)_\vfo^{\ts{}}$.
Therefore, given a state $s \in \Sigma$ we have
$\ts{}, s \models (\bigvee \Psi)\vfo$ if and only if there exists
$\psi \in \Psi$ such that $\ts{}, s \models \psi\vfo$. Arbitrary
countable conjunction can be obtained similarly.

Now, let $\ts{1} = \tup{\const, T,\stateset_1,s_{01},\abox_1,\trans_1}$
and $\ts{2} = \tup{\const, T,\stateset_2,s_{02},\abox_2,\trans_2}$.
Consider two states $s_1 \in \stateset_1$ and $s_2 \in
\stateset_2$ such that $s_1 \lbsim s_2$.
By induction hypothesis, we have for every valuations $\vfo_1$ and
$\vfo_2$ that assign to each of its free variables a constant
$c_1 \in \adom{\abox_1(s_1)}$ and $c_2 \in \adom{\abox_2(s_2)}$, such
that $c_2 = c_1$, we have that for every formula $\psi \in \Psi$, it
holds $\ts{1}, s_1 \models \psi \vfo_1$ if and only if
$\ts{2}, s_2 \models \tforb(\psi)\vfo_2$.
Given the semantics of $\bigvee \Psi$ above, this implies that
$\ts{1}, s \models (\bigvee \Psi) \vfo_1$ if and only if
$\ts{2}, s \models (\bigvee \tforb(\Psi)) \vfo_2$, where
$\tforb(\Psi) = \{\tforb(\psi) \mid \psi \in \Psi\}$. The proof is
then obtained by observing that
$\bigvee \tforb(\Psi) = \tforb(\bigvee \Psi)$.

\smallskip
\noindent
\textbf{Extension to full $\muladom$.}  In order to extend the result
to the whole \muladom, we resort to the well-known result stating that
fixpoints of the $\mu$-calculus can be translated into the infinitary
Hennessy Milner logic by iterating over \emph{approximants}, where the
approximant of index $\alpha$ is denoted by $\mu^\alpha Z.\Phi$
(resp.~$\nu^\alpha Z.\Phi$). This is a standard result that also holds
for \muladom. In particular, approximants are built as follows:
\[
\begin{array}{rl rl}
  \mu^0 Z.\Phi & = \false
  &  \nu^0 Z.\Phi & = \true\\
  \mu^{\beta+1} Z.\Phi & = \Phi[Z/\mu^\beta Z.\Phi]
  & \nu^{\beta+1} Z.\Phi & = \Phi[Z/\nu^\beta Z.\Phi]\\
  \mu^\lambda Z.\Phi & = \bigvee_{\beta < \lambda} \mu^\beta Z. \Phi &
  \nu^\lambda Z.\Phi & = \bigwedge_{\beta < \lambda} \nu^\beta Z. \Phi
\end{array}
\]
where $\lambda$ is a limit ordinal, and where fixpoints and their
approximants are connected by the following properties: given a
transition system $\ts{}$ and a state $s$ of $\ts{}$
\begin{compactitem}
\item $s \in \MODA{\mu Z.\Phi}$ if and only if there exists an ordinal
  $\alpha$ such that $s \in \MODA{\mu^\alpha Z.\Phi}$ and, for every
  $\beta < \alpha$, it holds that $s \notin \MODA{\mu^\beta Z.\Phi}$;
\item $s \notin \MODA{\nu Z.\Phi}$ if and only if there exists an
  ordinal $\alpha$ such that $s \notin \MODA{\nu^\alpha Z.\Phi}$ and,
  for every $\beta < \alpha$, it holds that $s \in \MODA{\nu^\beta
    Z.\Phi}$.
\end{compactitem}

\end{proof}

As a consequence, from Lemma
\ref{lem:leaping-bisimilar-states-satisfies-same-formula} above, we
can easily obtain the following lemma saying that two transition
systems which are L-bisimilar can not be distinguished by any \muladom
formula (in NNF) modulo a translation $\tforb$.

\begin{lemma}\label{lem:leaping-bisimilar-ts-satisfies-same-formula}
  Consider two transition systems
  $\ts{1} = \tup{\const,T,\stateset_1,s_{01},\abox_1,\trans_1}$ and
  $\ts{2} = \tup{\const,T,\stateset_2,s_{02},\abox_2,\trans_2}$ such
  that $\ts{1} \lbsim \ts{2}$.  For every closed $\muladom$ formula
  $\Phi$ in NNF, we have:
  \[
  \ts{1} \models \Phi \textrm{ if and only if } \ts{2} \models
  \tforb(\Phi).
  \]
\end{lemma}
\begin{proof}
  Since by the definition we have $s_{01} \lbsim s_{02}$, we obtain
  the proof as a consequence of
  Lemma~\ref{lem:leaping-bisimilar-states-satisfies-same-formula} due
  to the fact that
  \[
  \ts{1}, s_{01} \models \Phi \textrm{ if and only if } \ts{2}, s_{02}
  \models \tforb(\Phi)
  \]
\ \ 
\end{proof}

\subsection{Termination and Correctness of B-repair Program}


We now proceed to show that the b-repair program is always terminate
and produces the same result as the result of b-repair over a
knowledge base. To this aim, we first need to introduce some
preliminaries.
Below, we prove that every execution steps of a b-repair program always
reduces the number of ABox assertions that participate in the
inconsistency. Formally, it is stated below:

\begin{lemma}\label{lem:brepair-action-decrease}
  Given a TBox $T$, a $T$-inconsistent ABox $A$, a service call map
  $\scmap$, and a set $\actset_b^T$ of b-repair action over $T$. Let
  $\act \in \actset_b^T$ be an arbitrary b-repair action, and $\sigma$
  be a legal parameter assignment for $\act$. If
  $(\tup{A,\scmap}, \act\sigma, \tup{A', \scmap'}) \in
  \tell_{\filter_S}$, then $\card{\inc(A)} > \card{\inc(A')}$.
\end{lemma}
\begin{proof}
%
  We proof the claim by reasoning over all cases of b-repair actions
  as follows:
\begin{compactitem}
\item[\textbf{Case 1:}] \textit{The actions obtained from functionality
    assertion $\funct{R} \in T_f$.} \\
  Let $\act_{F}$ be such action and has the following form:
  \[
  \act_{F}(x,y):\set{\map{R(x,z) \wedge \neg [z = y]}{\del \set{R(x,z)}}}.
  \]
  Suppose, $\act_{F}$ is executable in $A$ with legal parameter
  assignment $\sigma$.  Since we have
  \[
  \gact{\exists z.\qunsatf(\funct{R}, x, y, z)}{\act_{F}(x,y)} \in
  \setinvocation_b^T,
  \]
  then there exists $c \in \adom{A}$ and
  $\set{c_1, c_2, c_3, \ldots c_n} \subseteq \adom{A}$ such that
  $\set{R(c,c_1), R(c,c_2), \ldots, R(c,c_n)} \subseteq A$ where
  $n \geq 2$.  W.l.o.g.\ let $\sigma$ substitute $x$ to $c$, and $y$
  to $c_1$, then we have
  $(\tup{A,\scmap}, \act\sigma, \tup{A', \scmap}) \in
  \tell_{\filter_S}$,
  where $A' = A \setminus \set{R(c,c_2), \ldots, R(c,c_n)}$. Therefore
  we have $\card{\inc(A)} > \card{\inc(A')}$.

\item[\textbf{Case 2:}] \textit{The actions obtained from negative
    concept $B_1 \ISA \neg B_2$ such that
    $T \models B_1 \ISA \neg B_2$.}  Let $\act_{B_1}$ be such action
  and has the following form:
  \[
  \act_{B_1}(x):\set{\map{\true}{\del \set{B_1(x)}}}.
  \]
  Suppose, $\act_{B_1}$ is executable in $A$ with legal parameter
  $\sigma$.  Since we have
  \[
  \gact{\qunsatn(B_1 \ISA \neg B_2, x)}{\act_{B_1}(x)} \in
  \setinvocation_b^T,
  \]
  then there exists $c \in \adom{A}$ such that
  $\set{B_1(c), B_2(c)} \subseteq A$.  W.l.o.g.\ let $\sigma$
  substitute $x$ to $c$, then we have
  $(\tup{A,\scmap}, \act\sigma, \tup{A', \scmap}) \in
  \tell_{\filter_S}$,
  where $A' = A \setminus \set{B_1(c}$. Therefore we have
  $\card{\inc(A)} > \card{\inc(A')}$.

\item[\textbf{Case 3:}] \textit{The actions obtained from negative
    role
    inclusion $R_1 \ISA \neg R_2$ s.t.\ $T \models R_1 \ISA \neg
    R_2$.} 
  The proof is similar to the case 2.
\end{compactitem}
\end{proof}

Having Lemma~\ref{lem:brepair-action-decrease} in hand, we are ready
to show the termination of b-repair program as follows:

\begin{lemma}\label{lem:bprog-termination}
  Given a TBox $T$, and a filter $\filter_S$. A b-repair program
  $\delta^T_{b}$ over $T$ is always terminate. I.e., given a state
  $\tup{A, \scmap, \delta^T_{b}}$, every program execution trace
  induced by $\delta^T_{b}$ on $\tup{A, \scmap, \delta^T_{b}}$ w.r.t.\
  filter $\filter_S$ is terminating.
\end{lemma}
\begin{proof}
We divide the proof into two cases:

\smallskip
\noindent
\textbf{Case 1: $A$ is $T$-consistent. }\\
Trivially true, since $\ask(\qunsatecq{T}, T, A) = \false$, we have
$\final{\tup{A, \scmap, \delta^T_b}}$, by the definition.

\smallskip
\noindent
\textbf{Case 2: $A$ is $T$-inconsistent. }\\
Given a state $\tup{A, \scmap, \delta^T_{b}}$ such that $A$ is
$T$-inconsistent, w.l.o.g.\ let
\[
\pi = \tup{A, \scmap, \delta^T_{b}} \gexectrans \tup{A_1, \scmap,
  \delta_1} \gexectrans \tup{A_2, \scmap, \delta_2} \gexectrans \cdots
\]
be an arbitrary program execution trace induced by $\delta^T_{b}$ on
$\tup{A, \scmap, \delta^T_{b}}$ w.r.t.\ filter $\filter_S$. Notice
that the service call map $\scmap$ always stay the same since every
b-repair action $\act \in \actset^T_b$ (which is the only action that
might appears in $\delta^T_b$) does not involve any service
calls. Now, we have to show that eventually there exists a state
$\tup{A_n, \scmap, \delta_n}$, such that
\[
\pi = \tup{A, \scmap, \delta^T_{b}} \gexectrans \tup{A_1, \scmap,
  \delta_1} \gexectrans \cdots \gexectrans \tup{A_n, \scmap, \delta_n}
\]
and $\final{\tup{A_n, \scmap, \delta_n}}$. 
By Lemma~\ref{lem:brepair-action-decrease}, we have that
\[
\card{\inc(A)} > \card{\inc(A_1)} > \card{\inc(A_2)} > \cdots
\]
Additionally, due to the following facts:
\begin{compactenum}[(1)]
\item Since we assume that every concepts (resp.\ roles) are
  satisfiable, inconsistency can only be caused by
  \begin{compactenum}
  \item pair of assertions $B_1(c)$ and $B_2(c)$ (resp.\
    $R_1(c_1,c_2)$ and $R_2(c_1,c_2)$) that violate a negative
    inclusion assertion $B_1 \sqsubseteq \neg B_2$ (resp.\
    $R_1 \sqsubseteq \neg R_2$) such that \
    $T \models B_1 \sqsubseteq \neg B_2$ (resp.\
    $T \models R_1 \sqsubseteq \neg R_2$), or
  \item $n$-number role assertions
    \[
    R(c, c_1), R(c, c_2), \ldots, R(c, c_n)
    \]
    that violate a functionality assertion $\funct{R} \in T$.
  \end{compactenum}
\item To deal with both source of inconsistency in the point (1):
  \begin{compactenum}
  \item we consider all negative concept inclusions $B_1
    \sqsubseteq \neg B_2$
    such that $T \models B_1 \sqsubseteq \neg B_2$
    when constructing the b-repair actions $\actset^T_b$ (i.e., we
    saturate the negative inclusion assertions w.r.t.\ $T$ obtaining
    all derivable negative inclusion assertions from $T$). Moreover,
    for each negative concept inclusion $B_1 \sqsubseteq \neg B_2$
    such that $T \models B_1 \sqsubseteq \neg B_2$, we have an action
    which remove $B_1(c)$ (for a constant $c$) in case
    $B_1 \sqsubseteq \neg B_2$ is violated. Similarly for negative
    role inclusions. 
    
%
  \item we consider all functionality assertions $\funct{R} \in T$
    when constructing the b-repair actions $\actset^T_b$, and each
    $\act_F \in \actset^T_b$ removes all role assertions that violates
    $\funct{R}$, except one.
  \end{compactenum}
\item Observe that $\ask(\qunsatecq{T}, T, A_n) = \true$ as long as
  $\card{\inc(A)} > 0$ (for any ABox $A$). Moreover, in such
  situation, by construction of $\setinvocation_b^T$, there always
  exists an executable action $\act \in \actset^T_b$ (Observe that
  $\qunsatecq{T}$ is a disjunction of every ECQ $Q$ that guard every
  corresponding atomic action invocation
  $\gact{Q(\vec{p})}{\act(\vec{p})} \in \setinvocation_b^T$ of each
  $\act \in \actset^T_b$ where each of its free variables are
  existentially quantified).
\end{compactenum}
As a consequence, eventually there exists $A_n$ such that
$\card{\inc(A_n)} = 0$. Hence by Lemma~\ref{lem:card-inc-abox} $A_n$
is $T$-consistent. Therefore $\ask(\qunsatecq{T}, T, A_n) = \false$,
and $\final{\tup{A_n, \scmap, \delta_n}}$.
\end{proof}

We now proceed to show the correctness of the b-repair program. I.e.,
showing that a b-repair program produces exactly the result of a
b-repair operation over the given (inconsistent) KB. As the first
step, we will show that every ABoxes produced by the b-repair program
is a maximal $T$-consistent subset of the given input ABox as
follows. Below we show that a b-repair program produces a maximal
$T$-consistent subset of the given ABox.

\begin{lemma}\label{lem:characteristic-bprog-result}
  Given a TBox $T$, an ABox $A$, a service call map $\scmap$, a
  b-repair program $\delta^T_{b}$ over $T$ and a filter $\filter_S$,
  we have that if $A' \in \progres(A, \scmap, \delta^T_b)$ then $A'$
  is a maximal $T$-consistent subset of $A$.
\end{lemma}
\begin{proof} 
  Assume that $A' \in \progres(A, \scmap, \delta^T_b)$. We have to
  show that
\begin{compactenum}[(1)]
\item $A' \subseteq A$
\item $A'$ is $T$-consistent
\item There does not exists $A''$ such that $A' \subset A'' \subseteq
  A$ and $A''$ is $T$-consistent.
\end{compactenum}
We divide the proof into two cases:
\begin{compactenum}
\item[\textbf{Case 1: $A$ is $T$-consistent.}]~Trivially true, because
  $\ask(\qunsatecq{T}, T, A) = \false$, hence $\final{\tup{A, \scmap,
    \delta^T_{b}}}$ and $A \in \progres(A, \scmap,
  \delta^T_b)$. Thus, $A$ trivially satisfies the condition (1) - (3).

\item[\textbf{Case 2: $A$ is $T$-inconsistent.}]~Let
\[
\pi = \tup{A, \scmap, \delta^T_{b}} \gexectrans \tup{A_1, \scmap,
  \delta_1} \gexectrans \cdots \gexectrans \tup{A', \scmap, \delta'}
\]
be the corresponding program execution trace that produces $A'$ (This
trace should exists because $A' \in \progres(A, \scmap, \delta^T_b)$).

\smallskip
\noindent
\textbf{For condition (1).}~
Trivially true from the construction of b-repair program $\delta^T_b$.
Since, each step of the program always and only removes some ABox
assertions and also by recalling Lemma~\ref{lem:brepair-action-decrease}
that we have
\[
\card{\inc(A)} > \card{\inc(A_1)} > \card{\inc(A_2)} > \cdots
\]

\smallskip
\noindent
\textbf{For condition (2).}~Since the b-repair program $\delta^T_b$ is
terminated at a final state $\tup{A', \scmap, \delta'}$ where
$\ask(\qunsatecq{T}, T, A') = \false$, hence $A'$ is $T$-consistent.

\smallskip
\noindent
\textbf{For condition (3).}~Suppose by contradiction that there exists $A''$
s.t. $A' \subset A'' \subseteq A$ and $A''$ is $T$-consistent. Recall
that in \dllitea, since we assume that every concepts (resp.\ roles)
are satisfiable, inconsistency is only caused by
\begin{compactenum}[\it (i)]
\item pair of assertions $B_1(c)$ and $B_2(c)$ (resp.\ $R_1(c_1,c_2)$
  and $R_2(c_1,c_2)$) that violate a negative inclusion assertion $B_1
  \sqsubseteq \neg B_2$ (resp.\ $R_1 \sqsubseteq \neg R_2$) s.t.\ $T
  \models B_1 \sqsubseteq \neg B_2$ (resp.\ $T \models R_1 \sqsubseteq
  \neg R_2$), or
\item $n$-number role assertions 
  \[
  R(c, c_1), R(c, c_2), \ldots, R(c, c_n)
  \]
  that violate a functionality assertion $\funct{R} \in T$.
\end{compactenum}
However, by the construction of b-repair program $\delta^T_b$,
we have that each action $\act \in \actset^T_b$ is executable when
there is a corresponding inconsistency (detected by each guard $Q$ of
each corresponding atomic action invocation
$\gact{Q(\vec{p})}{\act(\vec{p})} \in \setinvocation_b^T$) and each
action only either
\begin{compactenum}[\it (i)]
\item removes one of the pair of assertions that violate a negative
  inclusion assertion, or
\item removes $n-1$ role assertions among $n$ role assertions that
  violate a functionality assertion.
\end{compactenum}
Hence, if $A''$ exists, then there exists an ABox assertion that
should not be removed, but then we will have $A''$ is
$T$-inconsistent. Thus, we have a contradiction. Hence, there does not
exists $A''$ such that $A' \subset A'' \subseteq A$ and $A''$ is
$T$-consistent.
\end{compactenum}
\ \ 
\end{proof}

From Lemma~\ref{lem:characteristic-bprog-result}, we can show that
every ABox that is produced by b-repair program is in the set of
b-repair of the given (inconsistent) KB. Formally it is stated below:

\begin{lemma}\label{lem:brep-is-bprog}
  Given a TBox $T$, an ABox $A$, a service call map $\scmap$ and a
  b-repair program $\delta^T_{b}$ over $T$, we have that if
  $A' \in \progres(A, \scmap, \delta^T_b)$ then $A' \in \arset{T,A}$.
\end{lemma}
\begin{proof}
  By Lemma~\ref{lem:characteristic-bprog-result} and the definition of
  $\arset{T,A}$.
\end{proof}

In order to complete the proof that a b-repair program produces
exactly all b-repair results of the given (inconsistent) KB, we will
show that every b-repair result of the given (inconsistent) KB is
produced by the b-repair program. 

\begin{lemma}\label{lem:bprog-is-brep}
  Given a TBox $T$, an ABox $A$, a service call map $\scmap$ and a
  b-repair program $\delta^T_{b}$ over $T$. If $A' \in \arset{T,A}$,
  then $A' \in \progres(A, \scmap, \delta^T_b)$.
\end{lemma}
\begin{proof}
  We divide the proof into two cases:
  \begin{compactenum}
  \item[\textbf{Case 1: $A$ is $T$-consistent.}]~Trivially true,
    because $\arset{T,A}$ is a singleton set containing $A$ and since
    $\ask(\qunsatecq{T}, T, A) = \false$, we have
    $\final{\tup{A, \scmap, \delta^T_{b}}}$ and hence
    $\progres(A, \scmap, \delta^T_b)$ is also a singleton set
    containing $A$.

  \item[\textbf{Case 2: $A$ is $T$-inconsistent.}]~
%
%
    Let $A_1$ be an arbitrary ABox in $\arset{T,A}$, we have to show
    that there exists $A_2 \in \progres(A, \scmap, \delta^T_b)$ such
    that $A_2 = A_1$.
%
%


    Now, consider an arbitrary concept assertion $N(c) \in A_1$
    (resp.\ role assertion $P(c_1,c_2) \in A_1$), we have to show that
    $N(c) \in A_2$ (resp.\ $P(c_1,c_2) \in A_2$).
%
%
%
    For compactness reason, here we only consider the case for $N(c)$
    (the case for $P(c_1,c_2)$ is similar). Now we have to consider
    two cases:
    \begin{compactenum}[\bf (a)]
    \item $N(c)$ does not violate any negative concept inclusion
      assertion,
    \item $N(c)$, together with another assertion, violate a negative
      concept inclusion assertion. 
    \end{compactenum}
    The proof is as follows:
    \begin{compactitem}

    \item[Case \textbf{(a)}:] It is easy to see that there exists
      $A_2 \in \progres(A, \scmap, \delta^T_b)$ such that
      $N(c) \in A_2$ because by construction of $\delta^T_b$, every
      action $\act \in \actset^T_b$ never deletes any assertion that
      does not violate any negative inclusion.

    \item[Case \textbf{(b)}:] Due to the fact about the source of
      inconsistency in \dllitea, there exists 
      \begin{compactenum}[i.]
      \item $N(c) \in A$,
      \item a negative inclusion $N \sqsubseteq \neg B$ (such that
        $T \models N \sqsubseteq \neg B$), and
      \item $B(c) \in A$.
      \end{compactenum}
      Since $N(c) \in A_1$, then there exists $A_1' \in \arset{T,A}$
      such that $B(c) \in A_1'$.  Now, it is easy to see from the
      construction of b-repair program $\delta^T_b$ that we have
      actions $\act_1, \act_2 \in \actset^T_b$ that one removes only
      $N(c)$ from $A$ and the other removes only $B(c)$ from
      $A$. Hence, w.l.o.g.\ we must have
      $A_2, A_2' \in \progres(A, \scmap, \delta^T_b)$ such that
      $N(c) \in A_2$ but $N(c) \not\in A_2'$ and $B(c) \not\in A_2$
      but $B(c) \in A_2'$.

    \end{compactitem}

    Now, since $N(c)$ is an arbitrary assertion in $A$, by the two
    cases above, and also considering that the other case can be
    treated similarly, we have that
    $A_2 \in \progres(A, \scmap, \delta^T_b)$, where $A_2 = A_1$.

\end{compactenum}
\ \ 
\end{proof}

As a consequence of Lemmas~\ref{lem:bprog-is-brep} and \ref{lem:brep-is-bprog}, we
finally can show the correctness of b-repair program (i.e., produces the
same result as the result of b-repair over KB) as follows.

\begin{theorem}\label{thm:bprog-equal-brep}
  Given a TBox $T$, an ABox $A$, a service call map $\scmap$ and a
  b-repair program $\delta^T_{b}$ over $T$, we have that
  $\progres(A, \scmap, \delta^T_b) =\arset{T,A}$.
\end{theorem}
\begin{proof}
  Direct consequence of Lemmas~\ref{lem:bprog-is-brep} and
  \ref{lem:brep-is-bprog}.
\end{proof}

\subsection{Recasting the Verification of B-GKABs Into S-GKABs}

To show that the verification of \muladom over B-GKAB can be recast as
verification of S-GKAB, we make use the L-Bisimulation.
In particular, we first show that given a B-GKAB $\gkabsym$, its
transition system $\ts{\gkabsym}^{\filter_B}$ is L-bisimilar to the
transition system $\ts{\tgkabb(\gkabsym)}^{\filter_S}$ of S-GKAB
$\tgkabb(\gkabsym)$ that is obtained via the translation $\tgkabb$.
As a consequence, we have that both transition systems
$\ts{\gkabsym}^{\filter_B}$ and $\ts{\tgkabb(\gkabsym)}^{\filter_S}$
can not be distinguished by any \muladom (in NNF) modulo the
translation $\tforb$.

\begin{lemma}\label{lem:bgkab-to-sgkab-bisimilar-state}
  Let $\gkabsym$ be a B-GKAB with transition system
  $\ts{\gkabsym}^{\filter_B}$, and let $\tgkabb(\gkabsym)$ be an
  S-GKAB with transition system $\ts{\tgkabb(\gkabsym)}^{\filter_S}$
  obtain through $\tgkabb$.
  Consider 
\begin{inparaenum}[]
\item a state $\tup{A_b,\scmap_b, \delta_b}$ of
  $\ts{\gkabsym}^{\filter_B}$ and
\item a state $\tup{A_s,\scmap_s, \delta_s}$ of
  $\ts{\tgkabb(\gkabsym)}^{\filter_S}$.
\end{inparaenum}
If  
\begin{inparaenum}[]
\item $A_s = A_b$, $\scmap_s = \scmap_b$ and
\item $\delta_s = \tgprogb(\delta_b)$,
\end{inparaenum}
then
$\tup{A_b,\scmap_b, \delta_b} \lbsim \tup{A_s,\scmap_s, \delta_s}$.
\end{lemma}
\begin{proof}
Let 
\begin{compactitem}
\item $\gkabsym = \tup{T, \initABox, \actset, \ginitprog}$ and \\
  $\ts{\gkabsym}^{\filter_B} = \tup{\const, T, \stateset_b, s_{0b},
    \abox_b, \trans_b}$,
\item
  $\tgkabb(\gkabsym) = \tup{T_s, \initABox, \actset_s,
    \ginitprog_{s}}$ and \\
  $\ts{\tgkabb(\gkabsym)}^{\filter_S} = \tup{\const, T_s, \stateset_s,
    s_{0s}, \abox_s, \trans_s}$.
%
\end{compactitem}
We have to show the following: for every state
$\tup{A''_b,\scmap''_b, \delta''_b}$ such that
$\tup{A_b,\scmap_b, \delta_b} \trans \tup{A''_b,\scmap''_b,
  \delta''_b}$,
there exists states $t_1, \ldots, t_n$, $s'$, and $s''$
such that:
\begin{compactenum}[\bf (a)]
\item $s \trans_s s' \trans_s t_1 \trans_s \ldots \trans_s t_n \trans_s s''$,
  where $s = \tup{A_s,\scmap_s, \delta_s}$,
  $s'' = \tup{A''_s,\scmap''_s, \delta''_s}$, $n \geq 0$,
  $\temp \not\in A''_s$, and
  $\temp \in \abox_s(t_i)$ for $i \in \set{1, \ldots, n}$;
\item $A''_s = A''_b$;
\item $\scmap''_s = \scmap''_b$;
\item $\delta''_s = \tgprogb(\delta''_b)$.
\end{compactenum}

By definition of $\ts{\gkabsym}^{\filter_B}$, 
Since $\tup{A_b,\scmap_b, \delta_b} \trans \tup{A''_b,\scmap''_b,
  \delta''_b}$, we have $\tup{A_b,\scmap_b, \delta_b}
\gprogtrans{\act\sigma_b, \filter_B} \tup{A''_b,\scmap''_b,
  \delta''_b}$.
Hence, by the definition of $\gprogtrans{\act\sigma_b, \filter_B}$, we
have:
\begin{compactitem}
\item there exists an action $\act \in \actset$ with a corresponding
  action invocation $\gact{Q(\vec{p})}{\act(\vec{p})}$ and a legal
  parameter assignment $\sigma_b$ such that $\act$ is executable in
  $A_b$ with legal parameter assignment
  $\sigma_b$, 
\item
  $\tup{\tup{A_b, \scmap_b}, \act\sigma_b, \tup{A''_b, \scmap''_b}}
  \in \tell_{\filter_B}$.
\end{compactitem}
Since
$\tup{\tup{A_b, \scmap_b}, \act\sigma_b, \tup{A''_b, \scmap''_b}} \in
\tell_{\filter_B}$,
by the definition of $\tell_{\filter_B}$, there exists
$\theta_b \in \eval{\addfactss{A_b}{\act\sigma_b}}$ such that
\begin{compactitem}
\item $\theta_b$ and $\scmap_b$ agree on the common values in their
  domains.
\item $\scmap''_b = \scmap_b \cup \theta_b$.
\item
  $\tup{A_b, \addfactss{A_b}{\act\sigma_b}\theta_b, \delfactss{A_b}{\act\sigma_b}, A_b''} \in \filter_B$.
\item $A''_b$ is $T$-consistent.
\end{compactitem}
Since
$\tup{A_b, \addfactss{A_b}{\act\sigma_b}\theta_b,
  \delfactss{A_b}{\act\sigma_b}, A_b''} \in \filter_B$,
by the definition of $\filter_B$, there exists $A'_b$ such that
$A''_b \in \arset{T, A'_b}$, and
$A'_b = (A_b \setminus \delfactss{A_b}{\act\sigma_b}) \cup
\addfactss{A_b}{\act\sigma_b}\theta_b$.

Since $\delta_s = \tgprogb(\delta_b)$, by the definition of
$\tgprogb$, we have that
\[
\begin{array}{@{}l@{}l@{}}
  \tgprogb(&\gact{Q(\vec{p})}{\act(\vec{p})}) =\\
  &\gact{Q(\vec{p})}{\act(\vec{p})} ; \gact{\true}{\act^+_{tmp}()}; \delta^T_b ; \gact{\true}{\act^-_{tmp}()}
\end{array}
\]
Hence, the next executable sub-program on state
$\tup{A_s,\scmap_s, \delta_s}$ is
\[
\delta_s' = 
\gact{Q(\vec{p})}{\act(\vec{p})} ; \gact{\true}{\act^+_{tmp}()}; \delta^T_b ; \gact{\true}{\act^-_{tmp}()}.
\]
Now, since 
\begin{compactitem}
\item $\sigma_b$ maps parameters of $\act \in \actset$ to constants
  in $\adom{A_b}$, 
\item $A_b = A_s$
\end{compactitem}
we can construct
$\sigma_s$ such that $\sigma_s = \sigma_b$.
%
Moreover, we also know that the certain answers computed over $A_b$
are the same to those computed over $A_s$.
Hence, $\act \in \actset_s$ is executable in $A_s$ with legal
parameter assignment $\sigma_s$.
Now, since we have $\scmap_s = \scmap_b$, we can construct
$\theta_s$ such that $\theta_s = \theta_b$. 
Hence, we have the following:
\begin{compactitem}
\item $\theta_s$ and $\scmap_s$ agree on the common values in their
  domains.
\item $\scmap''_s = \theta_s \cup \scmap_s = \theta_b \cup \scmap_b = \scmap_b''$.
\end{compactitem}
Let
\[
A_s' = (A_s \setminus \delfactss{A_s}{\act\sigma_s}) \cup
\addfactss{A_s}{\act\sigma_s}\theta_s,
\]
as a consequence, we have that
\[
\tup{A_s, \addfactss{A_s}{\act\sigma_s}\theta_s,
  \delfactss{A_s}{\act\sigma_s}, A_s'} \in \filter_S.
\]
Since $A_s = A_b$, $\sigma_s = \sigma_b$ and $\theta_s = \theta_b$, it
follows that
\begin{compactitem}
\item
  $\delfactss{A_s}{\act\sigma_s} = \delfactss{A_b}{\act\sigma_b}$,
  and
\item
  $\addfactss{A_s}{\act\sigma_s}\theta_s =
  \addfactss{A_b}{\act\sigma_b}\theta_b$.
\end{compactitem}
Hence, by the construction of $A_s'$ and $A_b'$ above, we have
$A_b' = A_s'$.  
By the definition of $\tgkabb$, we have $T_s = T_p$ (i.e., only
positive inclusion assertion of $T$), hence $A'_s$ is
$T_s$-consistent. Thus, by the definition of $\tell_{\filter_s}$, we
have
$\tup{\tup{A_s,\scmap_s}, \act\sigma_s, \tup{A'_s, \scmap''_s}} \in
\tell_{\filter_s}$.
Moreover, we have
\[
\tup{A_s, \scmap_s, \gact{Q(\vec{p})}{\act(\vec{p})};\delta_0}
\gprogtrans{\act\sigma_s, \filter_s} \tup{A_s', \scmap_s'', \delta_0}
\]
where $\delta_0 = \gact{\true}{\act^+_{tmp}()}; \delta^T_b ; \gact{\true}{\act^-_{tmp}()}$.

Now, we need to show that the rest of program in $\delta_s'$ that
still need to be executed (i.e., $\delta_0$)
will bring us into a state $\tup{A_s'', \scmap_s'', \delta_s''}$
s.t. the claim \textbf{(a)} - \textbf{(e)} are proved.
It is easy to see that
\[
\tup{A_s', \scmap_s'', \delta_0}
\gprogtrans{\act^+_{tmp}\sigma_1, \filter_s} \tup{A_1, \scmap_s'', \delta_1}
\]
where $\delta_1 = \delta^T_b ; \gact{\true}{\act^-_{tmp}()}$.
Since $\delta_1$ does not involve any service calls, w.l.o.g.\ let
\[
\pi = \tup{A_1, \scmap_s'', \delta_1} \gexectrans \tup{A_2, \scmap_s'',
  \delta_2} \gexectrans \cdots 
\]
be a program execution trace induced by $\delta_1$ on
$\tup{A_1, \scmap_s'', \delta_1}$.
By Lemma~\ref{lem:bprog-termination} and Theorem~\ref{thm:bprog-equal-brep}, we
have that
\begin{compactitem}
\item $\delta^T_b$ is always terminate,
\item $\delta^T_b$ produces an ABox $A_n$ such that
  $A_n \in \arset{T, A_1}$,
\end{compactitem}
additionally, by the construction of $\delta^T_b$ and
$\act^-_{tmp}$, we have that 
\begin{compactitem}
\item $\delta^T_b$ never deletes $\temp$, and 
\item $\act^-_{tmp}$ only deletes $\temp$ from the corresponding
  ABox,
\end{compactitem}
therefore, there exists $\tup{A_s'', \scmap_s'', \delta}$ such that\\
$
\pi = \tup{A_s', \scmap_s'', \delta_0} \gexectrans \tup{A_1,
  \scmap_s'', \delta_1} \gexectrans \cdots \\
\hspace*{25mm} \cdots \gexectrans
\tup{A_n, \scmap_s'', \delta_n} \gexectrans \tup{A_s'', \scmap_s'',
  \delta_{n+1}}
$\\
where 
\begin{compactitem}
\item $\temp \not\in A''_s$, 
\item 
  $\temp \in A_i$ (for $1 \leq i \leq n$), 
\item $\final{\tup{A_s'', \scmap_s'', \delta_{n+1}}}$
\item $A_n \in \arset{T, A_1}$ (by Theorem~\ref{thm:bprog-equal-brep})
\item $A_s'' \in \arset{T, A'_b}$ (Since $A'_b = A'_s$,
  $A_s' = A_1 \setminus \temp$, $A_s'' = A_n \setminus \temp$,
  $A_n \in \arset{T, A_1}$, and $\temp$ is a special marker).
\end{compactitem}
W.l.o.g., by Theorem~\ref{thm:bprog-equal-brep}, we have
$A_s'' = A_b''$. Since $\final{\tup{A_s'', \scmap_s'', \delta_{n+1}}}$, we have finished executing $\delta_s'$, and by the
definition of $\tgprogb$ the rest of the program to be executed is
$\delta_s'' = \tgprogb(\delta_b'')$.

Therefore, we have shown that there exists $s', s'', t_1, \ldots, t_n$
(for $n \geq 0$) such that
  \[
  s \trans_s s' \trans_s t_1 \trans_s \ldots \trans_s t_n \trans_s s''
  \]
where
\begin{compactitem} 
\item $s = \tup{A_s,\scmap_s, \delta_s}$,
  $s'' = \tup{A''_s,\scmap''_s, \delta''_s}$,
\item $\temp \not\in A''_s$, and 
\item 
  $\temp \in \abox_2(t_i)$ for $i \in \set{1, \ldots, n}$;
\item $A''_s = A''_b$
\end{compactitem}

The other direction of bisimulation relation can be proven
symmetrically.

\end{proof}

Having Lemma~\ref{lem:bgkab-to-sgkab-bisimilar-state} in hand, we can
easily show that given a B-GKAB $\gkabsym$, its transition system
$\ts{\gkabsym}^{\filter_B}$ is L-bisimilar to the transition
$\ts{\tgkabb(\gkabsym)}^{\filter_S}$ of S-GKAB $\tgkabb(\gkabsym)$
(which is obtained via the translation $\tgkabb$).

\begin{lemma}\label{lem:bgkab-to-sgkab-bisimilar-ts}
  Given a B-GKAB $\gkabsym$, we have
  $\ts{\gkabsym}^{\filter_B} \lbsim
  \ts{\tgkabb(\gkabsym)}^{\filter_S}$
\end{lemma}
\begin{proof}
Let
\begin{compactenum}
\item $\gkabsym = \tup{T, \initABox, \actset, \ginitprog_b}$ and \\
  $\ts{\gkabsym}^{\filter_B} = \tup{\const, T, \stateset_b, s_{0b},
    \abox_b, \trans_b}$,
\item
  $\tgkabb(\gkabsym) = \tup{T_s, \initABox, \actset_s, \ginitprog_s}$,
  and \\
  $\ts{\tgkabb(\gkabsym)}^{\filter_S} = \tup{\const, T_s, \stateset_s,
    s_{0s}, \abox_s, \trans_s}$.
\end{compactenum}
We have that $s_{0b} = \tup{A_0, \scmap_b, \delta_b}$ and
$s_{0s} = \tup{A_0, \scmap_s, \delta_s}$ where
$\scmap_b = \scmap_s = \emptyset$. By the definition of $\tgkabb$, we
also have $\delta_s = \tgprogb(\delta_b)$. Hence, by
Lemma~\ref{lem:bgkab-to-sgkab-bisimilar-state}, we have
$s_{0b} \lbsim s_{0s}$. Therefore, by the definition of
L-bisimulation, we have
$\ts{\gkabsym}^{\filter_B} \lbsim \ts{\tgkabb(\gkabsym)}^{\filter_S}
$.
\end{proof}

With all of these machinery in hand, we are now ready to show that the
verification of \muladom over B-GKABs can be recast as verification
over S-GKAB as follows.


\begin{theorem}\label{thm:bgkab-to-sgkab}
  Given a B-GKAB $\gkabsym$ and a closed $\muladom$ formula $\Phi$ in
  NNF,
  \begin{center}
    $\ts{\gkabsym}^{\filter_B} \models \Phi$ iff
    $\ts{\tgkabb(\gkabsym)}^{\filter_S} \models \tforb(\Phi)$
  \end{center}
\end{theorem}
\begin{proof}
  By Lemma~\ref{lem:bgkab-to-sgkab-bisimilar-ts}, we have that
  $\ts{\gkabsym}^{\filter_B} \lbsim
  \ts{\tgkabb(\gkabsym)}^{\filter_S}$.
  Hence, the claim is directly follows from
  Lemma~\ref{lem:leaping-bisimilar-ts-satisfies-same-formula}.
\end{proof}


\section{From C-GKABs to S-GKABs}\label{sec:proof-cgkab-to-sgkab}

We devote this section to show that the verification of \muladom
properties over C-GKAB can be recast as a corresponding verification
over S-GKAB.
Formally, given a C-GKAB $\gkabsym$ and a $\muladom$ formula $\Phi$,
we show that $\ts{\gkabsym}^{\filter_C} \models \Phi$ iff
$\ts{\tgkabc(\gkabsym)}^{\filter_S} \models \tford(\Phi)$ (This claim
is formally stated and proven in Theorem~\ref{thm:cgkab-to-sgkab}).
%
%
The core idea of the proof is to use a certain bisimulation relation
in which two bisimilar transition systems (w.r.t.\ this bisimulation
relation) can not be distinguished by any \muladom properties modulo
the formula translation $\tford$. Then, we show that the transition
system of a C-GKAB is bisimilar to the transition system of its
corresponding S-GKAB w.r.t.\ this bisimulation relation, and as a
consequence, we easily obtain the proof that we can recast the
verification of \muladom over C-GKABs into the corresponding
verification over S-GKAB. To this purpose, we first introduce several
preliminaries below.


We now define a translation function $\tgprogc$ that essentially
concatenates each action invocation with a c-repair action in order
to simulate the action executions in C-GKABs.  Additionally, the
translation function $\tgprogc$ also serves as a one-to-one
correspondence (bijection) between the original and the translated
program (as well as between the sub-program).
%
  Formally, given a program $\delta$ and a TBox $T$, 
the   \emph{translation $\tgprogc$} which translate a program into a
  program is defined inductively as follows:
\[
\begin{array}{@{}l@{}l@{}}
  \tgprogc(\gact{Q(\vec{p})}{\act(\vec{p})}) &=  
                                               \gact{Q(\vec{p})}{\act(\vec{p})};\gact{\true}{\act^T_c()}\\
  \tgprogc(\gemptyprog) &= \gemptyprog \\
  \tgprogc(\delta_1|\delta_2) &= \tgprogc(\delta_1)|\tgprogc(\delta_2) \\
  \tgprogc(\delta_1;\delta_2) &= \tgprogc(\delta_1);\tgprogc(\delta_2) \\
  \tgprogc(\gif{\varphi}{\delta_1}{\delta_2}) &= \gif{\varphi}{\tgprogc(\delta_1)}{\tgprogc(\delta_2)} \\
  \tgprogc(\gwhile{\varphi}{\delta}) &= \gwhile{\varphi}{\tgprogc(\delta)}
\end{array}
\]
where $\act^T_c$ is a c-repair action over $T$.  

Next, we formally define the translation $\tford$ that transform
\muladom properties $\Phi$ to be verified over C-GKABs into the
corresponding properties to be verified over an S-GKAB as follows.
%
%
\begin{definition}[Translation $\tford$]\label{def:tdup}
  We define a \emph{translation $\tford$} that takes a \muladom
  formula $\Phi$ as an input and produces a new \muladom formula
  $\tford(\Phi)$ by recurring over the structure of $\Phi$ as follows:
  \[
  \begin{array}{lll}
    \bullet\ \tford(Q) &=& Q \\
    \bullet\ \tford(\neg \Phi) &=& \neg \tford(\Phi) \\
    \bullet\ \tford(\exists x.\Phi) &=& \exists x. \tford(\Phi) \\
    \bullet\ \tford(\Phi_1 \vee \Phi_2) &=& \tford(\Phi_1) \vee \tford(\Phi_2) \\
    \bullet\ \tford(\mu Z.\Phi) &=& \mu Z. \tford(\Phi) \\
    \bullet\ \tford(\DIAM{\Phi}) &=& \DIAM{\DIAM{\tford(\Phi)}} 
  \end{array}
  \]
\ \ 
\end{definition}

\subsection{Skip-one Bisimulation (S-Bisimulation)}
We now proceed to define the notion of \emph{skip-one bisimulation}
that we will use to reduce the verification of C-GKABs into S-GKABs as
follows.
\begin{definition}[Skip-one Bisimulation (S-Bisimulation)] 
  Let $\ts{1} = \tup{\const, T, \Sigma_1, s_{01}, \abox_1, \trans_1}$
  and $\ts{2} = \tup{\const, T, \Sigma_2, s_{02}, \abox_2, \trans_2}$
  be transition systems, with
  $\adom{\abox_1(s_{01})} \subseteq \const$
  and $\adom{\abox_2(s_{02})} \subseteq \const$.
  A \emph{skip-one bisimulation} (S-Bisimulation) between $\ts{1}$ and
  $\ts{2}$ is a relation $\B \subseteq \Sigma_1 \times\Sigma_2$ such
  that $\tup{s_1, s_2} \in \B$ implies that:
  \begin{compactenum}
  \item $\abox_1(s_1) = \abox_2(s_2)$
  \item for each $s_1'$, if $s_1 \Rightarrow_1 s_1'$ then there exists
    $t$, and $s_2'$ with
    \[
    s_2 \Rightarrow_2 t \Rightarrow_2 s_2'
    \] 
    such that $\tup{s_1', s_2'}\in\B$, $\tmp \not\in \abox_2(s_2')$
    and $\tmp \in \abox_2(t)$.
  \item for each $s_2'$, if 
    \[
    s_2 \Rightarrow_2 t \Rightarrow_2 s_2'
    \] 
    with $\tmp \in \abox_2(t)$ for $i \in \set{1, \ldots, n}$ and
    $\tmp \not\in \abox_2(s_2')$, then there exists $s_1'$ with
    $s_1 \Rightarrow_1 s_1'$, such that $\tup{s_1', s_2'}\in\B$.
 \end{compactenum}
\ \ 
\end{definition}

\noindent
 Let $\ts{1} = \tup{\const, T, \Sigma_1, s_{01}, \abox_1, \trans_1}$
  and $\ts{2} = \tup{\const, T, \Sigma_2, s_{02}, \abox_2, \trans_2}$
  be transition systems, 
a state $s_1 \in \Sigma_1$ is \emph{S-bisimilar} to
$s_2 \in \Sigma_2$, written $s_1 \sbsim s_2$, if there exists an
S-bisimulation relation $\B$ between $\ts{1}$ and $\ts{2}$ such that
$\tup{s_1, s_2}\in\B$.
A transition system $\ts{1}$ is \emph{S-bisimilar} to $\ts{2}$,
written $\ts{1} \sbsim \ts{2}$, if there exists an S-bisimulation
relation $\B$ between $\ts{1}$ and $\ts{2}$ such that
$\tup{s_{01}, s_{02}}\in\B$.

Now, we advance further to show that two S-bisimilar transition
systems can not be distinguished by any \muladom formula modulo the
translation $\tford$.

\begin{lemma}\label{lem:sbisimilar-state-satisfies-same-formula}
  Consider two transition systems
  $\ts{1} = \tup{\const, T,\stateset_1,s_{01},\abox_1,\trans_1}$ and
  $\ts{2} = \tup{\const, T,\stateset_2,s_{02},\abox_2,\trans_2}$, with
  $\adom{\abox_1(s_{01})} \subseteq \const$ and
  $\adom{\abox_2(s_{02})} \subseteq \const$.  Consider two states
  $s_1 \in \stateset_1$ and $s_2 \in \stateset_2$ such that
  $s_1 \sbsim s_2$. Then for every formula $\Phi$ of
  $\muladom$, 
  and every valuations $\vfo_1$ and $\vfo_2$ that assign to each of
  its free variables a constant $c_1 \in \adom{\abox_1(s_1)}$ and
  $c_2 \in \adom{\abox_2(s_2)}$, such that $c_1 = c_2$, we have that
  \[
  \ts{1},s_1 \models \Phi \vfo_1 \textrm{ if and only if } \ts{2},s_2
  \models \tford(\Phi) \vfo_2.
  \]
\end{lemma}
\begin{proof}
  Similar to
  Lemma~\ref{lem:jumping-bisimilar-states-satisfies-same-formula}, we
  divide the proof into three parts:
  \begin{compactenum}[(1)]
  \item First, we obtain the proof of the claim for formulae of
    $\ladom$
  \item Second, we extend the results to the infinitary logic obtained
    by extending $\ladom$ with arbitrary countable disjunction.
  \item Last, we recall that fixpoints can be translated into this
    infinitary logic, thus proving that the theorem holds for
    $\muladom$.
\end{compactenum}
Since the step (2) and (3) is similar to the proof of
Lemma~\ref{lem:jumping-bisimilar-states-satisfies-same-formula}, here
we only highlight some interesting cases of the proof for the step (1)
(the other cases of step (1) can be shown similarly):

\smallskip
\noindent
\textbf{Proof for $\ladom$.}  

\smallskip
\noindent
\textit{Base case:}
\begin{compactitem}
\item[\textbf{($\Phi = Q$)}.] Since $s_1 \sbsim s_2$, we have
  $\abox_1(s_1) = \abox_2(s_2)$, and hence
  $\Ans(Q, T, \abox_1(s_1)) = \Ans(Q, T, \abox_2(s_2))$.
  Thus, since $\tforj(Q) = Q$ for every valuations $\vfo_1$ and
  $\vfo_2$ that assign to each of its free variables a constant
  $c_1 \in \adom{\abox_1(s_1)}$ and $c_2 \in \adom{\abox_2(s_2)}$,
  such that $c_1 = c_2$, we have
  \[
  \ts{1},s_1 \models Q \vfo_1 \textrm{ if and only if } \ts{2},s_2
  \models \tforj(Q) \vfo_2.
  \]
\end{compactitem}

\smallskip
\noindent
\textit{Inductive step:}
\begin{compactitem}
\item[\textbf{($\Phi = \neg\Psi$)}.]  By Induction hypothesis, for
  every valuations $\vfo_1$ and $\vfo_2$ that assign to each of its
  free variables a constant $c_1 \in \adom{\abox_1(s_1)}$ and
  $c_2 \in \adom{\abox_2(s_2)}$, such that $c_2 = c_1$, we have that
  $\ts{1},s_1 \models \Psi \vfo_1$ if and only if
  $\ts{2},s_2 \models \tford(\Psi) \vfo_1$. Hence,
  $\ts{1},s_1 \not\models \Psi \vfo_1$ if and only if
  $\ts{2},s_2 \not\models \tford(\Psi) \vfo_2$. By definition,
  $\ts{1},s_1 \models \neg \Psi \vfo_1$ if and only if
  $\ts{2},s_2 \models \neg \tford(\Psi) \vfo_2$.  Hence, by the
  definition of $\tford$, we have
  $\ts{1},s_1 \models \neg \Psi \vfo_1$ if and only if
  $\ts{2},s_2 \models \tford(\neg \Psi) \vfo_2$.

\item[\textbf{($\Phi = \DIAM{\Psi}$)}.]  Assume
  $\ts{1},s_1 \models (\DIAM{\Psi}) \vfo_1$, then there exists $s_1'$
  such that $s_1 \trans_1 s_1'$ and $\ts{1},s_1' \models \Psi
  \vfo_1$. Since $s_1 \sbsim s_2$, there exist $t$ and $s_2'$ s.t.\
    \[
    s_2 \trans_2 t \trans_2 s_2'
    \] 
    and $s_1' \sbsim s_2'$.
    Hence, by induction hypothesis, for every valuations $\vfo_2$ that
    assign to each free variables $x$ of $\tford(\Psi)$ a constant $c_2 \in
    \adom{\abox_2(s_2)}$, such that $c_2 = c_1$ with $x/c_1 \in
    \vfo_1$, we have
    \[
    \ts{2},s_2' \models \tford(\Psi_1) \vfo_2.
    \]
    Since $\abox_2(s_2) = \abox_1(s_1)$, and
    $ s_2 \trans_2 t \trans_2 s_2', $ we therefore get
    \[
    \ts{2},s_2 \models ( \DIAM{\DIAM{\tford(\Psi)}} )\vfo_2.
    \]
    Since $\tford(\DIAM{\Phi}) = \DIAM{\DIAM{\tford(\Phi)}} $, we
    therefore have
    \[
    \ts{2},s_2 \models \tford(\DIAM{\Psi} )\vfo_2. 
    \]
    The other direction can be shown in a symmetric way.

\end{compactitem}

\end{proof}

Having Lemma~\ref{lem:sbisimilar-state-satisfies-same-formula} in
hand, we can easily show that two S-bisimilar transition systems can
not be distinguished by any \muladom formulas modulo translation
$\tford$.

\begin{lemma}\label{lem:sbisimilar-ts-satisfies-same-formula}
  Consider two transition systems
  $\ts{1} = \tup{\const_1, T, \stateset_1, s_{01}, \abox_1, \trans_1}$
  and
  $\ts{2} = \tup{\const_2, T, \stateset_2, s_{02}, \abox_2, \trans_2}$
  such that $\ts{1} \sbsim \ts{2}$.  For every closed \muladom formula
  $\Phi$, we have:
  \[
  \ts{1} \models \Phi \textrm{ if and only if } \ts{2} \models
  \tford(\Phi)
  \]
\end{lemma}
\begin{proof}
  Since $s_{01} \sbsim s_{02}$, by
  Lemma~\ref{lem:sbisimilar-state-satisfies-same-formula},
  we have
  \[
  \ts{1}, s_{01} \models \Phi \textrm{ if and only if } \ts{2}, s_{02}
  \models \tford(\Phi)
  \]
  then we have that the proof is completed by observing the definition
  of S-bisimilar transition systems.
\end{proof}

\subsection{Properties of C-Repair and C-Repair Actions.}

To the aim of reducing the verification of C-GKABs into S-GKABs, we
now show some important properties of b-repair, c-repair and also
c-repair action that we will use to recast the verification of C-GKABs
into S-GKABs.  The main purpose of this section is to show that
a c-repair action produces the same results as the computation of
c-repair.

As a start, below we show that for every pair of ABox assertions that
violates a certain negative inclusion assertion, each of them will be
contained in two different ABoxes in the result of b-repair.
%
%

\begin{lemma}\label{lem:disjoint-brep}
  Let $T$ be a TBox, and $A$ be an ABox. For every negative concept
  inclusion assertion $B_1 \sqsubseteq \neg B_2$ such that
  $T \models B_1 \sqsubseteq \neg B_2$ and $B_1 \neq B_2$, if
  $\set{B_1(c), B_2(c)} \subseteq A$ (for any constant
  $c \in \const$), then there exist $A' \in \arset{T, A}$ such that
\begin{inparaenum}[\it (i)]
\item $B_1(c) \in A'$, 
\item $B_2(c) \not\in A'$.
\end{inparaenum}
(Similarly for the case of negative role inclusion assertion $R_1
\sqsubseteq \neg R_2$ s.t.\ $T \models R_1 \sqsubseteq \neg R_2$).
\end{lemma}
\begin{proof}
  Suppose by contradiction $\set{B_1(c), B_2(c)} \subseteq A$, and
  there does not exist $A' \in \arset{T, A}$ such that $B_1(c) \in A'$
  and $B_2(c) \not\in A'$.
%
  Since in \dllitea the violation of negative concept inclusion
  $B_1 \sqsubseteq \neg B_2$ 
  is only caused by a pair of assertions $B_1(c)$ and $B_2(c)$ (for
  any constant $c \in \const$) and by the definition of
  $\arset{T, A}$, it contains all maximal $T$-consistent subset of
  $A$, 
  then there should be a $T$-consistent ABox $A' \in \arset{T, A}$
  such that $B_1(c) \in A'$ and $B_2(c) \not\in A'$ that is obtained
  by just removing $B_2(c)$ from $A$ and keep $B_1(c)$ (otherwise we
  will not have all maximal $T$-consistent subset of $A$ in
  $\arset{T, A}$, which contradicts the definition of $\arset{T, A}$
  itself). 
%
%
  Hence, we have a contradiction 
  Thus, there exists $A' \in \arset{T, A}$ such that
  \begin{inparaenum}[\it (i)]
  \item $B_1(c) \in A'$,
  \item $B_2(c) \not\in A'$.
  \end{inparaenum}
  The proof for the case of negative role inclusion is similar.
\end{proof}

Similarly for the case of functionality assertion, below we show that
for each role assertion that violates a functional assertion, there
exists an ABox in the set of b-repair result that contains only this
role assertion but not the other role assertions that together they
violate the corresponding functional assertion.

\begin{lemma}\label{lem:disjoint-brep-for-funct}
  Given a TBox $T$, and an ABox $A$, for every functional assertion
  $\funct{R}$, if
  $\set{R(c,c_1), R(c,c_2), \ldots, R(c,c_n)} \subseteq A$ (for any
  constants $\set{c, c_1, c_2, \ldots, c_n} \subseteq \const$), then
  there exist $A' \in \arset{T, A}$ such that
\begin{inparaenum}[\it (i)]
\item $R(c,c_1) \in A'$, 
\item $R(c,c_2) \not\in A', \ldots, R(c,c_n) \not\in A'$,
\end{inparaenum}
\end{lemma}
\begin{proof}
  Similar to the proof of Lemma~\ref{lem:disjoint-brep}.
\end{proof}

Below, we show that the result of c-repair does not contains any ABox
assertion that, together with another ABox assertion, violates a
negative inclusion assertion. Intuitively, this fact is obtained by
using Lemma~\ref{lem:disjoint-brep} which said that for every pair of
ABox assertions that violates a negative inclusion assertion, each of
them will be contained in two different ABoxes in the result of
b-repair. As a consequence, we have that both of them are not in the
result of c-repair when we compute the intersection of all of b-repair
results.

\begin{lemma}\label{lem:inconsistent-assertion-not-in-crep}
  Given a TBox $T$, and an ABox $A$, for every negative concept
  inclusion assertion $B_1 \sqsubseteq \neg B_2$ such that
  $T \models B_1 \sqsubseteq \neg B_2$ and $B_1 \neq B_2$, if
  $\set{B_1(c), B_2(c)} \subseteq A$ (for any constant
  $c \in \const{}$), then $B_1(c) \not\in \iarset{T, A}$, and
  $B_2(c) \not\in \iarset{T, A}$.  (Similarly for the case of negative
  role inclusion assertion).
\end{lemma}
\begin{proof}
  Let $B_1(c), B_2(c) \in A$ (for a constant $c \in \const$), then
  by Lemma~\ref{lem:disjoint-brep}, there exist $A' \in \arset{T, A}$
  and $A'' \in \arset{T, A}$ such that
\begin{inparaenum}[\it (i)]
\item $B_1(c) \in A'$, 
\item $B_2(c) \not\in A'$,
\item $B_2(c) \in A''$, and 
\item $B_1(c) \not\in A''$.
\end{inparaenum}
By the definition of c-repair, $\iarset{T, A} = \cap_{A_i \in\arset{T,A}} A_i$.
Since $ B_1(c), B_2(c) \not\in A' \cap A''$, then we have that 
$B_1(c), B_2(c) \not\in \iarset{T, A}$.
The proof for the case of negative role inclusion is similar.
\end{proof}

Similarly, below we show that the result of c-repair does not contains
any role assertion that, together with another role assertion,
violates a functional assertion. The intuition of the proof is similar
to the proof of
Lemma~\ref{lem:inconsistent-assertion-not-in-crep}. I.e., they are
thrown away when we compute the intersection of all of b-repair
results.

\begin{lemma}\label{lem:inconsistent-assertion-not-in-crep-for-functional}
  Given a TBox $T$, and an ABox $A$, for every functionality assertion
  $\funct{R} \in T$, 
  if there exists $\set{R(c,c_1), \ldots, R(c,c_n)} \subseteq A$ (for
  any constants $c, c_1, \ldots, c_n \in \const$), then
  $R(c,c_1) \not\in \iarset{T, A}, \ldots, R(c,c_n) \not\in \iarset{T,
    A}$.
\end{lemma}
\begin{proof}
Similar to the proof of Lemma~\ref{lem:inconsistent-assertion-not-in-crep}
\end{proof}

Now, in the two following Lemmas we show a property of a c-repair
action, namely that a c-repair action deletes all ABox assertions
that, together with another ABox assertion, violate a negative
inclusion or a functionality assertion.

\begin{lemma}\label{lem:inconsistent-assertion-not-in-act-crep}
  Given a TBox $T$, and an ABox $A$, a service call map $\scmap$ and a
  c-repair action $\act^T_c$. 
%
  If
  $(\tup{A,\scmap}, \act^T_c\sigma, \tup{A', \scmap}) \in
  \tell_{\filter_S}$ (where $\sigma$ is an empty substitution), 
  then for every negative concept inclusion assertion
  $B_1 \sqsubseteq \neg B_2$ such that
  $T \models B_1 \sqsubseteq \neg B_2$ and $B_1 \neq B_2$, if
  $\set{B_1(c), B_2(c)} \subseteq A$ (for some $c \in \const$), then
  $B_1(c) \not\in A'$, and $B_2(c) \not\in A'$.  (Similarly for the
  case of the negative role inclusion assertion).
\end{lemma}
\begin{proof}
  Since $T \models B_1 \sqsubseteq \neg B_2$, by the definition of
  $\act^T_c$, we have
\[
\map{\qunsatn(B_1 \ISA \neg B_2, x)} {\set{\del \set{ B_1(x), B_2(x)} } }
\in \eff{\act^T_c}.
\]
Since, $(\tup{A,\scmap}, \act^T_c\sigma, \tup{A', \scmap}) \in
\tell_{\filter_S}$, by the definition of $\tell_{\filter_S}$, we have
$\tup{A, \addfactss{A}{\act^T_c\sigma}\theta, \delfactss{A}{\act^T_c\sigma}, A'}
\in \filter_S$.
Now, it is easy to see that by the definition of filter $\filter_S$
and $\delfactss{A}{\act^T_c\sigma}$, we have $B_1(c) \not\in A'$, and
$ B_2(c) \not\in A'$.
\end{proof}

\begin{lemma}\label{lem:inconsistent-assertion-not-in-act-crep-for-functional}
  Given a TBox $T$, an ABox $A$, a service call map $\scmap$ and a
  c-repair action $\act^T_c$.
%
  If
  $(\tup{A,\scmap}, \act^T_c\sigma, \tup{A', \scmap}) \in
  \tell_{\filter_S}$ (where $\sigma$ is an empty substitution)
  then for every functional assertion $\funct{R} \in T$, if
  $\set{R(c, c_1), \ldots, R(c, c_n)} \subseteq A$ (for some constants
  $\set{c, c_1, \ldots, c_n} \subseteq \const$), then
  $R(c, c_1) \not\in A', \ldots, R(c, c_n) \not\in A'$.
\end{lemma}
\begin{proof}
  Similar to the proof of
  Lemma~\ref{lem:inconsistent-assertion-not-in-act-crep}.
\end{proof}

Next, we show that every Abox assertion that does not violate any
TBox assertions will appear in all results of a b-repair.

\begin{lemma}\label{lem:all-consistent-assertion-in-brep}
  Given a TBox $T$, and an ABox $A$, for every concept assertion
  $C(c) \in A$ (for any constant $c \in \const$) such that
  $C(c) \not\in \inc(A)$,
  it holds that for every $A' \in \arset{T, A}$, we have
  $C(c) \in A'$.
  (Similarly for role assertion).
\end{lemma}
\begin{proof}
  Suppose by contradiction there exists $C(c) \in A$ such that
  $C(c) \not\in \inc(A)$ and there exists $A' \in \arset{T, A}$, such
  that $C(c) \not\in A'$. Then, since $C(c) \not\in \inc(A)$, there
  exists $A''$ such that $A' \subset A'' \subseteq A$ and $A''$ is
  $T$-consistent (where $A'' = A' \cup \set{C(c)}$. Hence,
  $A' \not\in \arset{T, A}$.  Thus we have a contradiction. Therefore
  we proved the claim. The proof for the case of role assertion can be
  done similarly.
\end{proof}

From the previous Lemma, we can show that every ABox assertion that
does not violate any negative inclusion assertion will appear in the
c-repair results.

\begin{lemma}\label{lem:all-consistent-assertion-in-crep}
  Given a TBox $T$, and an ABox $A$, for every concept assertion
  $C(c) \in A$ s.t.\ $C(c) \not\in \inc(A)$
  we have that $C(c) \in \iarset{T, A}$.
  (Similarly for role assertion).
\end{lemma}
\begin{proof}
  Let $C(c) \in A$ be any arbitrary concept assertion s.t.\
  $C(c) \not\in \inc(A)$. By
  Lemma~\ref{lem:all-consistent-assertion-in-brep}, for every
  $A' \in \arset{T, A}$, we have $C(c) \in A'$. Hence, since
  $\iarset{T, A} = \cap_{A_i \in\arset{T,A}} A_i$, we have
  $C(c) \in \iarset{T, A}$. The proof for the case of role assertion
  can be done similarly.
\end{proof}

Finally, we can say that the c-repair action is correctly mimic the
c-repair computation, i.e., they produce the same result.

\begin{lemma}\label{lem:cact-equal-crep}
  Given a TBox $T$, an ABox $A$, a service call map $\scmap$ and a
  c-repair action $\act^T_c$.
  Let $A_1 = \iarset{T,A}$, and
  $(\tup{A,\scmap}, \act^T_c\sigma, \tup{A_2, \scmap}) \in
  \tell_{\filter_S}$
  where $\sigma$ is an empty substitution, then we have $A_1 = A_2$
\end{lemma}
\begin{proof}
  The proof is obtained by observing that $\act^T_c$ never deletes any
  ABox assertion that does not involve in any source of inconsistency,
  and also by using
  Lemmas~\ref{lem:inconsistent-assertion-not-in-act-crep},
  \ref{lem:inconsistent-assertion-not-in-act-crep-for-functional}, and
  \ref{lem:all-consistent-assertion-in-crep}.
\end{proof}

\subsection{Reducing the Verification of C-GKABs Into S-GKABs}

In the following two Lemmas, we aim to show that the transition
systems of a C-GKAB and its corresponding S-GKAB (obtained through
$\tgkabc$) are S-bisimilar.


\begin{lemma}\label{lem:cgkab-to-sgkab-bisimilar-state}
  Let $\gkabsym$ be a B-GKAB with transition system
  $\ts{\gkabsym}^{\filter_C}$, and let $\tgkabc(\gkabsym)$ be an
  S-GKAB with transition system $\ts{\tgkabc(\gkabsym)}^{\filter_S}$
  obtain through
  $\tgkabc$. 
  Consider
  \begin{inparaenum}[]
  \item a state $\tup{A_c,\scmap_c, \delta_c}$ of
    $\ts{\gkabsym}^{\filter_C}$ and
  \item a state $\tup{A_s,\scmap_s, \delta_s}$ of
    $\ts{\tgkabc(\gkabsym)}^{\filter_S}$.
  \end{inparaenum}
If  
\begin{inparaenum}[]
\item $A_s = A_c$, $\scmap_s = \scmap_c$ and
\item $\delta_s = \tgprogc(\delta_c)$,
\end{inparaenum}
then
$\tup{A_c,\scmap_c, \delta_c} \sbsim \tup{A_s,\scmap_s, \delta_s}$.
\end{lemma}
\begin{proof}
Let 
\begin{compactenum}
\item $\gkabsym = \tup{T, \initABox, \actset, \ginitprog}$, and\\
  $\ts{\gkabsym}^{\filter_C} = \tup{\const, T, \stateset_c, s_{0c},
    \abox_c, \trans_c}$,
\item
  $\tgkabc(\gkabsym) = \tup{T_s, \initABox, \actset_s,
    \ginitprog_{s}}$, and\\
  $\ts{\tgkabc(\gkabsym)}^{\filter_S} = \tup{\const, T_s,
    \stateset_s, s_{0s}, \abox_s, \trans_s}$.
\end{compactenum}
Now, we have to show the following: For every state
$\tup{A''_c,\scmap''_c, \delta''_c}$ such that
  \[
  \tup{A_c,\scmap_c, \delta_c} \trans \tup{A''_c,\scmap''_c,
    \delta''_c},
  \]
  there exists states $\tup{A'_s,\scmap'_s, \delta'_s}$ and
  $\tup{A''_s,\scmap''_s, \delta''_s}$ such that:
\begin{compactenum}[\bf (a)]
\item we have
  $ \tup{A_s,\scmap_s, \delta_s} \trans_s \tup{A'_s,\scmap'_s,
    \delta'_s} \trans_s \tup{A''_s,\scmap''_s, \delta''_s} $
\item $A''_s = A''_c$;
\item $\scmap''_s = \scmap''_c$;
\item $\delta''_s = \tgprogc(\delta''_c)$.
\end{compactenum}

By definition of $\ts{\gkabsym}^{\filter_C}$, 
Since $\tup{A_c,\scmap_c, \delta_c} \trans \tup{A''_c,\scmap''_c,
  \delta''_c}$, we have $\tup{A_c,\scmap_c, \delta_c}
\gprogtrans{\alpha\sigma_c, \filter_C} \tup{A''_c,\scmap''_c,
  \delta''_c}$.
Hence, by the definition of $\gprogtrans{\act\sigma_c, \filter_C}$, we
have that:
\begin{compactitem}
\item there exists an action $\act \in \actset$ with a corresponding
  action invocation $\gact{Q(\vec{p})}{\act(\vec{p})}$ and legal
  parameter assignment $\sigma_c$ such that $\act$ is executable in
  $A_c$ with (legal parameter assignment)
  $\sigma_c$, 
  and
\item
  $\tup{\tup{A_c, \scmap_c}, \act\sigma_c, \tup{A''_c, \scmap''_c}}
  \in \tell_{\filter_C}$.
\end{compactitem}
Since
$\tup{\tup{A_c, \scmap_c}, \act\sigma_c, \tup{A''_c, \scmap''_c}} \in
\tell_{\filter_C}$,
by the definition of $\tell_{\filter_C}$, there exists
$\theta_c \in \eval{\addfactss{A_c}{\act\sigma_c}}$ such that
\begin{compactitem}
\item $\theta_c$ and $\scmap_c$ agree on the common values in their domains.
\item $\scmap''_c = \scmap_c \cup \theta_c$.
\item
  $(A_c, \addfactss{A_c}{\act\sigma_c}\theta_c,
  \delfactss{A_c}{\act\sigma_c}, A_c'') \in \filter_C$.
\item $A''_c$ is $T$-consistent.
\end{compactitem}
Since
$\tup{A_c, \addfactss{A_c}{\act\sigma_c}\theta_c,
  \delfactss{A_c}{\act\sigma_c}, A_c''} \in \filter_C$,
by the definition of $\filter_C$, there exists $A'_c$ such that
$A''_c \in \iarset{T, A'_c}$ where
$A'_c = (A_c \setminus \delfactss{A_c}{\act\sigma_c}) \cup
\addfactss{A_c}{\act\sigma_c}\theta_c$.
Furthermore, since $\delta_s = \tgprogc(\delta_c)$, by the definition of
$\tgprogc$, we have that
\[
\begin{array}{l}
  \tgprogc(\gact{Q(\vec{p})}{\act(\vec{p})}) =
  \gact{Q(\vec{p})}{\act(\vec{p})}; \gact{\true}{\act^T_c()}.
\end{array}
\]
Hence, we have that the next executable part of program on state
$\tup{A_s,\scmap_s, \delta_s}$ is
\[
\gact{Q(\vec{p})}{\act(\vec{p})}; \gact{\true}{\act^T_c()}.
\]
%

Now, since $\sigma_c$ maps parameters of $\act \in \actset$ to
constants in $\adom{A_c}$, and $A_c = A_s$, we can construct
$\sigma_s$ mapping parameters of $\act \in \actset_s$ to constants in
$\adom{A_s}$ such that $\sigma_c = \sigma_s$.
Moreover, since $A_s = A_c$, the certain answers computed over $A_c$
are the same to those computed over $A_s$.
%
%
Hence, $\act \in \actset_s$ is executable in $A_s$ with (legal
parameter assignment) $\sigma_s$.
Now, since 
we have $\scmap_s= \scmap_c$, then we can construct $\theta_s$ such
that $\theta_s = \theta_c$.
%
%
Hence, we have the following:
\begin{compactitem}
\item $\theta_s$ and $\scmap_s$ agree on the common values in their
  domains.
 \item $\scmap'_s = \theta_s \cup \scmap_s = \theta_c \cup \scmap_c =\scmap_c''$.
\end{compactitem}
Let
$A_s' = (A_s \setminus \delfactss{A_s}{\act\sigma_s}) \cup
\addfactss{A_s}{\act\sigma_s}\theta_s$,
as a consequence, we have
$\tup{A_s, \addfactss{A_s}{\act\sigma_s}\theta_s, \delfactss{A_s}{\act\sigma_s}, A_s'} \in \filter_S$.
Since $A_s = A_c$, $\theta_s = \theta_c$, and $\sigma_s = \sigma_c$, it follows that
\begin{compactitem}
\item
  $\delfactss{A_s}{\act\sigma_s} = \delfactss{A_c}{\act\sigma_c}$,
  and
\item
  $\addfactss{A_s}{\act\sigma_s}\theta_s =
  \addfactss{A_c}{\act\sigma_c}\theta_c$.
\end{compactitem}
Hence, by the construction of $A_s'$ and $A_c'$ above, we also have
$A_s' = A_c'$.
By the definition of $\tgkabc$, we have $T_s = T_p$ (i.e., only
positive inclusion assertion of $T$), hence $A'_s$ is
$T_s$-consistent. Thus, by the definition of $\tell_{\filter_s}$, we
have
$\tup{\tup{A_s,\scmap_s}, \act\sigma_s, \tup{A'_s, \scmap'_s}} \in
\tell_{\filter_s}$.
Moreover, we have 
\[
\tup{A_s, \scmap_s,
  \gact{Q(\vec{p})}{\act(\vec{p})};\delta_0}
\gprogtrans{\act\sigma_s, \filter_s} \tup{A_s', \scmap_s',
  \delta_0}
\]
where $\delta_0 = \gact{\true}{\act^T_c()}$. 
Now, it is easy to see that
 \[
 \tup{A_s', \scmap_s', \gact{\true}{\act^T_c()}}
 \gprogtrans{\act^T_c\sigma, \filter_s} \tup{A_s'', \scmap_s'',
   \gemptyprog}
 \]
 where 
\begin{compactitem}
\item $\scmap_s'' = \scmap_s'$ (since $\act^T_c$ does not involve
 any service call), 
\item $\sigma$ is empty substitution (because $\act^T_c$
 is a 0-ary action),
\item $\final{\tup{A_s'', \scmap_s'', \gemptyprog}}$.
\item $A_s'' = \iarset{T, A_s'}$ (by Lemma~\ref{lem:cact-equal-crep})
\end{compactitem}
Since $A_s' = A_c'$, $A_s'' = \iarset{T, A_s'}$, and
$A_c'' = \iarset{T, A_c'}$, then we have $A_s'' = A_c''$.  Moreover,
since $\final{\tup{A_s'', \scmap_s'', \gemptyprog}}$, we have
successfully finished executing 
\[
\gact{Q(\vec{p})}{\act(\vec{p})}; \gact{\true}{\act^T_c()},
\]
and by the definition of $\tgprogc$
the rest of the program to be executed is
$\delta_s'' = \tgprogc(\delta_c'')$. Thus, we have
\[
\tup{A_s,\scmap_s, \delta_s} \trans_s \tup{A'_s,\scmap'_s, \delta'_s}
\trans_s \tup{A''_s,\scmap''_s, \delta''_s}
\]
where
\begin{compactenum}[\bf (a)]
\item $A''_s = A''_c$;
\item $\scmap''_s = \scmap''_c$;
\item $\delta''_s = \tgprogc(\delta''_c)$.
\end{compactenum}
The other direction of bisimulation relation can be proven
symmetrically.
\end{proof}

Having Lemma~\ref{lem:cgkab-to-sgkab-bisimilar-state} in hand, we can
easily show that given a C-GKAB, its transition system is S-bisimilar
to the transition of its corresponding S-GKAB that is obtained via
the translation $\tgkabc$ as follows.

\begin{lemma}\label{lem:cgkab-to-sgkab-bisimilar-ts}
  Given a C-GKAB $\gkabsym$, we have
  $\ts{\gkabsym}^{\filter_C} \sbsim
  \ts{\tgkabc(\gkabsym)}^{\filter_S}$
\end{lemma}
\begin{proof}
Let
\begin{compactenum}
\item $\gkabsym = \tup{T, \initABox, \actset, \ginitprog_c}$, and\\
  $\ts{\gkabsym}^{\filter_C} = \tup{\const, T, \stateset_c, s_{0c},
    \abox_c, \trans_c}$,
\item
  $\tgkabc(\gkabsym) = \tup{T_s, \initABox, \actset_s, \ginitprog_s}$,
  and \\
  $\ts{\tgkabc(\gkabsym)}^{\filter_S} = \tup{\const, T_s,
    \stateset_s, s_{0s}, \abox_s, \trans_s}$.
\end{compactenum}
We have that $s_{0c} = \tup{A_0, \scmap_c, \delta_c}$ and
$s_{0s} = \tup{A_0, \scmap_s, \delta_s}$ where
$\scmap_c = \scmap_s = \emptyset$. By the definition of $\tgprogc$ and
$\tgkabc$, we also have $\delta_s = \tgprogc(\delta_c)$. Hence, by
Lemma~\ref{lem:cgkab-to-sgkab-bisimilar-state}, we have
$s_{0c} \sbsim s_{0s}$. Therefore, by the definition of
S-bisimulation, we have
$\ts{\gkabsym}^{\filter_C} \sbsim \ts{\tgkabc(\gkabsym)}^{\filter_S}$.
\ \ 
\end{proof}

Finally, we are now ready to show that the verification of \muladom
formulas over C-GKABs can be recast as verification of \muladom
formulas over S-GKAB as follows.

\begin{theorem}\label{thm:cgkab-to-sgkab}
  Given a C-GKAB $\gkabsym$ and a closed \muladom property $\Phi$,
\begin{center}
  $\ts{\gkabsym}^{\filter_C} \models \Phi$ iff
  $\ts{\tgkabc(\gkabsym)}^{\filter_S} \models \tford(\Phi)$
\end{center}
\end{theorem} 
\begin{proof}
  By Lemma~\ref{lem:cgkab-to-sgkab-bisimilar-ts}, we have that
  $\ts{\gkabsym}^{\filter_C} \sbsim
  \ts{\tgkabc(\gkabsym)}^{\filter_S}$.
  Hence, by Lemma~\ref{lem:sbisimilar-ts-satisfies-same-formula}, it is
  easy to see that the claim is proved.
\end{proof}

\section{From E-GKABs to S-GKABs.}

Here we show that the verification of \muladom properties over
E-GKABs can be recast as verification over S-GKABs.
Formally, given an E-GKAB $\gkabsym$ and a $\muladom$ formula $\Phi$,
we show that $\ts{\gkabsym}^{\filter_E} \models \Phi$ if and only if
$\ts{\tgkabe(\gkabsym)}^{\filter_S} \models \tford(\Phi)$ (This claim
is formally stated and proven in Theorem~\ref{thm:egkab-to-sgkab}).
%
%
The strategy of the proof is similar to the reduction from the
verification of C-GKABs into the verification of S-GKABs in
Section~\ref{sec:proof-cgkab-to-sgkab}. I.e., to show that the
transition system of an E-GKAB is S-bisimilar to the transition system
of its corresponding S-GKAB, and hence they can not be distinguish by
any \muladom formulas modulo translation $\tford$.

As a preliminary, we define a translation function $\tgproge$ that
essentially concatenates each action invocation with an evolution
action in order to simulate the action executions in E-GKABs.
Additionally, the translation function $\tgproge$ also serves as a
one-to-one correspondence (bijection) between the original and the
translated program (as well as between the sub-programs).
Formally, given a program $\delta$ and a TBox $T$, we define a
\emph{translation $\tgproge$} which translate a program into a program
inductively as follows:
\[
\begin{array}{@{}l@{}l@{}}
  \tgproge(\gact{Q(\vec{p})}{\act(\vec{p})}) &=  
                                               \gact{Q(\vec{p})}{\act'(\vec{p})};\gact{\true}{\act^T_e()}\\
  \tgproge(\gemptyprog) &= \gemptyprog \\
  \tgproge(\delta_1|\delta_2) &= \tgproge(\delta_1)|\tgproge(\delta_2) \\
  \tgproge(\delta_1;\delta_2) &= \tgproge(\delta_1);\tgproge(\delta_2) \\
  \tgproge(\gif{\varphi}{\delta_1}{\delta_2}) &= \gif{\varphi}{\tgproge(\delta_1)}{\tgproge(\delta_2)} \\
  \tgproge(\gwhile{\varphi}{\delta}) &= \gwhile{\varphi}{\tgproge(\delta)}
\end{array}
\]
where $\act'$ and $\act^T_e$ are defined as in the Section
\ref{EGKABToSGKAB}. 

\subsection{Reducing the Verification of E-GKABs Into S-GKABs}

As the first step, we show an important property of the filter
$\filter_E$ (which is also a property of $\evol$
operator). Particularly, we show that every ABox assertion in the
evolution result is either a new assertion or it was already in the
original ABox and it was not deleted as well as did not violate any
TBox constraints (together with another ABox assertions). Formally the
claim is stated below.

\begin{lemma}\label{lem:evol-prop}
  Given
  \begin{inparaenum}[]
  \item a TBox $T$,
  \item a $T$-consistent ABox $A$,
  \item a $T$-consistent set $\facta$ of ABox assertion to be added,
    and
  \item a set $\factd$ of ABox assertion to be deleted.
  \end{inparaenum}
  such that $A_e = \evol(T, A, \facta, \factd)$, 
  we have $N(c) \in A_e$
  if and only if either
  \begin{compactenum}
  \item $N(c) \in \facta$, or
  \item $N(c) \in (A \setminus \factd)$ and there does not exists
    $B(c) \in \facta$ such that
    $T \models N \sqsubseteq \neg B$.
  \end{compactenum}
  (Similarly for the case of role assertion). 
\end{lemma}
\begin{proof}\ \\
\begin{compactitem} 
\item[``$\Lora$'':]
  Assume $N(c) \in A_e$, since
  $A_e = \evol(T, A, \facta, \factd)$, 
  by the definition of $\evol(T, A, \facta, \factd)$,
  we have $A_e = \facta \cup A'$, where
  \begin{compactenum}
  \item $A' \subseteq (A \setminus \factd)$,
  \item $\facta \cup A'$ is $T$-consistent, and
  \item there does not exists $A''$ such that
    $A' \subset A'' \subseteq (A \setminus \factd)$ and
    $\facta \cup A''$ is $T$-consistent.
  \end{compactenum}
  Hence, we have either
  \begin{compactenum}[(1)]
  \item $N(c) \in \facta$, or
  \item $N(c) \in A'$. 
  \end{compactenum}
  For the case (2), as a consequence:
  \begin{compactitem}
  \item Since $N(c) \in A'$ and $A' \subseteq (A \setminus \factd)$ it
    follows that $N(c) \in (A \setminus \factd)$.
  \item Since $F^+ \cup A'$ is $T$-consistent, then we have that there
    does not exists $B(c) \in \facta$ s.t.\
    $T \models N \sqsubseteq \neg B$.
  \end{compactitem}
  Thus, the claim is proven.

\smallskip
\item[``$\Lola$'':]
We divide the proof into two parts:
\begin{compactenum}[(1)]

\item Assume $N(c) \in \facta$. Then simply by the definition of
  $\evol(T, A, \facta, \factd)$, we have $N(c) \in A_e$.

\item Supposed by contradiction we have that
  $N(c) \in (A \setminus \factd)$ and there does not exists
  $B(c) \in \facta$ s.t.\ $T \models N \sqsubseteq \neg B$, and
  $N(c) \not\in A_e$. Since $N(c) \not\in A_e$, by the definition of
  $\evol(T, A, \facta, \factd)$, we have that $N(c) \not\in \facta$
  and $N(c) \not\in A'$ in which $A'$ should satisfies the following:
  \begin{compactitem}
  \item $A' \subseteq (A \setminus \factd)$,
  \item $\facta \cup A'$ is $T$-consistent, and
  \item there does not exists $A''$ such that
    $A' \subset A'' \subseteq (A \setminus \factd)$ and $\facta\cup A''$ is
    $T$-consistent.
  \end{compactitem}
  But then we have a contradiction since there exists
  $A'' = A' \cup \set{N(c)}$ such that
  $A' \subset A'' \subseteq (A \setminus \factd)$ and $\facta\cup A''$
  is $T$-consistent. 
  Hence, we must have $N(c) \in A_e$.
\end{compactenum}
\end{compactitem}
\ \ 
\end{proof}

Now we show an important property of evolution action $\act^T_e$ which
says that every ABox assertion in the result of the execution of
$\act^T_e$ is either a newly added assertion, or an old assertion that
does not violate any TBox constraints. Precisely we state this
property below.

\begin{lemma}\label{lem:eact-prop}
  Given 
\begin{compactitem}
\item an E-GKAB
  $\gkabsym = \tup{T, \initABox, \actset_e, \ginitprog_e}$ with
  transition system $\ts{\gkabsym}^{\filter_E}$,
  and 
\item an S-GKAB
  $\tgkabe(\gkabsym) = \tup{T_s, \initABox, \actset_s, \ginitprog_s}$
  (with transition system $\ts{\tgkabe(\gkabsym)}^{\filter_S}$)
  that is obtained from $\gkabsym$ through $\tgkabe$,
  where $T_s = T_p \cup T^n$.
\end{compactitem}
Let $\tup{A, \scmap, \delta}$ be any state in
$\ts{\tgkabe(\gkabsym)}^{\filter_S}$, $\act' \in \actset_s$ be any
action, $A$ is $T_s$-consistent and does not contain any ABox
assertions constructed from $\voc(T^n)$ and we have:
\[
\tup{A, \scmap, \delta} 
\gprogtrans{\act'\sigma, \filter_S} 
\tup{A', \scmap', \delta'} 
\gprogtrans{\act^T_e\sigma', \filter_S} 
\tup{A'', \scmap'', \delta''}
\]
for
\begin{compactitem}
\item a particular legal parameter assignment $\sigma$
\item an empty substitution $\sigma'$,
\item a particular service call evaluation
  $\theta \in \eval{\addfactss{A}{\act'\sigma}}$ that agree with
  $\scmap$ on the common values in their domains.
\end{compactitem}
We have $N(c) \in A''$
if and only if $N$ is not in the vocabulary of TBox $T^n$ and either
\begin{compactenum}
\item $N(c) \in \addfactss{A}{\act'\sigma}\theta$, or
\item $N(c) \in (A \setminus \delfactss{A}{\act'\sigma})$ and there
  does not exists $B(c) \in \addfactss{A}{\act'\sigma}\theta$ such
  that $T \models N \sqsubseteq \neg B$.
\end{compactenum}
(Similarly for the case of role assertion). 
\end{lemma}
\begin{proof}\ \\
\begin{compactitem}
\item[``$\Lora$'':] Assume $N(c) \in A''$, since the evolution action
  $\act^T_e$ only
\begin{compactenum}
\item removes old assertions when inconsistency arises,
\item flushes every ABox assertions constructed by the vocabulary of
  $T^n$,
\end{compactenum}
then we have the following:
\begin{compactenum}
\item $N$ is not in the vocabulary of TBox $T^n$ (otherwise it will
  be flushes by $\act^T_e$)
\item $N(c) \in A'$ (because $\act^T_e$ never introduce a new ABox
  assertion),
\item if there exists $B(c) \in A'$ such that
  $T \models N \sqsubseteq \neg B$, then $B(c) \not\in A''$,
  $B^n(c) \not\in A'$, and $N^n(c) \in A'$ (i.e., if
  $N(c) \in A'$ violates a negative inclusion assertion, $N(c)$
  must be a newly added ABox assertion, otherwise it will be deleted
  by $\act^T_e$).
\end{compactenum}
Now, since $A$ and $A'$ are $T_s$-consistent (because
$\tup{A, \scmap, \delta} \gprogtrans{\act'\sigma, \filter_S} \tup{A',
  \scmap', \delta'} $),
then $\addfactss{A}{\act'\sigma}\theta$ is $T_s$-consistent.  Hence we
have either
\begin{compactenum}
\item $N(c) \in \addfactss{A}{\act'\sigma}\theta$ (and there does
  not exists $B(c)$ such that
  $B(c) \in \addfactss{A}{\act'\sigma}\theta$, and
  $T \models N \sqsubseteq \neg B$), or
\item $N(c) \in (A \setminus \delfactss{A}{\act'\sigma})$ and
  there does not exists 
  $B(c) \in \addfactss{A}{\act'\sigma}\theta$ such that
  $T \models N \sqsubseteq \neg B$ (otherwise we have
  $\set{N(c), B(c), B^n(c)} \subseteq A'$ and then $N(c)$ will
  be deleted by $\act_e^T$).
\end{compactenum}
Therefore, the claim is proved.

\smallskip
\item[``$\Lola$'':] We divide the proof into two parts:
\begin{compactenum}

\item Assume $N(c) \in \addfactss{A}{\act'\sigma}\theta$. Then, by
  the construction of $\act'$ and the definition of
  $\gprogtrans{\act'\sigma, \filter_S}$, it is easy to see that
  $N(c), N^n(c) \in A'$. Moreover, $N(c) \in A''$ (by
  construction of $\act^T_e$).

\item Assume $N(c) \in (A \setminus \delfactss{A}{\act'\sigma})$
  and there does not exists
  $B(c) \in \addfactss{A}{\act'\sigma}\theta$ s.t.\
  $T \models N \sqsubseteq \neg B$. Hence, by the definition of
  $\gprogtrans{\act'\sigma, \filter_S}$, we have $N(c) \in A'$.
  Moreover, because $N(c) \in A'$ does not violate any negative
  inclusion assertions, by construction of $\act^T_e$, we also simply
  have $N(c) \in A''$.
\end{compactenum}

\end{compactitem}

\end{proof}

Next, in the following two Lemmas we aim to show that the transition
system of an E-GKAB is S-bisimilar to the transition system of its
corresponding S-GKAB that is obtained from translation $\tgkabe$.

\begin{lemma}\label{lem:egkab-to-sgkab-bisimilar-state}
  Let $\gkabsym = \tup{T, \initABox, \actset, \ginitprog}$ be an E-GKAB with transition system
  $\ts{\gkabsym}^{\filter_E}$, and let $\tgkabe(\gkabsym) = \tup{T_s, \initABox, \actset_s, \ginitprog_s}$ be an
  S-GKAB with transition system $\ts{\tgkabe(\gkabsym)}^{\filter_S}$
  obtain through $\tgkabe$. 
%
  Consider
\begin{inparaenum}[]
\item a state $\tup{A_e,\scmap_e, \delta_e}$ of
  $\ts{\gkabsym}^{\filter_E}$ and
\item a state $\tup{A_s,\scmap_s, \delta_s}$ of
  $\ts{\tgkabc(\gkabsym)}^{\filter_S}$.
\end{inparaenum}
If  
\begin{inparaenum}[]
\item $A_s = A_e$, $\scmap_s = \scmap_e$, $A_s$ is $T$-consistent
  and 
\item $\delta_s = \tgproge(\delta_e)$, 
\end{inparaenum}
then
$\tup{A_e,\scmap_e, \delta_e} \sbsim \tup{A_s,\scmap_s, \delta_s}$.
\end{lemma}
\begin{proof}
Let 
\begin{inparaitem}[]
\item 
  $\ts{\gkabsym}^{\filter_E} = \tup{\const, T, \stateset_e, s_{0e},
    \abox_e, \trans_e}$, and  
\item
%
  $\ts{\tgkabe(\gkabsym)}^{\filter_S} = \tup{\const, T_s,
    \stateset_s, s_{0s}, \abox_s, \trans_s}$. 
\end{inparaitem}
We have to show the following: for every state
$\tup{A''_e,\scmap''_e, \delta''_e}$ such that
  \[
  \tup{A_e,\scmap_e, \delta_e} \trans \tup{A''_e,\scmap''_e,
    \delta''_e},
  \]
  there exist states $\tup{A'_s,\scmap'_s, \delta'_s}$ and
  $\tup{A''_s,\scmap''_s, \delta''_s}$ such that:
\begin{compactenum}[\bf (a)]
\item
  $ \tup{A_s,\scmap_s, \delta_s} \trans_s \tup{A'_s,\scmap'_s,
    \delta'_s} \trans_s \tup{A''_s,\scmap''_s, \delta''_s} $
\item $A''_s = A''_e$;
\item $\scmap''_s = \scmap''_e$;
\item $\delta''_s = \tgproge(\delta''_e)$.
\end{compactenum}
By definition of $\ts{\gkabsym}^{\filter_E}$, 
since $\tup{A_e,\scmap_e, \delta_e} \trans \tup{A''_e,\scmap''_e,
  \delta''_e}$, we have $\tup{A_e,\scmap_e, \delta_e}
\gprogtrans{\act\sigma_e, \filter_E} \tup{A''_e,\scmap''_e,
  \delta''_e}$.
Hence, by the definition of $\gprogtrans{\act\sigma_e, \filter_E}$, we
have:
\begin{compactitem}
\item there exists an action $\act \in \actset$ with a corresponding
  action invocation $\gact{Q(\vec{p})}{\act(\vec{p})}$ and legal
  parameter assignment $\sigma_e$ such that $\act$ is executable in
  $A_e$ with (legal parameter assignment)
  $\sigma_e$, 
  and
\item $\tup{\tup{A_e, \scmap_e}, \act\sigma_e, \tup{A''_e, \scmap''_e}} \in \tell_{\filter_E}$.
\end{compactitem}
Since
$\tup{\tup{A_e, \scmap_e}, \act\sigma_e, \tup{A''_e, \scmap''_e}} \in
\tell_{\filter_E}$,
by the definition of $\tell_{\filter_E}$, there exists
$\theta_e \in \eval{\addfactss{A_e}{\act\sigma_e}}$ such that
\begin{compactitem}
\item $\theta_e$ and $\scmap_e$ agree on the common values in their domains.
\item $\scmap''_e = \scmap_e \cup \theta_e$.
\item
  $\tup{A_e, \addfactss{A_e}{\act\sigma_e}\theta_e,
    \delfactss{A_e}{\act\sigma_e}, A_e''} \in \filter_E$.
\item $A''_e$ is $T$-consistent.
\end{compactitem}
Since
$\tup{A_e, \addfactss{A_e}{\act\sigma_e}\theta_e,
  \delfactss{A_e}{\act\sigma_e}, A_e''} \in \filter_E$,
by the definition of $\filter_E$, we have
\begin{compactitem}
\item $\addfactss{A_e}{\act\sigma_e}\theta_e$ is $T$-consistent.
\item
  $ A_e'' = \evol(T, A_e, \addfactss{A_e}{\act\sigma_e}\theta_e,
  \delfactss{A_e}{\act\sigma_e})$.
\end{compactitem}
Furthermore, since $\delta_s = \tgproge(\delta_e)$, by the definition of
$\tgproge$, we have that
\[
\begin{array}{l}
  \tgproge(\gact{Q(\vec{p})}{\act(\vec{p})}) =  \gact{Q(\vec{p})}{\act'(\vec{p})}; \gact{\true}{\act^T_e()}
\end{array}
\]
Hence, the part of program that we need to execute on state
$\tup{A_s,\scmap_s, \delta_s}$ is
\[
\gact{Q(\vec{p})}{\act'(\vec{p})}; \gact{\true}{\act^T_e()}.
\]
%
%

Now, since: 
\begin{compactitem}
\item $\act' \in \actset_s$ is obtained from $\act \in \actset$
through $\tgkabe$,
%
%
\item the translation $\tgkabe$ transform $\act$ into $\act'$ without
  changing its parameters, 
\item $\sigma_e$ maps parameters of $\act \in \actset$ to individuals
  in $\adom{A_e}$
\end{compactitem}
then we can construct $\sigma_s$ mapping parameters of
$\act' \in \actset_s$ to individuals in $\adom{A_s}$ such that
$\sigma_s = \sigma_e$
%
Moreover, since $A_s = A_e$, we know that the certain answers computed
over $A_e$ are the same to those computed over $A_s$. Hence
%
%
$\act' \in \actset_s$ is executable in $A_s$ with (legal parameter
assignment) $\sigma_s$.
Furthermore, since $\scmap_s= \scmap_e$, then we can construct
$\theta_s$, such that $\theta_s = \theta_e$. 
%
%
Hence, we have the following:
\begin{compactitem}
\item $\theta_s$ and $\scmap_s$ agree on the common
values in their domains. 
\item
  $\scmap'_s = \theta_s \cup \scmap_s = \theta_e \cup \scmap_e =
  \scmap_e''$.
\end{compactitem}
Let
$A_s' = (A_s \setminus \delfactss{A_s}{\act'\sigma_s}) \cup
\addfactss{A_s}{\act'\sigma_s}\theta_s$,
as a consequence, we have
$\tup{A_s, \addfactss{A_s}{\act'\sigma_s}\theta_s, \delfactss{A_s}{\act'\sigma_s}, A_s'} \in \filter_S$.

Since $A_s = A_e$, $\theta_s = \theta_e$, and $\sigma_s = \sigma_e$,
it follows that
\begin{compactitem}
\item $\delfactss{A_e}{\act\sigma_e}  =  \delfactss{A_s}{\act'\sigma_s}$.
\item $N(c) \in \addfactss{A_e}{\act\sigma_e}\theta_e$ if and only
  if $N(c), N^n(c) \in \addfactss{A_s}{\act'\sigma_s}\theta_s$.
\item $P(c_1,c_2) \in \addfactss{A_e}{\act\sigma_e}\theta_e$ if and
  only if
  $P(c_1,c_2), P^n(c_1,c_2) \in
  \addfactss{A_s}{\act'\sigma_s}\theta_s$.
\end{compactitem}

\noindent
As a consequence, since $\addfactss{A_e}{\act\sigma_e}\theta_e$ is
$T$-consistent, then we have
$\addfactss{A_s}{\act'\sigma_s}\theta_s$ is $T_s$-consistent.
Moreover, 
because $A_s$ is $T_s$-consistent,
$\addfactss{A_s}{\act'\sigma_s}\theta_s$ is $T_s$-consistent, and
also considering how $A_s'$ is constructed, we then have $A_s'$ is
$T_s$-consistent.
Thus we have
$\tup{\tup{A_s,\scmap_s}, \act'\sigma_s, \tup{A_s', \scmap_s'}} \in
\tell_{\filter_S}$, and we also have\\ \\
$
\tup{A_s,\scmap_s,
  \gact{Q(\vec{p})}{\act'(\vec{p})};\gact{\true}{\act^T_e()}} \\ \hspace*{33mm}
\gprogtrans{\act'\sigma_s, \filter_S} \tup{A_s', \scmap_s',
  \gact{\true}{\act^T_e()}}.
$\\
It is to see that we have
\[
\begin{array}{@{}l@{}l@{}}
  \tup{A_s', \scmap_s', \gact{\true}{\act^T_e()}}
  \gprogtrans{\act^T_e\sigma'_s, \filter_S} 
  \tup{A_s'', \scmap_s'', \gemptyprog}
\end{array}
\]
 where 
\begin{compactitem}
\item $\final{\tup{A_s'', \scmap_s'', \gemptyprog}}$
\item $\sigma'_s$ is empty legal parameter assignment (because $\act^T_e$ is
  0-ary action).
\item $\scmap_s'' = \scmap_e''$, (due to the fact that $\act^T_e$ does
  not involve any service call (i.e., $\scmap_s'' = \scmap_s'$) and
  $\scmap_s' = \scmap_e''$).
\end{compactitem}
Additionally, by the definition of $\tgproge$, we have
$\delta''_s = \tgproge(\delta''_e)$ as the rest of the program to be
executed (because $\tup{A_s'', \scmap_s'', \gemptyprog}$ is a final
state). Hence, we have
\[
\tup{A_s,\scmap_s, \delta_s} \trans_s 
\tup{A'_s,\scmap'_s, \delta'_s} \trans_s 
\tup{A''_s,\scmap''_s, \delta''_s}
\]
%
To complete the proof, we obtain $A_s'' = A_e''$ simply as a
consequence of the following facts:
\begin{compactenum}
\item $A_s = A_e$;

\item   By Lemma~\ref{lem:evol-prop}, we have $N(c) \in A''_e$ if and only
  if either
\begin{compactitem}[$\bullet$]
\item $N(c) \in \addfactss{A_e}{\act\sigma_e}\theta_e$, or
\item $N(c) \in (A_e \setminus \delfactss{A_e}{\act\sigma_e})$ and
  there does not exists
  $B(c) \in \addfactss{A_e}{\act\sigma_e}\theta_e$ such that
  $T \models N \sqsubseteq \neg B$;
\end{compactitem}

\item By Lemma~\ref{lem:eact-prop}, we have $N(c') \in A''_s$ if and only
  if $N$ is not in the vocabulary of TBox $T^n$ and either
\begin{compactitem}[$\bullet$]
\item $N(c') \in \addfactss{A_s}{\act'\sigma_s}\theta_s$, or
\item $N(c') \in (A_s \setminus \delfactss{A_s}{\act\sigma_s})$ and
  there does not exists
  $B(c') \in \addfactss{A_s}{\act'\sigma_s}\theta_s$ such that
  $T \models N \sqsubseteq \neg B$.
\end{compactitem}

\item $\delfactss{A_e}{\act\sigma_e}  =  \delfactss{A_s}{\act'\sigma_s}$.
\item $N(c) \in \addfactss{A_e}{\act\sigma_e}\theta_e$ if and only
  if $N(c), N^n(c) \in \addfactss{A_s}{\act'\sigma_s}\theta_s$.
\item $P(c_1,c_2) \in \addfactss{A_e}{\act\sigma_e}\theta_e$ if and
  only if
  $P(c_1,c_2), P^n(c_1,c_2) \in
  \addfactss{A_s}{\act'\sigma_s}\theta_s$.
\item $\act^T_e$ flushes all ABox assertions made by using
  $\voc(T^n)$.
\end{compactenum}
The other direction of bisimulation relation can be proven
symmetrically.
\end{proof}

Having Lemma~\ref{lem:egkab-to-sgkab-bisimilar-state} in hand, we can
easily show that given an E-GKAB, its transition system is S-bisimilar
to the transition of its corresponding S-GKAB that is obtained via the
translation $\tgkabe$.

\begin{lemma}\label{lem:egkab-to-sgkab-bisimilar-ts}
  Given a E-GKAB $\gkabsym$, we have $\ts{\gkabsym}^{\filter_E} \sbsim
  \ts{\tgkabe(\gkabsym)}^{\filter_S} $
\end{lemma}
\begin{proof}
Let
\begin{compactenum}
\item $\gkabsym = \tup{T, \initABox, \actset, \ginitprog_e}$ and\\
  $\ts{\gkabsym}^{\filter_E} =
  \tup{\const, T, \stateset_e, s_{0e}, \abox_e, \trans_e}$,
\item
  $\tgkabb(\gkabsym) = \tup{T_s, \initABox, \actset_s, \ginitprog_s}$
  and \\ %
  $\ts{\tgkabe(\gkabsym)}^{\filter_S} = \tup{\const, T_s, \stateset_s,
    s_{0s}, \abox_s, \trans_s}$)
\end{compactenum}
We have that $s_{0e} = \tup{A_0, \scmap_e, \delta_e}$ and
$s_{0s} = \tup{A_0, \scmap_s, \delta_s}$ where
$\scmap_e = \scmap_s = \emptyset$. 
By the definition of $\tgproge$ and $\tgkabe$, we also have
$\delta_s = \tgproge(\delta_e)$.  Hence, by
Lemma~\ref{lem:egkab-to-sgkab-bisimilar-state}, we have
$s_{0e} \sbsim s_{0s}$. Therefore, by the definition of one-step
jumping history bisimulation, we have
$\ts{\gkabsym}^{\filter_E} \sbsim \ts{\tgkabe(\gkabsym)}^{\filter_S}$.
\end{proof}

Having all of the ingredients in hand, we are now ready to show that
the verification of \muladom properties over E-GKABs can be recast as
verification over S-GKABs as follows.

\begin{theorem}\label{thm:egkab-to-sgkab}
  Given an E-GKAB $\gkabsym$ and a closed $\muladom$ property $\Phi$,
\begin{center}
  $\ts{\gkabsym}^{\filter_E} \models \Phi$ iff
  $\ts{\tgkabe(\gkabsym)}^{\filter_S} \models \tford(\Phi)$
\end{center}
\end{theorem} 
\begin{proof}
  By Lemma~\ref{lem:egkab-to-sgkab-bisimilar-ts}, we have that
  $\ts{\gkabsym}^{\filter_E} \sbsim
  \ts{\tgkabe(\gkabsym)}^{\filter_S}$.
  Hence, by Lemma~\ref{lem:sbisimilar-ts-satisfies-same-formula}, we have
  that the claim is proved.
\end{proof}

\section{Putting It All Together:  From I-GKABs to S-KABs}

\subsubsection{Proof of Theorem~\ref{thm:itos}}
\begin{proof}
As a consequence of Theorems~\ref{thm:bgkab-to-sgkab},
\ref{thm:cgkab-to-sgkab}, and \ref{thm:egkab-to-sgkab}, we essentially
show that the verification of \muladom properties over I-GKABs can be
recast as verification over S-GKABs since we can recast the
verification of \muladom properties over B-GKABs, C-GKABs, and E-GKABs
as verification over S-GKABs.
\end{proof}

\subsubsection{Proof of Theorem~\ref{thm:itoverys}}
\begin{proof}
  The proof is easily obtained from the Theorems~\ref{thm:itos} and
  \ref{thm:gtos}, since by Theorem~\ref{thm:itos} we can recast the
  verification of \muladom over I-GKABs as verification over S-GKABs
  and then by Theorem~\ref{thm:gtos} we can recast the verification of
  \muladom over S-GKABs as verification over S-KABs. Thus combining
  those two ingredients, we can reduce the verification of \muladom
  over I-GKABs into the corresponding verification of \muladom over
  S-KABs.
\end{proof}

\subsection{Verification of Run-bounded I-GKABs}

This section is devoted to show the proof of Theorem
\ref{thm:i-to-kab}. As the preliminary step, we formalize the notion
of run-boundedness as follows.

\begin{definition}[Run of a GKAB Transition System]
  Given a GKAB $\gkabsym$, a \emph{run of
    $\ts{\gkabsym}^\filter = \tup{\const, T, \stateset, s_{0}, \abox, \trans}$}
  is a (possibly infinite) sequence $s_0s_1\cdots$ of states of
  $\ts{\gkabsym}^\filter$ such that $s_i\trans s_{i+1}$, for all $i\geq 0$.
\end{definition}

\begin{definition}[Run-bounded GKAB]
  Given a GKAB $\gkabsym$, we say \emph{$\gkabsym$ is run-bounded} if
  there exists an integer bound $b$ such that for every run
  $\pi = s_0s_1\cdots$ of $\ts{\gkabsym}^\filter$, we have that
  $\card{\bigcup_{s \textrm{ state of } \pi}\adom{\abox(s)}} < b$.
\end{definition}
\noindent
The notion of run-bounded KABs is similar.

Now we proceed to show that the reductions from I-GKABs to S-GKABs
preserve run-boundedness.

\begin{lemma}\label{lem:run-bounded-preservation-bgkab}
  Let $\gkabsym$ be a B-GKAB and $\tgkabb(\gkabsym)$ be its
  corresponding S-GKAB. We have if $\gkabsym$ is run-bounded, then
  $\tgkabb(\gkabsym)$ is run-bounded.
\end{lemma}
\begin{proof}
  Let
  \begin{compactenum}
  \item $\gkabsym = \tup{T, \initABox, \actset, \ginitprog}$ and
    $\ts{\gkabsym}^{\filter_B} $ be its transition system.
  \item $\ts{\tgkabb(\gkabsym)}^{\filter_S}$ be the transition system
    of $\tgkabb(\gkabsym)$.
  \end{compactenum}
  The proof is easily obtained since
  \begin{compactitem}
  \item the translation $\tgkabb$ essentially only appends each action
    invocation in $\delta$ with some additional programs to manage
    inconsistency.
  \item the actions introduced to manage inconsistency never inject
    new individuals, but only remove facts causing inconsistency,
  \item by Lemma~\ref{lem:bgkab-to-sgkab-bisimilar-ts}, we have that
    $\ts{\gkabsym}^{\filter_B} \lbsim \tgkabb(\gkabsym)$. Thus,
    basically they are equivalent modulo repair states (states
    containing $\temp$).
  \end{compactitem}
\end{proof}

\begin{lemma}\label{lem:run-bounded-preservation-cgkab}
  Let $\gkabsym$ be a C-GKAB and $\tgkabc(\gkabsym)$ be its
  corresponding S-GKAB. We have if $\gkabsym$ is run-bounded, then
  $\tgkabc(\gkabsym)$ is run-bounded.
\end{lemma}
\begin{proof}
  Similar to the proof of
  Lemma~\ref{lem:run-bounded-preservation-bgkab} but using the
  S-Bisimulation.
\end{proof}

\begin{lemma}\label{lem:run-bounded-preservation-egkab}
  Let $\gkabsym$ be a E-GKAB and $\tgkabe(\gkabsym)$ be its
  corresponding S-GKAB. We have if $\gkabsym$ is run-bounded, then
  $\tgkabe(\gkabsym)$ is run-bounded.
\end{lemma}
\begin{proof}
  Similar to the proof of
  Lemma~\ref{lem:run-bounded-preservation-bgkab} but using the
  S-Bisimulation.
\end{proof}

Below we formally state the fact that the reduction from S-GKABs to
S-KABs preserves run-boundedness.

\begin{lemma}\label{lem:run-bounded-preservation-sgkab}
  Let $\gkabsym$ be an S-GKAB and $\tgkab(\gkabsym)$ be its
  corresponding S-KAB. We have if $\gkabsym$ is run-bounded, then
  $\tgkab(\gkabsym)$ is run-bounded.
\end{lemma}
\begin{proof}
  Let $\ts{\gkabsym}^{\filter_S}$ be the transition system of
  $\gkabsym$, and $\ts{\tgkab(\gkabsym)}^S$ be the transition system
  of $\tgkab(\gkabsym)$.
  The proof is then easily obtained since
  \begin{compactitem}
  \item only a bounded number of new individuals are introduced, when
    emulating the Golog program with S-KAB condition-action rules and
    actions.
  \item by Lemma~\ref{lem:sgkab-to-kab-bisimilar-ts}, we have that
    $\ts{\gkabsym}^{\filter_S} \jbsim \ts{\tgkab(\gkabsym)}^{S} $.
    Thus, basically they are equivalent modulo intermediate states
    (states containing $\tmp$). Moreover, each bisimilar states are
    basically equivalent modulo special markers.
  \end{compactitem}
\end{proof}

\subsubsection{Proof of Theorem~\ref{thm:i-to-kab}}

\begin{proof}
  By Lemmas~\ref{lem:run-bounded-preservation-bgkab},
  \ref{lem:run-bounded-preservation-cgkab}, and
  \ref{lem:run-bounded-preservation-egkab}, the translation from
  I-GKABs to S-GKABs preserves run-boundedness. 
%
  Furthermore, by Lemma~\ref{lem:run-bounded-preservation-sgkab},
  run-boundedness is also preserved from S-GKABs to S-KABs. 
%
  In the end, the claim follows by also combining
  Theorem~\ref{thm:itoverys} with the results
  in~\cite{BCDDM13,BCMD*13} for run-bounded S-KABs.
\end{proof}

\end{document}
